\documentclass[11pt]{gsasthesis} 

\usepackage{etex} 

\usepackage[margin={1.2in}]{geometry}

\usepackage[titletoc]{appendix}

\usepackage{rotating}

\usepackage{longtable}

\usepackage{natbib}

\usepackage[sc]{mathpazo}
\usepackage{courier}

\usepackage[utf8]{inputenc}
\usepackage[T1]{fontenc}

\usepackage[american]{babel}
\usepackage{blindtext}

\usepackage{amsmath}
\usepackage{amssymb}
\usepackage{mathtools} 
\usepackage{graphicx}
\usepackage{amsfonts}
\usepackage{bm}
\newcommand{\norm}[1]{\left\lVert#1\right\rVert}

\usepackage{subcaption}
\usepackage{float}
\usepackage{booktabs}

\usepackage{amsthm}
\newtheorem{proposition}{Proposition}

\usepackage{algorithm}
\usepackage{algpseudocode}

\usepackage{microtype}

\usepackage[stable,multiple]{footmisc}

\RequirePackage[font=small,format=plain,labelfont=bf,textfont=it]{caption}
\title{Understanding Uncertainty in Bayesian Deep Learning} 
\author{Cooper Lorsung} 
\degreename{Master of Engineering}
\degreefield{Computational Science and Engineering} 
\department{School of Engineering and Applied Sciences} 
\degreemonth{May} 
\degreeyear{2021} 
\principaladvisor{Weiwei Pan}

\theoremstyle{definition}
\newtheorem{definition}{Definition}[section]

\begin{document}


\pagenumbering{roman} 

\thesistitlepage
\copyrightpage

\renewcommand{\contentsname}{\protect\centering\protect\Large Contents}
\renewcommand{\listtablename}{\protect\centering\protect\Large List of Tables}
\renewcommand{\listfigurename}{\protect\centering\protect\Large List of Figures}

\tableofcontents 

\listoftables
\listoffigures
\begin{acknowledgments}
  The author would like to acknowledge Finale Doshi-Velez, Weiwei Pan, and Yaniv Yacoby for their gracious advising and mentoring, and Yaniv Yacoby additionally for collaborative work. The author would also like to acknowledge Sujay Thakur for collaborative work. Lastly, the author would like to acknowledge Harvard FAS Research Computing for providing computational resources that made this thesis possible.
\end{acknowledgments}
\begin{dedication}
  To my parents Chris and Jenni, my step-parents Jandeen and John, my sister Sophie, and especially my brother Conor for our machine learning discussions.
\end{dedication}

\begin{abstract}
  Neural Linear Models (NLM) are deep Bayesian models that produce predictive uncertainty by learning features from the data and then performing Bayesian linear regression over these features. 
  Despite their popularity, few works have focused on formally evaluating the predictive uncertainties of these models. Furthermore, existing works point out the difficulties of encoding domain knowledge in models like NLMs, making them unsuitable for applications where interpretability is required.
  In this work, we show that traditional training procedures for NLMs can drastically underestimate uncertainty in data-scarce regions. 
  We identify the underlying reasons for this behavior and propose a novel training method that can both capture useful predictive uncertainties as well as allow for incorporation of domain knowledge.
\end{abstract}

\pagenumbering{arabic} 


\chapter{Introduction}\label{ch:1}


In high-stakes, safety critical applications of machine learning, reliable measurements of model predictive uncertainty matter just as much as predictive accuracy. 
Traditionally, applications requiring predictive uncertainty relied on Gaussian Processes (GPs) \citep{rasmussen_gp} for two reasons: (1) they produce high predictive uncertainty in data-scarce regions and low uncertainty in data rich ones, and (2) their predictive uncertainty can be easily and meaningfully tuned via a set of hyperparameters of the kernel function (for example, the expressiveness of RBF kernels can be tuned via the length-scale and amplitude), allowing domain experts to encode task-relevant knowledge.
However, due to the computational complexity of GP inference, recent works have advocated for the use of deep Bayesian models with approximate inference as fast and scalable GP alternatives \citep{springenberg_bayesian,Snoek}. These alternatives, unfortunately, often do not retain the two desired properties of GPs.
That is, these models are often overly certain on test points coming from data-poor regions of the input space \citep{uci_gap}, and it is unintuitive to tune their hyperparameters in order to achieve task-appropriate behavior \citep{sun2018functional}. 
Bayesian Neural Networks (BNNs) \citep{neal2012bayesian}, for example, provide a way of explicitly capturing model uncertainty - uncertainty from having insufficient observations to determine the ``true'' predictor - by placing a prior distribution over network weights. 
Like GP inference, rather than point estimates, Bayesian inference for BNNs produces distributions over possible predictions, whose variance can be used as an indicator of model confidence during test time. 
However, inference for large BNNs remains challenging;
that is, many tractable approximations of BNN posteriors yield posterior predictive uncertainties that can identify out-of-distribution points far from the training data, but  cannot reliably distinguish data-rich from data-poor regions close to the training data \citep{uci_gap, uncertainty_quality} - these approximations underestimate the so-called "in-between" uncertainties.

For this reason, Neural Linear Models (NLM), a model similar to BNNs but with tractable inference, is gaining popularity ~\citep{Snoek, rl_nlm, activelearning_nlm, automl_nlm}.
NLMs place a prior only on the last layer of weights and learn point estimates for the remaining layers; inference for the last layer can then be performed analytically. 
One can interpret the deterministic layers as a finite dimensional feature-space embedding of the data, and the last layer of NLMs as performing Bayesian linear regression on the \emph{feature basis}, that is, the basis defined by the feature embedding of the data. 

Although NLMs are easy to implement and are scalable~\citep{Snoek}, in order to deploy them in applications that require useful predictive uncertainties, we nonetheless need to verify that these models retain desirable properties of GPs. But despite their increasing popularity, little work has been done to formally evaluate the quality of uncertainty estimates produced by NLMs. 
In the first paper to do so \citep{Rasmussen}, the authors show that NLMs can achieve high log-likelihood on test data sampled from training data-scarce regions; they treat this as evidence that NLM uncertainties can distinguish data-scarce and data rich regions. 
However, as noted by \cite{uncertainty_quality}, log-likelihood measures only how well predictive uncertainty aligns with the variation in the actual data and not how well these uncertainties predict data-scarcity. In fact, we will show that, like BNNs learned with various approximate inference methods, the predictive uncertainties of NLMs resulting from traditional inference are overly confident on data-scarce regions close to the training data (NLMs underestimate in-between uncertainties). 
Furthermore, unlike in the case of GP models, it is much more difficult to encode domain or functional knowledge in deep Bayesian models \citep{sun2018functional}, and hence the predictive uncertainties of these models are often difficult to interpret in context of a specific downstream task.

\section{Bayesian Statistics}
In statistics, broadly, two schools of thought prevail in approaching problems.
They are the frequentist and Bayesian approaches.
Frequentist statistics views probabilities as long-term outcomes of repeatable experiments (\cite{MIT_stats_17}).
Frequentists place probabilities over data, given a hypothesis.
That is, the hypothesis is formulated, then the probability the given data was collected given the hypothesis is calculated.
\begin{equation*}
    \mathcal{L}(H; D) = P(D|H)
\end{equation*}
If the probability is sufficiently high (or low), the hypothesis is accepted (or rejected).

The primary focus of this work is in Bayesian methods and we therefore focus more on understanding those methods here.
In Bayesian statistics, probability is viewed differently.
Probability distribution are placed over both data and hypotheses, and probabilities are updated according to Bayes' Rule:
\begin{equation*}
    P(H|D) = \frac{P(D|H)P(H)}{P(D)}
\end{equation*}
Here, the distribution over the hypothesis $P(H)$ is known as the \textit{prior distribution}, and is the Bayesian way of incorporating prior knowledge into the model.
The distribution $P(H|D)$ is known as the \textit{posterior distribution} because it is our model \textit{after} seeing data.
The \textit{likelihood} function is how we incorporate data into our model, and calculates how likely the data is given a specific hypothesis.
Prediction can be done both before and after the model has been updated with data.
Given a model, we can make a predictive inference about a data point by integrating over the prior distribution, known as the \textit{prior predictive}:
\begin{equation*}
    p(y) = \int p(y|\theta)p(\theta) d\theta
\end{equation*}
Of much focus and concern, is predictive inference over new points $\Tilde{y}$ after updating our model with our given data.
This is known as the \textit{posterior predictive} distribution, and is calculated by integrating over our model parameters $\theta$:
\begin{equation*}
    p(\Tilde{y}|y) = \int p(\Tilde{y}|\theta) p(\theta|y) d\theta
\end{equation*}
Using this basic mathematical machinery, it is possible to create much more complex models that capture data trends in many situations (\cite{gelmanbda04}).
Examples of this machinery in action are given throughout the remainder of this text, with the simplest example given in section \ref{sec:bayesian_linear}.

\section{Uncertainty}
\label{sec:uncertainty}
Bayesian statistics makes interpreting uncertainty intuitive.
People have an intuitive understanding of what uncertainty means in every day decisions.
For example, if someone is 50\% certain it will rain today, they believe it is as likely to rain as it is to not rain.
In other words, this person is unsure if it will rain or not.
In the case of machine learning, this understanding of uncertainty can be applicable.
For a model that is classifying images as either a cat or dog, if the model outputs a probability of 50\% cat and 50\% dog for a given image, we would rightfully interpret this as the model being uncertain if the image is a cat or dog.

In a regression setting where we are predicting continuous values, the interpretation is a bit less intuitive, but can still easily understood.
For example, say historical data says the temperature tomorrow is supposed to be $50^{\circ}$F, and the standard deviation of temperature is $10^{\circ}$F.
A model that predicts the temperature to be 52$^{\circ}$F should do so with relatively small uncertainty.
Historical trends suggest this is a good prediction since it is within one standard deviation of the mean.
However, if the model predicts the temperature to be $90^{\circ}$, it should do so with relatively high uncertainty.
Historical trends suggest this prediction is incredibly unlikely since it is four standard deviations away from the mean.
Another way to say this is that predictions well-supported by the data should have low uncertainty.
The model should be fairly certain in predictions that are near the data: predicting a temperature that is within a standard deviation or two of the historical mean.
Models should be uncertain of predictions that are not near the data: predicting a temperature many standard deviations away from the historical mean.
This analogy works because we understand that weather tends to vary day-to-day, but not too drastically.
If the predicted temperature from our model is $2000^{\circ}$, the model should have incredibly high uncertainty, considering the ambient temperature on earth has never been this high in all of recorded human history.
Jumps this large, with correspondingly high uncertainties can sometimes be ideal, however.
In materials science, small increases in temperature can cross a phase transition boundary, leading to a large difference in heat capacity.

This begs the question: how do we know if the uncertainty from a prediction is any good or not?
Before this question can be answered, the different types of uncertainty must be understood.

\subsection{Types of Uncertainty}
\label{sec:types_of_uncertainty}
Broadly, the two type of uncertainty are \textit{aleatoric} and \textit{epistemic} uncertainty.
Aleatoric uncertainty is the type of uncertainty inherent to the data.
This can come from the data generation or collection process.
Aleatoric uncertainty can be incorporated into a model as a data noise parameter, for example in bayesian regression.
An easy way to remember this is to think \textit{aleatoric} is inherent to \textit{all} of the data.
Epistemic uncertainty is the uncertainty that is present due to a lack of data.
For the remainder of this work, epistemic uncertainty and uncertainty will be used interchangeably.
Aleatoric uncertainty will mostly be referred to data noise or output noise.

\subsection{Good Uncertainty}
\label{sec:good_uncertainty}
Having a 'good' uncertainty estimate is critical in detecting out-of-distribution data, as well as regions with little or no training data.
'Good' is in quotes because there is no one metric that can say whether an uncertainty estimate is good or not.
Additionally, 'good' uncertainty is task and context dependent.
What is good uncertainty on one task may be bad uncertainty in another, potentially with the same training data.
This dependence on task and context means there is no ground truth good uncertainty.
The lack of ground truth good uncertainty makes the task of quantification difficult.

In this work, we use visual analysis, average uncertainty, and propose a new method of benchmarking uncertainty against gold-standard methods, like GPs, to better understand uncertainty.
First, an intuitive understanding of what constitutes 'good' uncertainty must be established.
Good uncertainty usually should increase where there is no data.
Similar to how we decide things in the face of little knowledge, we want our machine learning models to do the same.
For example, if all we know is the temperature outside right now, we would be very uncertain about whether or not it will be rainy tomorrow.
Next, we generally want our uncertainties to increase smoothly as we get farther from our data.
For example, when forecasting weather, weather models are less certain about the weather two weeks from now than they are two hours from now.
While these are often good guidelines, they do not hold in every case.

We see this 'ideal' behavior clearly in \ref{fig:low_good_high}.
On the left, there is no increase in uncertainty where there is no data.
This hold true in both the gap region, and the out-of-distribution regions on either side.
The model is therefore overly confident in its demonstrably bad predictions in the gap region, where it both does not capture the ground truth, and also has low uncertainty.
The middle plot shows good uncertainty.
Where there is no data, the uncertainty increases smoothly the farther from the data we get.
The uncertainty is in the same order of magnitude as the data, and captures simple, but unexpected perturbations in the data within two standard deviations.
The right plot is an example of when a model has too much uncertainty.
The model does not capture the unexpected perturbations like in the other two plots, but the uncertainty is orders of magnitude larger than the data.
This is also not ideal because the model is saying, in effect, that predictions that vary wildly and unrealistically outside of the data are as likely as predictions that don't.
To draw on the weather example once again, a prediction with uncertainty like this is saying we are incredibly uncertain what the weather will be in 15 minutes, and just as uncertain what it will be in 15 minutes as in 2 weeks.

However, as a counterexample, say the ground truth was simply a cubic function rather than a perturbed cubic function.
Additionally, the data was generated from a physics experiment where we know a priori the class of possible functions in this context are all scalar multiples of $f(x) = Cx^3$.
Then it follows that the left plot has better uncertainty because it is as confident as we are that the ground truth is cubic.
Both the middle and right plots have too much uncertainty in the gap since the bands of uncertainty include functions that cannot be described by a function $f(x) = Cx^3$.

Additionally, in other contexts, the right most plot may be the ideal uncertainty.
For example, if the out-of-distribution regions are forbidden or incredibly unlikely, we would expect the model to have near-infinite uncertainty (for standard regression models infinite uncertainty is impossible, so a significant increase is what we would expect).
In this case, the left and middle plots do not have the ideal uncertainty because there is far too little.

We may also want our uncertainty to be calibrated.
That is, where the 95th percentile predictive uncertainty captures 95 percent of the data.
In simple examples like this cubic gap, it is fairly easy to have all of these predictive uncertainty characteristics.

\begin{figure}[H]
    \centering
    \includegraphics[width=1.0\linewidth]{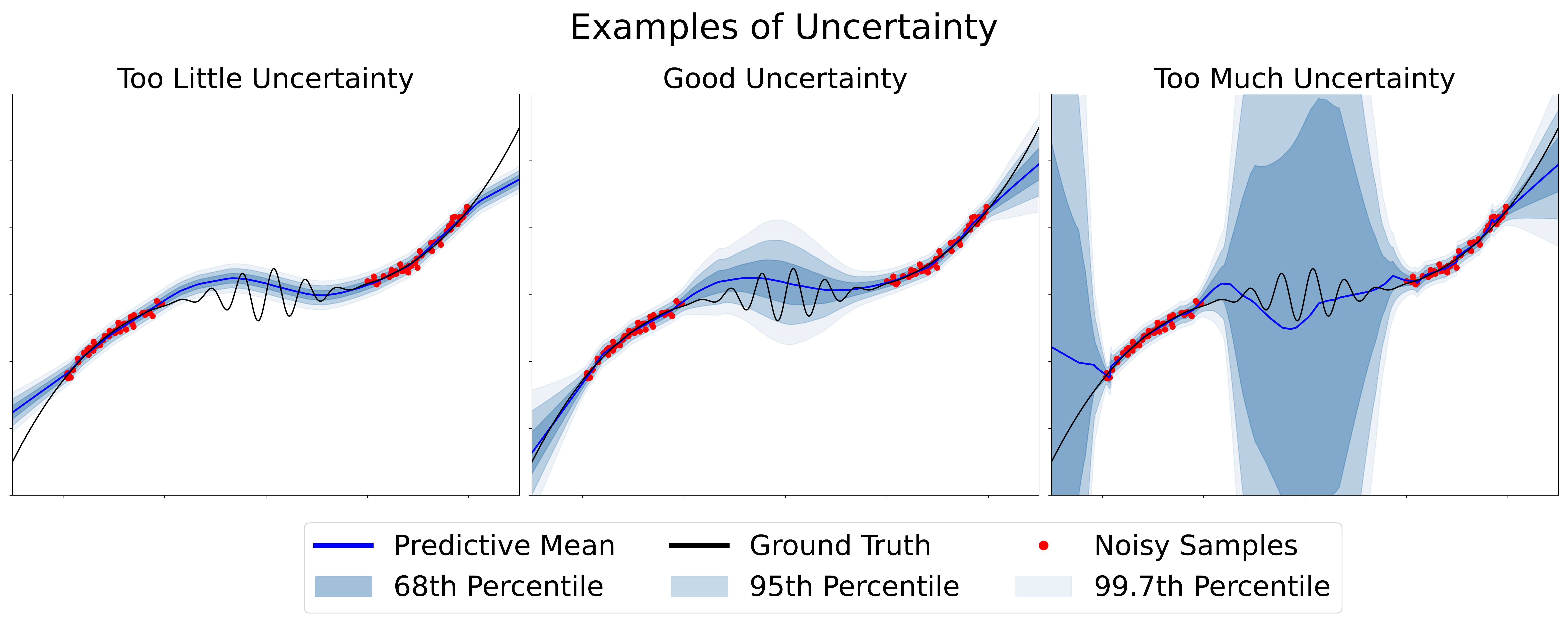}
    \caption{On a cubic function with a gap in training data where there are unexpected perturbations, we see examples of good and bad uncertainty. Bad uncertainty can be both in the form of too much or too little uncertainty.}
    \label{fig:low_good_high}
\end{figure}

However, hen visual analysis is difficult or impossible, such as in higher dimensions, we need another way to determine if uncertainty is good or bad.
In the case of real data experiments, the input is often high dimensional and while visualization may be helpful, this if often only in specific cases.
For example, for UCI gap data, we can visualize the gap, and get an idea of how the uncertainty increases, but this is primarily helpful because we know where the gap is beforehand.
For UCI Gap\cite{uci_gap} data, we sort along a feature and remove the middle third of the data for training.
This artificially introduces a gap into the data.
In figure \ref{fig:gap_high_d}, we see that the uncertainty clearly increases in the gap.
However, this visualization is only useful because we know the gap in data is there.
\begin{figure}[H]
    \centering
    \includegraphics[width=1.0\linewidth]{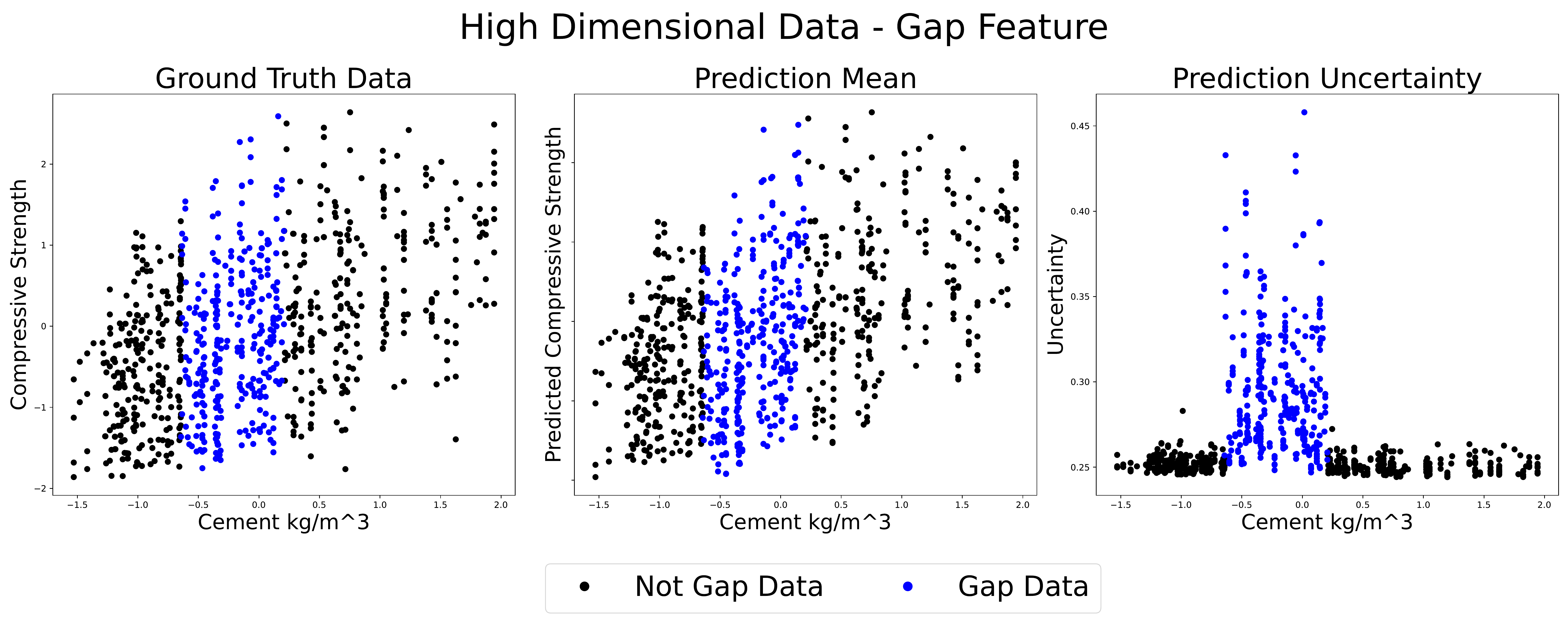}
    \caption{Concrete compressive strength with a gap introduced into the Cement feature. We see there is clearly an increase in uncertainty in the gap region.}
    \label{fig:gap_high_d}
\end{figure}

In figure \ref{fig:not_gap_high_d}, we plot a feature we did not introduce a gap into.
We see there is an increase in uncertainty for some of the data, but there is no way to determine visually where this gap is or why there is an increase.
\begin{figure}[H]
    \centering
    \includegraphics[width=1.0\linewidth]{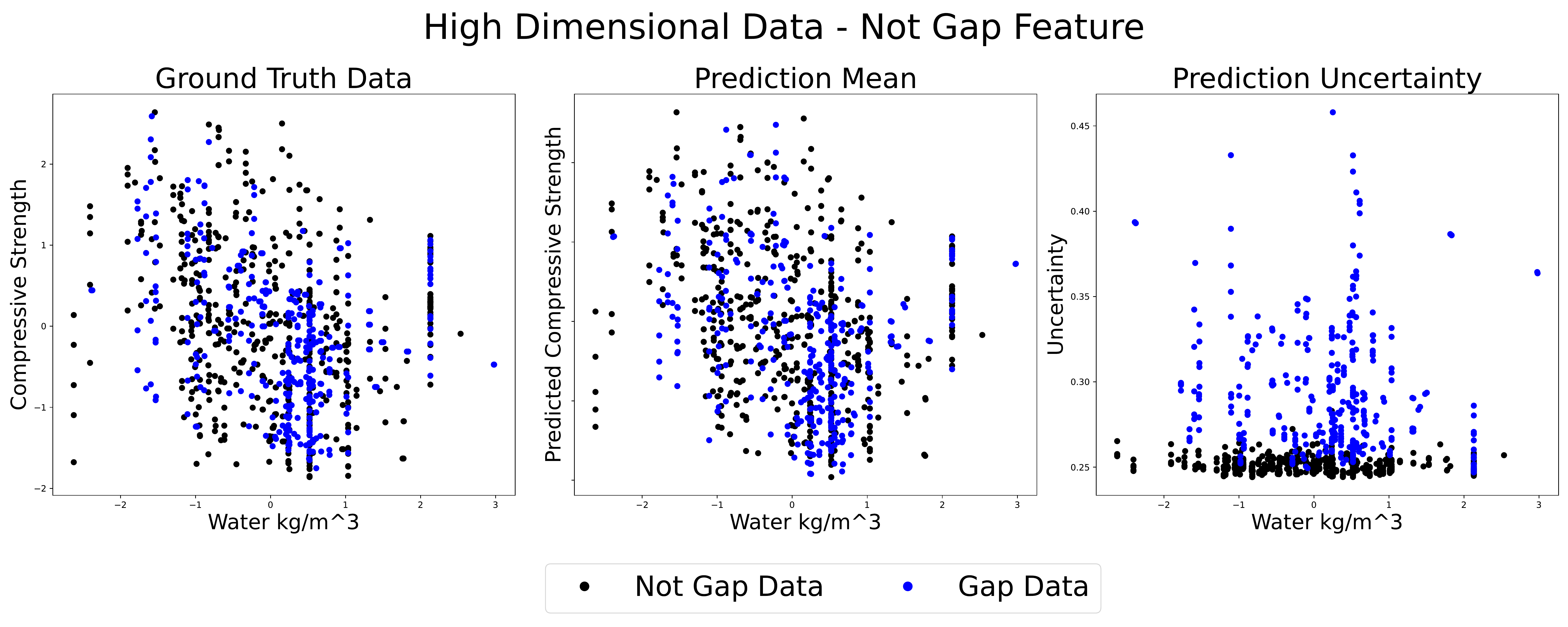}
    \caption{Concrete compressive strength with a gap introduced into the Cement feature plotted along the Water feature. We see there is clearly an increase in uncertainty for some of the data points, but there is no clear relationship between the Water feature and the increase in uncertainty.}
    \label{fig:not_gap_high_d}
\end{figure}
Because of this, in higher dimensional data, there is difficulty in establishing a standard benchmark as to what "good" uncertainty is.
In this work, "good" uncertainty in high dimensional data is simply a significant increase from known, in-distribution data.
For example, in the high-dimensional UCI data sets, we simply take the average uncertainty across all gap and not-gap data points.
If a circular gap region was introduced in the middle of two features, or a spherical in the middle of three, visualization would become impossible.
In order to create a more standard benchmark for these models, a novel experiment is proposed in section \ref{sec:uncertainty_benchmark}

\section{This Work}

The remainder of this work is dedicated to, first, understanding how and why popular models fail to adequately quantify uncertainty.
In Chapter \ref{ch:2}, gold-standard and benchmark models are explored.
The gold-standard models are Gaussian Processes and Bayesian neural networks (BNN) sampled with Hamiltonian Monte Carlo (HMC).
The benchmark models explored are BNN trained with Variational Inference (VI), the Neural Linear Model (NLM), Monte Carlo Dropout (MCD), Bootstrapped Ensembles, Anchored Ensembled, and Spectral-normalized Neural Gaussian Processes (SNGP).
Extra emphasis on building understanding the models from the ground up is given through explanation of neural networks, standard training algorithms, and Bayesian linear regression.

Second, recent advances and related works are explored to understand the current landscape of models.
Advances in GPs, BNNs, and other models are discussed to determine their viability for quantifying uncertainty.

Third, a novel framework for quantifying uncertainty is developed.
This framework, UNA, augments the NLM with auxiliary regressors.
These auxiliary regressors are either penalized to create principled uncertainty, or fit to reference functions.
UNA is shown to be able to fit data well, and provide consistently good uncertainty.

Fourth, a diverse set of experiments are run in order to test uncertainty in a variety of tasks.
The novel radial uncertainty benchmark is explored in both one and two dimensions.
This benchmark aims to provide a standard way to compare models to gold-standard methods.
The experiments are simple, and easy to interpret.
Toy examples are also explored to see the how well each model can capture in-between uncertainty.
Additionally, toy examples are given to demonstrate encoding uncertainty to get specific results.
Real data experiments are then run using the standard UCI regression benchmark data sets.
Both standard and gap sets are used.
The downstream task of Bayesian Optimization was selected to determine the utility of the uncertainty of each model.
Four optimization tasks are used.
The last experiment done is with applying concept whitening in order to create interpretable uncertainty in UNA.

Lastly, the results are discussed to better understand them.
Potential shortcoming of UNA and the uncertainty benchmark are further discussed.
Future directions are briefly looked at.
Appendices are also given primarily for experimental details and determining the effect of hyperparameters on UNA.

\chapter{Background}\label{ch:2}
Many machine learning models exist for regression tasks in both the frequentist and Bayesian framework.
Some of the more popular methods are covered here.
Model formulation, inference procedure, as well as shortcomings are discussed here.
Due the the substantial amount of literature present for most methods outlined here, this chapter is not an exhaustive exposition of these models.
Readers are referred to the literature cited in the beginning of each subsection for a more detail exploration of each model.

\section{Neural Networks}
\label{sec:nn}
\begin{figure}
    \centering
    \includegraphics[width=0.5\linewidth]{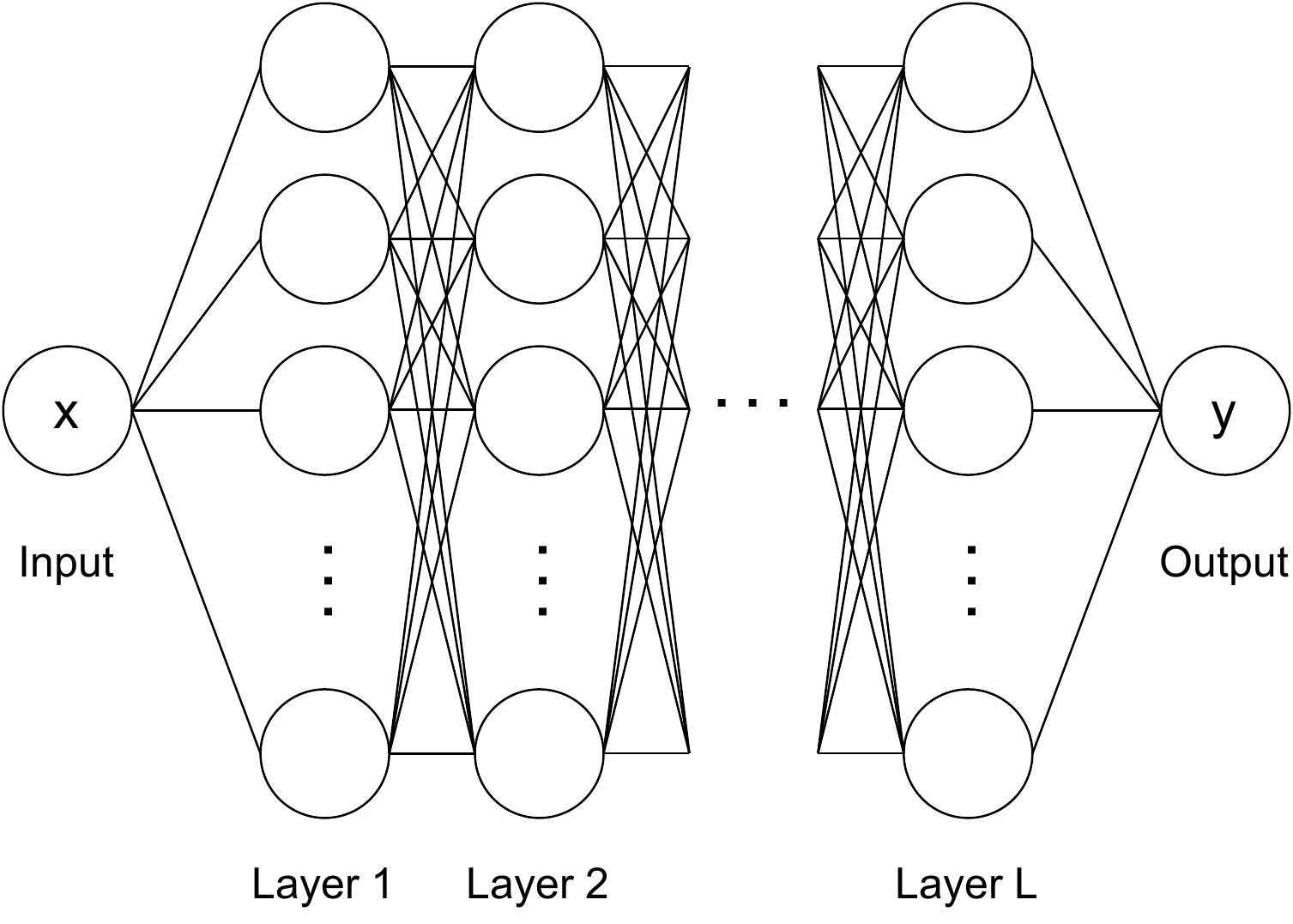}
    \caption{Neural Network with $L$ hidden layers.}
    \label{fig:my_label}
\end{figure}
Neural Networks are one of the most popular and well studied methods in machine learning.
Neural Networks can be understood as a series of nested nonlinear transformations.
Given the architecture seen above, the output is given by
\begin{equation}
    f(x) = g_l(g_{l-1}(\cdots g_{0}(\mathbf{x})\cdots )) = h(\mathbf{W}_{l}h(\mathbf{W}_{l-1}h(\cdots h(\mathbf{W}_0\mathbf{x} + \mathbf{b}_0) + \mathbf{b}_{l-1}) + \mathbf{b}_l)
    \label{eq:nn}
\end{equation}
for a given input $\mathbf{x}$.
For a given neural network $f_{\theta}$, we say it has the learnable parameters $\theta$ that encompass the weights and biases of each layer.
In the context of regression, one of the most popular loss functions is mean squared error (MSE), defined as
\begin{equation}
    \mathcal{L}(\mathbf{X}, \mathbf{y}) = \frac{1}{N}\sum_{i = 1}^{N} \left\lVert f(\mathbf{x}_i) - \mathbf{y}_i \right\rVert^2
    \label{eq:mse}
\end{equation}

\subsection{Training}
\label{sec:nn_train}
In order to use Neural Networks, they must be trained on the data.
The standard training involves taking gradients of the loss function and propagating those gradients back through the network, also known as backpropagation.
This standard procedure can be accomplished with multiple different ways of calculating gradients, each with their own benefits and downsides.
For each of these training algorithms, the gradients can be either calculated one point at a time ($N$ = 1 in equation \ref{eq:mse}), or in a batched way, where the data is subdivided into $B$ batches of $M$ points ($N$ = $M$ in equation \ref{eq:mse}) and the average gradient is used.

\subsubsection{Gradient Descent}
\label{sec:gd}
One of the most popular and easily implemented optimization methods is gradient descent.
Gradient descent generally calculates gradients using all of the data at once, and steps parameters in the direction of steepest decrease in loss.

\begin{algorithm}
\begin{algorithmic}[1]
\State{\textbf{Input}: Neural Network $f_{\theta}$, learning rate $\eta$, data $\mathcal{D} = \{\mathbf{x}_i, \mathbf{y}_i\}_{i=1}^N$, training epochs $T$, target distribution $\pi(\mathbf{q})$}
\State{\textbf{Output}: Trained Neural Network $f_{\theta}$}
\For{$i$ in $1, \ldots, T$}
    \State{$\theta_{i+1} \gets \theta_{i} + \eta\nabla_{\theta}\mathcal{L}(\mathbf{X}, \mathbf{y})$}
\EndFor
\end{algorithmic}
\caption{Gradient Descent}
\label{algo:gd}
\end{algorithm}

One of the downsides to gradient descent is it gets stuck in local minima.
There is no way for the loss to increase outside of a large learning rate.
This is an unreliable method of getting out of local minima and often leads to unstable performance.

\subsubsection{Stochastic Gradient Descent}
\label{sec:sgd}
In order to improve upon the local minimum issue of gradient descent, stochasticity is introduced in the form of batched gradients using a subset of points, rather than using the entire data set.

\begin{algorithm}
\begin{algorithmic}[1]
\State{\textbf{Input}: Neural Network $f_{\theta}$, learning rate $\eta$, data $\mathcal{D} = \{\mathbf{x}_i, \mathbf{y}_i\}_{i=1}^N$, training epochs $T$, loss function $\mathcal{L}(\mathbf{X}, \mathbf{y})$}
\State{\textbf{Output}: Trained Neural Network $f_{\theta}$}
\For{$i$ in $1, \ldots, T$}
    \State{Shuffle data \textbf{X}, \textbf{y}}
    \For{$j$ in $1, \ldots, N$}
        \State{$\theta_{j+1} \gets \theta_{j} + \eta\nabla_{\theta}\mathcal{L}(\mathbf{x}_j, \mathbf{y}_j)$}
    \EndFor
\EndFor
\end{algorithmic}
\caption{Stochastic Gradient Descent}
\label{algo:gd}
\end{algorithm}
While performing better than traditional gradient descent, SGD tends to lead to long training times.

\subsubsection{Adam}
\label{sec:adam}
Stochastic methods such as SGD can often be benefited by adding momentum terms.
Momentum is incorporated as an additional term that helps the optimizer get out of local minima and take larger steps.
With larger steps, momentum-based algorithms often require fewer training epochs.
Adam is a very popular momentum-based optimization algorithm\cite{AdamOpt2014}.
Adam has shown to be very successful in machine learning applications and is the primary optimization algorithm used in this work.

\begin{algorithm}
\begin{algorithmic}[1]
\State{\textbf{Input}: Neural Network $f_{\theta}$, learning rate $\eta$, data $\mathcal{D} = \{\mathbf{x}_i, \mathbf{y}_i\}_{i=1}^N$, training epochs $T$, loss function $\mathcal{L}(\mathbf{X}, \mathbf{y})$, weight decays $\beta_0, \beta_1 in [0, 1)$}
\State{\textbf{Output}: Trained Neural Network $f_{\theta}$}
\State{Initialize first and second momentum vectors $m_0, v_0 = 0$}
\For{$i$ in $1, \ldots, T$}
    \State{$g_{t} \gets \nabla_{\theta}\mathcal{L}(\mathbf{X}, \mathbf{y})$}
    \State{$m_{t} \gets \beta_1 m_{t-1} + (1 - \beta_1)g_{t}$}
    \State{$v_t \gets \beta_2 v_{t-1} + (1 - \beta_2)g_{t}^2$}
    \State{$\hat{m}_t \gets \frac{m_t}{1 - \beta_1^t}$}
    \State{$\hat{v}_t \gets \frac{v_t}{1 - \beta_2^t}$}
    \State{$\theta_t \gets \theta_{t-1} - \eta \frac{\hat{m}_t}{\sqrt{\hat{v}_t} + \epsilon}$}
\EndFor
\end{algorithmic}
\caption{Adam optimization}
\label{algo:gd}
\end{algorithm}

\section{Bayesian Linear Models}
\label{sec:bayesian_linear}
Linear models are often considered the 'workhorse' of machine learning.
Intuitive interpretations, as well as exact inference make this model one of the first covered in many machine learning courses.
In this work, the focus is Bayesian regression, known as Bayesian linear regression.

Following the derivation from section 2.1.1 of \cite{rasmussen_gp}, we start with the standard linear regression formulation that has gaussian noise.
$$f(\mathbf{x}) = \mathbf{w}^T\mathbf{x}, \quad y = f(\mathbf{x}) + \epsilon, \quad \epsilon \sim \mathcal{N}\left(0, \sigma^2_n\right)$$
Here, $\sigma^2$ represents the noise inherent in the data, and is independent for each sample.

With Gaussian noise, our likelihood is Gaussian.
If we place a Gaussian prior over the weights $\mathbf{w} \sim \mathcal{N}\left(0, \sigma^2_w\mathbb{I}\right)$, then by Bayes' Rule, our posterior is also Gaussian.
It is important to note that the weights can have a non-diagonal covariance matrix as well.
In this work, diagonal covariance is used exclusively, so the diagonal notation is used to emphasize this.

That is, 
\begin{align*}
    p(\mathbf{w}| \mathbf{x}, \mathbf{y}) \propto& p(\mathbf{y}|\mathbf{x}, \mathbf{w})p(\mathbf{w}) \\
    \propto& \mathcal{N}(\mathbf{y}|\mathbf{x}^T\mathbf{w}, \sigma^2_n)\mathcal{N}(0, \sigma^2_w\mathbb{I}) \\
    \propto& \mathcal{N}(\sigma^{-2}_nA^{-1}\mathbf{x}\mathbf{y}, A^{-1})
\end{align*}
Where $A = \sigma^{-2}_n\mathbf{x}\mathbf{x}^T + \sigma^{-2}_w\mathbb{I}$
In order to calculate the predictive distribution $f^*$ at some test point $f^*(\mathbf{x}^*)$, we must integrate over all possible weights $\mathbf{w}$, weighted by their posterior probability.
The posterior predictive is therefore given by 
\begin{equation}
    p(f_*|\mathbf{x}_*, \mathbf{x}, \mathbf{y}) = \int p(f_* | \mathbf{x}_*, \mathbf{x}, \mathbf{y}, \mathbf{w})p(\mathbf{w} | \mathbf{x}, \mathbf{y})
    = \mathcal{N}(\sigma^{-2}_n\mathbf{x}_*^{T}A^{-1}\mathbf{x}\mathbf{y}, \mathbf{x}_*^{T} A^{-1} \mathbf{x}_*)
\end{equation}
Thus, we have the predictive mean and variance for a test point.
This formulation often suffers from a lack of expressivity.
Low expressivity can be often be seen in prediction where points far from the data have low uncertainty, or the model struggles to fit to the data.
A common remedy to this is to project the data into a higher dimensional feature space, and perform bayesian regression on the new feature set.

Polynomial regression is often used, and can be represented by a feature basis of monomials: $\boldsymbol\Phi(x) = [1, x, x^2]$.
The feature map $\boldsymbol\Phi$ is often called the basis set, and the transformed data $\boldsymbol\Phi(\mathbf{x})$ is often called the feature basis or design matrix.
Here, the feature map transforms points $\boldsymbol\Phi: \mathbb{R}^1 \to \mathbb{R}^3$.
The posterior and posterior predictive are updated accordingly.
$A = \sigma^{-2}_n \boldsymbol\Phi(\mathbf{x})\boldsymbol\Phi(\mathbf{x})^T + \sigma^{-2}_{w}\mathbb{I}$, which gives us:
\begin{equation}
    p(\mathbf{w}| \mathbf{x}, \mathbf{y}) = \mathcal{N}(\sigma^{-2}_nA^{-1}\boldsymbol\Phi(\mathbf{x})\mathbf{y}, A^{-1})
\end{equation}
\begin{equation}
    p(y_*|\mathbf{x}_*, \mathbf{x}, \mathbf{y}) = \mathcal{N}(\sigma^{-2}_n\boldsymbol\Phi(\mathbf{x}_*)^{T}A^{-1}\boldsymbol\Phi(\mathbf{x})\mathbf{y}, \boldsymbol\Phi(\mathbf{x}_*)^{T} A^{-1} \boldsymbol\Phi(\mathbf{x}_*))
\end{equation}
This idea of using a feature basis underpins many of the models discussed in this work, including GPs, NLMs, SNGP, and UNA.

\section{Bayesian Neural Networks}
\label{sec:bnn}
An early model that combines both neural networks and the Bayesian framework is the Bayesian Neural Network (BNN).
Simply put, a BNN is a regular NN with a prior placed on the weights and biases.
We denote the set of parameters $\boldsymbol\theta = \{\mathbf{w}, \mathbf{b}\}$ such that we have the neural network $y = g_{\theta}(\mathbf{x})$.
We then place a prior (generally a Gaussian) over the neural network's parameters:
\begin{equation}
    \begin{split}
    \boldsymbol\theta &\sim \mathcal{N}(\mathbf{0}, \boldsymbol\sigma_{\boldsymbol\theta}^2) \\
    \boldsymbol\mu_n &= g_{\boldsymbol\theta}(\mathbf{x}_n) \\
    \mathbf{y}_n &\sim \mathcal{N}(\boldsymbol\mu_n, \boldsymbol\sigma_{\mathbf{y}}^2)
    \end{split}
    \label{eq:bnn_eq}
\end{equation}
We see this in figure \ref{fig:bnn_model}.
By Bayes' rule, we know the posterior is given by
\begin{equation}
    p(\boldsymbol\theta|\mathbf{x}, \mathbf{y}) = \frac{p(\mathbf{y}|\mathbf{x},\boldsymbol\theta)p(\boldsymbol\theta)}{\int p(\mathbf{y}|\mathbf{x}, \boldsymbol\theta')p(\boldsymbol\theta')d\boldsymbol\theta'}
    \label{eq:bnn_posterior}
\end{equation}
The posterior here is intractable because the denominator, called the evidence, cannot be computed easily.
Because of this, approximate techniques are required in order perform inference with BNNs.

\begin{figure}
    \centering
    \includegraphics[width=0.5\linewidth]{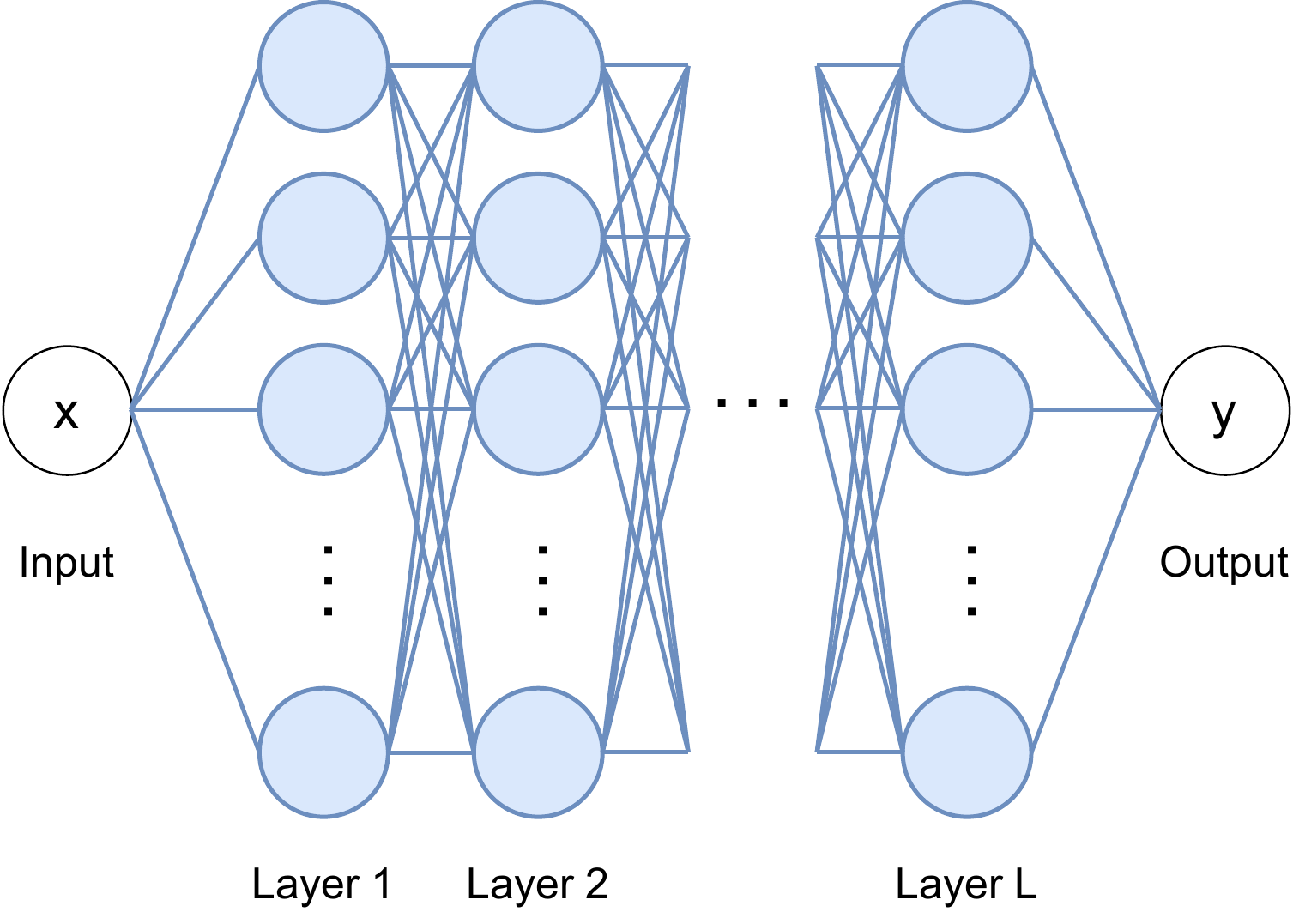}
    \caption{Fully bayesian neural network. The blue nodes and connecting lines indicate a bayesian prior is places over the biases and weights, respectively.}
    \label{fig:bnn_model}
\end{figure}

\subsection{Hamiltonian Monte Carlo}
\label{sec:bnn_hmc}
A gold-standard method for qualitatively good uncertainty is the physics-inspired Hamiltonian Monte Carlo (HMC).
The explanation here follows \cite{neal2012bayesian} closely.
HMC builds on traditional Markov-Chain Monte Carlo methods by using Hamiltonian dynamics to propose new points.
Monte Carlo methods stem from the Metropolis-Hastings algorithm.
Simply put, the Metropolis algorithm samples a distribution probabilistically accepting or rejecting a proposed point\cite{Metropolis1953}.
The distribution in question is called the canonical distribution and is give by
\begin{equation}
    P(q, p) = \frac{1}{\mathcal{Z}}\exp\left(-H(q, p)/T\right)
\end{equation}
The Hamiltonian, $H$ is comprised of the position $q$ and momentum $p$:
\begin{equation}
    H(q, p) = U(q) + K(p)
    \label{eq:hamiltonian}
\end{equation}
Where $U(p)$ is the potential energy, $K(p)$ is the kinetic energy, and $T$ is the temperature.
The constant $\mathcal{Z}$ is a normalizing constant that allows the canonical distribution to be interpreted as a probability of the state $\{q, p\}$.
Following the traditional Metropolis algorithm, we can sample the canonical ensemble by probabilistically accepting proposed states $\{p_{proposed}, q_{proposed}\}$ according to:
\begin{align*}
    \alpha &< \min \left\{1, \frac{\exp\left(-H(q_{proposed}, p_{proposed}/T)\right)}{\exp\left(-H(q, p)/T\right)} \right\} \\
    \alpha &< \min \left\{1, \exp\left[\frac{1}{T}\left(U(q) - U(q_{proposed}) + K(p) - K(p_{proposed})\right)\right]\right\}
\end{align*}
If the equation is true, we accept the proposed state and update $q \gets q_{proposed}$.
For proper comparison, we take the minimum between $1$ and the relative probabilities.
This is simply guaranteeing moves with higher probability are always accepted.
It should be noted here that the normalization constants cancel out and therefore do not need to be calculated.

The equations of motion describe how the system evolves over time and are given by
\begin{equation}
    \begin{split}
        \frac{dq}{dt} &= \frac{\partial H}{\partial p} \\
        \frac{dp}{dt} &= -\frac{\partial H}{\partial q}
    \end{split}
    \label{eq:ham_eom}
\end{equation}
Given a target distribution $\pi(q)$, we can sample this distribution by setting $U(q) = -log(\pi(q))$.
In this case, the target distribution is the BNN posterior.
The kinetic energy $K(p)$ is often defined as $K(p) = \frac{p^TM^{-1}p}{2}$, where $M$ is the symmetric, positive-definite mass matrix that is often a scalar multiple of the identity matrix.
This scalar multiple, called the mass parameter, can be tuned for better training.
The equations of motion must be numerically integrated in order to simulate the dynamics.
The Leapfrog integrator is chosen for this task because it is simple and offers high accuracy.
Quite simply, we take a half step of the momentum, full step of the position, and half step of the momentum again.
With a learning rate, $\eta$, this takes the form:
\begin{align}
    p(t + \eta/2) &= p(t) - \frac{\eta}{2}\frac{\partial H}{\partial q} \\
    q(t + \eta) &= q(t) + \frac{\eta}{2}\frac{\partial H}{\partial p} \\
    p(t + \eta) &= p(t + \eta/2) - \frac{\eta}{2}\frac{\partial H}{\partial q}
\end{align}

Combining all of this, we start with an initial condition $q$, sample an initial momentum $p$, simulate the dynamics to reach a new position, and decide whether or not to accept the new position.
The algorithm is given below in Algorithm \ref{algo:hmc}.
Proof of ergodicity and convergence are give in \cite{neal2012bayesian}.

\begin{algorithm}
\begin{algorithmic}[1]
\State{\textbf{Input}: Neural Network $g_{\theta}$, learning rate $\eta$, data $\mathcal{D} = \{\mathbf{x}_i, \mathbf{y}_i\}_{i=1}^N$, training epochs $T$, gradient steps $L$, potential energy function $U(q)$, kinetic energy function $K(p)$, initial guess $q_{current}$}
\State{\textbf{Output}: List of accepted $q$ steps, $q_{trace}$}
\For{$i$ in $1, \ldots, T$}
    \Statex{}
    \State{\# Get initial positions}
    \State{$q_{proposed} \gets q_{current}$}
    \State{$p_{proposed} \sim \mathcal{N}(\mathbf{0}, \mathbf{M}^2)$}
    \Statex{}
    \State{\# Perform leapfrog integration}
    \State{$p_{proposed} \gets p_{proposed} -\frac{\eta}{2}U'(q)$}
    \For{$j$ in $1, \ldots, L$}
        \State{$q_{proposed} \gets q_{proposed} + \eta p_{proposed}$}
        \State{$p_{proposed} \gets p_{proposed} - \eta U'(q_{proposed})$}
    \EndFor
    \State{$p_{proposed} \gets \frac{\eta}{2}U'(q_{proposed})$}
    \State{$p_{proposed} \gets -p_{proposed}$}
    \Statex{}
    \State{\# Calculate Hamiltonian of current and proposed states}
    \State{$U_{current} = U(q_{current}$})
    \State{$K_{current} = K(p_{current}$})
    \State{$U_{proposed} = U(q_{proposed}$})
    \State{$K_{proposed} = K(p_{proposed}$})
    \Statex{}
    \State{\# Accept or reject the move}
    \State{$\alpha \sim \mathcal{U}(0, 1)$}
    \If{$\alpha < \exp\left(U_{current} - U_{proposed} + K_{current} - K_{proposed}\right)$}
        \State{$q_{current} \gets q_{proposed}$}
        \State{Append $q_{current}$ to $q_{trace}$}
    \Else
        \State{Append $q_{current}$ to $q_{trace}$}
    \EndIf
\EndFor
\caption{Hamiltonian Monte Carlo}
\label{algo:hmc}
\end{algorithmic}
\end{algorithm}

It is well known that HMC scales poorly to large data sets.
Due to this scaling, it is rarely used in practice and other methods are required.

\subsection{Variational Inference}
\label{sec:bnn_vi}

Variational Inference (VI) is one such method that has seen widespread use\cite{Blei_2017}.
VI is well suited for large data sets, in part because it does not attempt to learn the true posterior.
A variational family is proposed to approximate the true posterior.
The variational family often times is a family of functions that has been restricted in some way to make inference easier.
A measure of difference, usually the Kullback-Leibler Divergence, is minimized between the true posterior and the variational family.

The variational family $\mathcal{Q}$ is a family of densities over the latent variables $\mathbf{z}$.
VI works to find the best density $q(\mathbf{z}) \in \mathcal{Q}$ by calculating
\begin{equation}
    q^*(\mathbf{q}) = \underset{q(\mathbf{q}) \in \mathcal{Q}}{\text{argmin}} \mathbb{KL}\left(q(\mathbf{z}) \lvert p(\mathbf{z}|\mathbf{x})\right)
    \label{eq:vi_obj}
\end{equation}
As we saw earlier, the posterior for a BNN is given by equation \ref{eq:bnn_posterior}.
The same trouble in calculating the evidence is present here.
The evidence lower bound (ELBO) is optimized in place of the KL divergence.
The ELBO is simply $-D_{KL} + \log p(\mathbf{x})$ and is given by:
\begin{equation}
    ELBO(q) = \mathbb{E}\left[\log p(\mathbf{z}, \mathbf{x}) \right] - \mathbb{E}\left[\log q(\mathbf{z}) \right]
\end{equation}
Because $p(\mathbf{x})$ has no dependence on $\mathbf{z}$, maximizing the ELBO is equivalent to minimizing the KL divergence.

Perhaps the most popular variational family, the Mean-field variational family assumes that latent variables are mutually independent.
That is,
\begin{equation}
    q(\mathbf{z}) = \prod_{j=1}^{m}q_j(\mathbf{z}_j)
\end{equation}
In this setting, each $q$ generally takes the form of a univeriate Gaussian distribution.
With the lack of correlation between parameters, the variational family effectively fits to the mean of the data posterior but cannot capture any covariance between parameters.

With the variational family chosen and fitting re-framed as a tractable optimization problem, fitting can take place.
A popular, and fairly intuitive, choice for this is coordinate ascent variational inference (CAVI) \cite{bishop:2006:PRML}.
This algorithm can be understood as updating each variational factor $q_{j}(\mathbf{z}_j)$ while keeping others constant.
Updates are performed as
\begin{equation}
    q_j^*(z_j) \propto \exp\left\{\mathbb{E}_{-j}\left[\log p(z_j | \mathbf{z}_{-j}, \mathbf{x})\right]\right\}
\end{equation}
Where the notation $-j$ means with respect to all coordinates not equal to $j$.
That is, the variational density $q_{-j}(\mathbf{z}_{-j}) = \prod_{i \neq j}q_{i}(\mathbf{z}_i)$.
After all coordinates have been updated, the ELBO is computed once again.
If the ELBO has converged then CAVI is complete.

Computing updates is often times quite challengind.
To avoid having to work through tedious calculations, a black-box approach is often taken, known as Black-box Variational Inference (BBVI) \cite{ranganath2013black}.

\section{Gaussian Processes}
\label{sec:gp}

Gaussian Processes (GP)\cite{rasmussen_gp} are an extremely popular model that offer exact inference that scales cubically with number of data points.
Due to this cubic scaling, inference on large datasets is often computationally infeasible.

\subsection{Model Details}
\label{sec:gp_model}
Model details follow the excellent reference\cite{rasmussen_gp}, and a more detailed exploration of GPs is given there.
GPs can be interpreted in multiple ways.
First, they can be viewed as a linear combination of inputs in a high-dimensional feature space.
This interpretation is known as the weight space view, because predictions are made using a weighted combination of the data.
The second, less intuitive interpretation is in function space.
This interpretation views GPs as a distribution over functions.

\subsubsection{Weight Space}
\label{sec:gp_weight}
The weight space view of GPs follows bayesian regression with a basis set from section \ref{sec:bayesian_linear} with the kernel trick.
Because GPs work in high dimensional feature space without the assumption that $w \sim \mathcal{N}(0, \Sigma_w)$ has diagonal covariance, it is more appropriate to do the derivation in general terms.
We therefore replace $\sigma_w^2$ with $\Sigma_w$.

The kernel trick involves replacing instances of $\boldsymbol\Phi(\mathbf{x})^T \Sigma_w \boldsymbol\Phi(\mathbf{x})$ with a kernel, or covariance, matrix $K$.
The kernel matrix acts as an inner product in $d$-dimensional feature space and can be thought of as a function $K: \mathbb{R}^{d\times d} \to \mathbb{R}$.

Once we do this, and simplify in the same way as \cite{rasmussen_gp}, we have the new posterior predictive:

\begin{equation}
\begin{split}
    p(y_*| \mathbf{x}_*, \mathbf{x}, \mathbf{y}) = \mathcal{N}(&\boldsymbol\Phi(\mathbf{x}_*)^T\Sigma_w\boldsymbol\Phi(\mathbf{x})(K + \sigma^2_n\mathbb{I})^{-1}\mathbf{y}, \\
    &\boldsymbol\Phi(\mathbf{x}_*)^T\Sigma_w\boldsymbol\Phi(\mathbf{x}_*)  - \boldsymbol\Phi(\mathbf{x}_*)^T\Sigma_p \boldsymbol\Phi(\mathbf{x})(K + \sigma^2_n \mathbb{I})\boldsymbol\Phi(\mathbf{x})^T\Sigma_w \boldsymbol\Phi(\mathbf{x}_*))
\end{split}
\end{equation}

\subsubsection{Function Space}
\label{sec:gp_func}
The other, equivalent, interpretation of a GP is in function space.
That is, we view a GP as a distribution over functions defined as\cite{rasmussen_gp}:
\begin{definition}
    A \textbf{Gaussian Process} is a collection of random variables, any finite number of which have a joint Gaussian distribution.
\end{definition}
In this view, a GP can be completely specified by its mean and covariance functions: $f(\mathbf{x}) \sim \mathcal{GP}(m(\mathbf{x}), k(\mathbf{x}, \mathbf{x}'))$.
Here, the mean and covariance functions are defined as:
\begin{equation}
    m(\mathbf{x}) = \mathbb{E}[f(\mathbf{x})]
\end{equation}
\begin{equation}
    k(\mathbf{x}, \mathbf{x}') = \mathbb{E}[(f(\mathbf{x}) - m(\mathbf{x}))(f(\mathbf{x}') - m(\mathbf{x}'))]
\end{equation}
The mean function is often set to 0.
This is reasonable to do when dealing with normalized data.
Stemming from the definition, if we have $y_1 \sim \mathcal{N}(\mu_1, \sigma_1)$, and $y_2 \sim \mathcal{N}(\mu_2, \sigma_2)$, then it must be true that $(y_1, y_2) \sim \mathcal{N}(\boldsymbol\mu, \boldsymbol\Sigma)$.
This can be satisfied by simply requiring that $\sigma_1 = \Sigma_{11}$ and $\sigma_2 = \Sigma_{22}$, also known as consistency.

With this, we can evaluate the posterior GP: $f|\mathcal{D} \sim \mathcal{GP}(m_{\mathcal{D}}, k_{\mathcal{D}})$.
Where for a point $\mathbf{x}$,
\begin{equation}
    m_{\mathcal{D}}(\mathbf{x}) = m(\mathbf{x}) + \Sigma_{*,\mathbf{x}}^T\Sigma^{-1}(\mathbf{f} - \boldsymbol{\mu})
\end{equation}
\begin{equation}
    k_{\mathcal{D}}(\mathbf{x}, \mathbf{x}') = k(\mathbf{x}, \mathbf{x}') - \Sigma_{*,\mathbf{x}}^T\Sigma^{-1}\Sigma_{*,\mathbf{x}'}
\end{equation}
In this case, $\Sigma^{-1}$ is the inverse covariance matrix for all points $\mathbf{x} \in \mathcal{D}$, $\Sigma_{\*,\mathbf{x}}$ is the covariance vector between all points in the training set and the new point $\mathbf{x}$, $\boldsymbol\mu$ is the vector of $m(x)$ evaluated at all points in the training set, and $\mathbf{f}$ is the GP evaluated at all points in the training set.
Predictions can be made in this way for new test points $x_*$.

\subsection{Inference}
\label{sec:gp_inference}
Inference for a GP is inference for the kernel's hyperparameters.
Because the posterior can be computed exactly using standard Gaussian update formulae, the kernel's hyperparameters are the only parameters that can be tuned.
The most popular way to do this is to integrate out the function values, $\mathbf{f}$, leading to output conditioned on just the data:
\begin{equation}
    p(\mathbf{y}|\mathbf{x}) = \int p(\mathbf{y}|\mathbf{f}, \mathbf{x})p(\mathbf{f}|\mathbf{x})d\mathbf{f}
\end{equation}
This leads to:
\begin{equation}
    \log p(\mathbf{y}|\mathbf{x}) = -\frac{1}{2}\mathbf{y}^T(K + \sigma_n^2\mathbb{I})^{-1}\mathbf{y} - \frac{1}{2}\log\left|K + \sigma^2_n\mathbb{I}\right| - \frac{n}{2}\log 2\pi
\end{equation}
where $\Sigma = K + \sigma^2_n\mathbb{I}$.
For a more detailed explanation, readers are referred to \cite{rasmussen_gp}.

\section{Neural Linear Model}
\label{sec:nlm}
\begin{figure}
    \centering
    \includegraphics[width=0.5\linewidth]{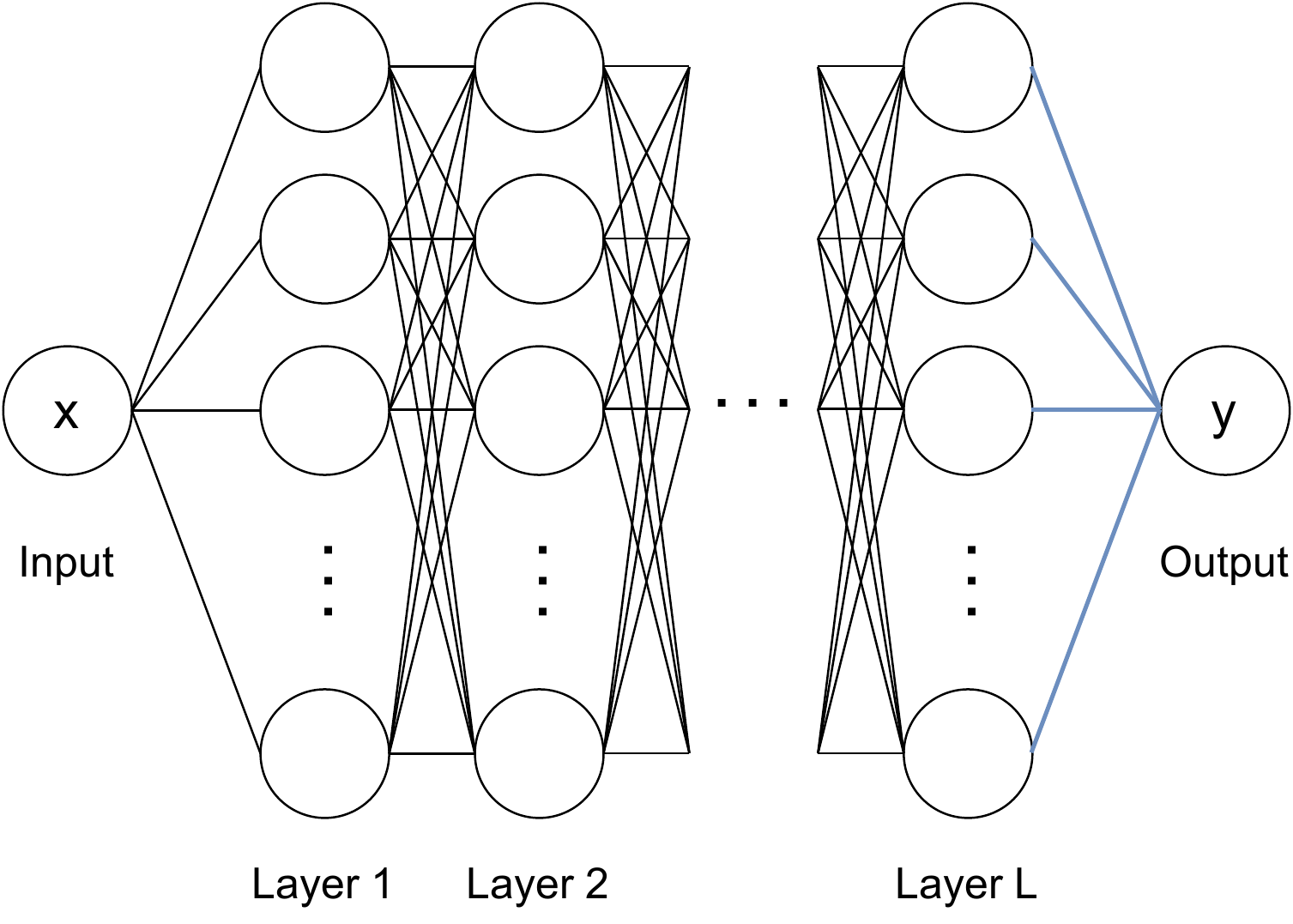}
    \caption{Neural Linear Model. Blue lines indicate a bayesian prior is placed over the last layer, allowing for Bayesian regression at prediction time.}
    \label{fig:my_label}
\end{figure}

Let the input space be $D$-dimensional, and suppose we have a dataset $\mathcal{D} = \{(\mathbf{x}_1, y_1), \ldots,  (\mathbf{x}_N, y_N)\}$, where $\mathbf{x}_n \in \mathbb{R}^D$ and $\mathbf{y}_n \in \mathbb{R}$. A Neural Linear Model (NLM) consists of: (1) a feature map $\phi_\theta:\mathbb{R}^D \to \mathbb{R}^L$, parameterized by a neural network with weights $\theta$, and (2) a Bayesian linear regression model fitted on the the data embedded in the feature space:
\begin{equation*}
\mathbf{y} \sim \mathcal{N}(\bm{\Phi}_\theta \mathbf{w}, \sigma^2\mathbf{I}), \quad 
\mathbf{w} \sim \mathcal{N}(\mathbf{0}, \alpha\mathbf{I}),
\end{equation*}
where the design matrix 
$
\bm{\Phi}_\theta = [\widetilde{\phi_\theta(\mathbf{x}_1)}, \hdots, \widetilde{\phi_\theta(\mathbf{x}_N)]^T}
$ is called the \emph{feature basis} and $\widetilde{\phi_\theta(\mathbf{x}_n)}$
is the feature vector $\phi_\theta(\mathbf{x}_n)$ augmented with a 1 (for a bias term). Thus, for the NLM, given the learned feature map, the posterior, marginal and posterior predictive distributions are all computed analytically. Intuitively, an NLM represents a neural network with a Gaussian prior over the last-layer weights $\mathbf{w}$, and with deterministic weights $\theta$ for the remaining layers. 

Inference for NLMs consists of two steps: 

\textbf{Step 1:} Learn $\theta$.

\textbf{Step 2:} Given $\theta$, infer $p(\mathbf{w} | \mathcal{D}, \theta)$ analytically.

In Step 1, there are three accepted methods of learning $\theta$: maximum likelihood (MLE), maximum a posteriori (MAP) and Marginal-Likelihood, of which MAP is the most common. 

\subsection{MAP Training}
\label{sec:nlm_map}
Here, we largely follow the specification of ~\cite{Rasmussen}. NLM uses a neural network to parameterize basis functions for a Bayesian linear regression model by treating the output weights of the network probabilistically, while treating the rest of the network's parameters $\theta$ as hyperparameters. 

Using notation from Section \ref{sec:bayesian_linear} and following standard Bayesian linear regression analysis, we can derive the posterior predictive as
$$p(y_\star|\mathbf{x}_\star, \mathcal{D}) = \mathcal{N}(y_\star; \mathbf{w}_N^T\phi_\theta(\mathbf{x}_\star), \sigma^2 + \phi_\theta(\mathbf{x}_\star)^T\mathbf{V}_N\phi_\theta(\mathbf{x}_\star))$$
where
\begin{equation}
    \begin{split}
        \mathbf{w}_N &= \frac{1}{\sigma^2}\mathbf{V}_N\bm{\Phi}_\theta^T\mathbf{y}\\
        \mathbf{V}_N^{-1} &= \frac{1}{\alpha}\mathbf{I}_{M\times M} + \frac{1}{\sigma^2}\bm{\Phi}_\theta^T\bm{\Phi}_\theta.
    \end{split}
\label{eq:nlm_posterior}
\end{equation}
For the MAP-trained NLM, we maximize the objective
\begin{equation*}
    \begin{split}
        \mathcal{L}_{\mathrm{MAP}}(\theta_\mathrm{Full}) &= \log{\mathcal{N}\left(\mathbf{y}; \bm{\Phi}_\theta\mathbf{w}, \sigma^2\mathbf{I}\right)} - \gamma\norm{\theta_\mathrm{Full}}_2^2\\
&=-\frac{N}{2}\log{2\pi\sigma^2} - \frac{1}{2\sigma^2}\norm{\mathbf{y}-\bm{\Phi}_\theta\mathbf{w}}_2^2\\
&\quad \quad -\gamma\norm{\theta_\mathrm{Full}}_2^2
    \end{split}
\end{equation*}
where $\theta_\mathrm{Full}$ represents the parameters of the full network (including the output weights). We would then extract $\theta$ from $\theta_\mathrm{Full}$ and perform the Bayesian linear regression as above.

In MAP training~\citep{Snoek}, one maximizes the likelihood of the observed data with respect to $\theta$ and a point estimate, $\mathbf{\widetilde{w}}$, for the weights of the last layer (i.e. we train the entire network deterministically), with an $\ell_2$-regularization term on the weights of the entire network:
\begin{equation}
    \mathcal{L}_{\text{MAP}}(\theta_\text{Full}) = \log{ \mathcal{N}\left(\mathbf{y}; \bm{\Phi}_\theta\mathbf{w}, \sigma^2\mathbf{I}\right)} - \gamma\norm{\theta_\text{Full}}_2^2
    \label{eq:nlm_obj}
\end{equation}

where $\theta_\text{Full} = (\theta, \mathbf{\widetilde{w}})$ are weights of the full network. 
In Step 2, $\mathbf{\widetilde{w}}$ is \emph{discarded} and we use the $\theta$ learned in Step 1 to infer $p(\mathbf{w} | \mathcal{D}, \theta)$.
MLE training is the same as MAP training, but with $\gamma = 0$. 

\subsection{Marginal Likelihood Training}
\label{sec:marginal_ll}

The NLM is defined the same as above, but we optimize $\theta$ to maximize the evidence or log marginal likelihood of the data instead (by integrating out $\mathbf{w}$). For training stability and identifiability, we further regularize $\theta$ as done by ~\cite{Rasmussen}. The full objective is hence:
\begin{equation*}
    \begin{split}
        \mathcal{L}_{\mathrm{Marginal}}(\theta_\mathrm{Full}) &= \log{\int p(\mathbf{y}|\mathbf{X}, \mathbf{w})p(\mathbf{w})d\mathbf{w}}\\
&=-\frac{N}{2}\log{2\pi\sigma^2} - \frac{1}{2\sigma^2}\norm{\mathbf{y}-\bm{\Phi}_\theta\mathbf{w}}_2^2\\
&\quad \quad -\frac{M}{2}\log{\alpha}-\frac{1}{2\alpha}\norm{\mathbf{w}_N}_2^2\\
&\quad \quad -\frac{1}{2}\log{|\mathbf{V}_N|}
    \end{split}
\end{equation*}
\cite{Rasmussen} note that the addition of a regularization term $\gamma\norm{\theta}_2^2$ to $\mathcal{L}_{\mathrm{Marginal}}$ is necessary for the estimates of the output noise since this tends to zero when the objective is unregularized.

\section{Monte Carlo Dropout}
\label{sec:mcd}
\begin{figure}
    \centering
    \includegraphics[width=0.5\linewidth]{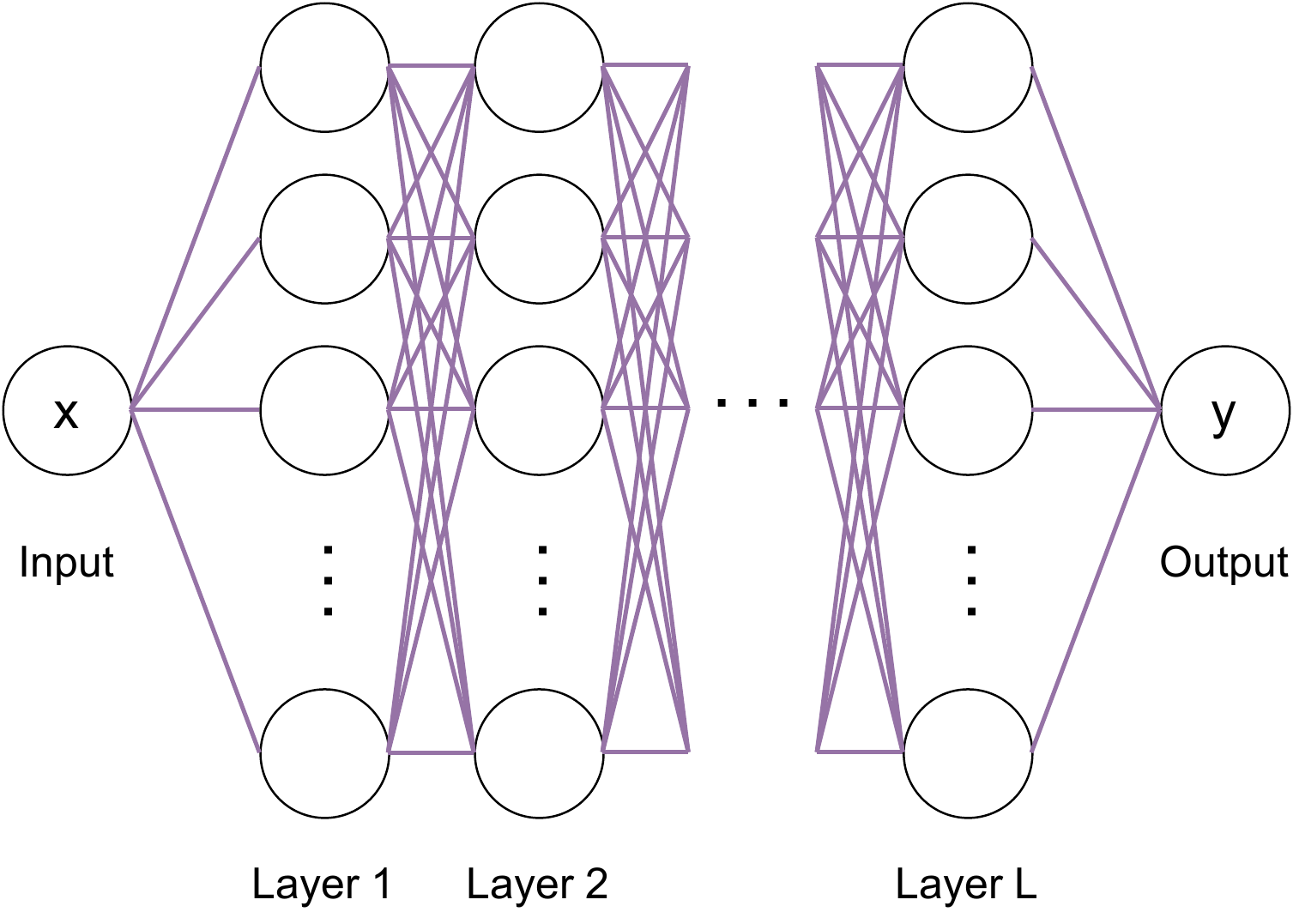}
    \caption{Monte Carlo Dropout is closely related to standard neural networks. The purple weights here signify that dropout is used in both training and prediction time.}
    \label{fig:my_label}
\end{figure}
Monte Carlo Dropout (MCD) casts dropout training in neural networks as approximate Bayesian inference in deep Gaussian processes \citep{mcd}. 
While dropout during training is a common feature of modern neural network architectures, MCD maintains the dropout during testing time too. 
Using multiple stochastic forward passes through the network and averaging the results, MCD is able to obtain a predictive distribution during inference.

\section{Ensemble Methods}
\label{sec:ens}
Ensembling is a common way of constructing multiple Neural Networks in order to have prediction mean and uncertainty.
Many such models exist, and they offer the built-in ability to train in parallel, which can significantly improve training time.

\subsection{Bootstrap Ensembles}
\label{sec:boot_ens}
Perhaps one of simplest and most intuitive type of ensemble, Bootstrap Ensembles are an easy to implement method.
Bootstrap Ensembles simply train each neural network on a different bootstrapped sample of the data.
That is, with $M$ networks in the ensemble, we sample the data $X_j, Y_j \sim \left\{ x_i, y_i \right\}_{i=1}^{N}$, for $j \in [1, M]$, with replacement.
Neural Network $j$ is then trained on the data set $\{X_j, Y_j\}$.
Prediction is then done by predicting with each neural network to get $y_i^*$, then taking the mean and standard deviation of all predictions.
That is, prediction is done by:
\begin{equation}
    y^* = \frac{1}{M}\sum_{j=i}^{M}y^*_j, \quad \sigma_{y^*} = \sqrt{\frac{1}{M}\sum_{j=1}^{M}(y^*_j - y^*)^2}
\label{eq:ens_predict}
\end{equation}
Bootstrapping is an intuitive way of incorporating data noise into the model, however, when regularization is added to the neural networks, the predictions tend to collapse where there is no data.

\subsection{Anchored Ensembles}
\label{sec:anc_ens}
Anchored Ensembles is a method of estimating uncertainty in a Bayesian framework using an ensemble of neural networks.
The key idea of the method is to regularize each neural net's parameters against an anchoring distribution:
\begin{equation*}
    Loss_j = \frac{1}{N}\left\lVert\mathbf{y} - \hat{\mathbf{y}}_j\right\rVert_2^2 + \frac{1}{N} \left\lVert\bm{\Gamma}^{1/2}(\bm{\theta}_j - \bm{\theta}_{anc, j})\right\rVert_2^2
\end{equation*}
The regularization matrix is defined as $\text{diag}(\bm{\Gamma}) = \sigma_{\epsilon}^2/\sigma_{prior_i}^2$, where $\sigma_{\epsilon}$ is the noise variance of the data, $\sigma_{prior_i}$ is the anchor variance.
That is, the anchoring parameters are sampled according to $\bm{\sigma} \sim \mathcal{N}\left( \bm{\mu}_{prior}, \bm{\Sigma}_{prior} \right)$.
In our case, we used one value of $\sigma_{prior}$ and always set $\bm{\mu}_{anchor} = \mathbf{0}$.
Additionally, we follow the original work and decouple the initial parameters from the anchoring distribution, where initial parameters are sampled according to $\theta \sim \mathcal{N}(0, \sigma_{init}^2)$.

These anchoring points ensure that the ensemble fits to the data but also maintains a the variety in initializations.
This variety in initializations is what allows Anchored Ensembles to capture uncertainty in data-scarce regions.
Prediction is done in the same way as Bootstrap Ensembles, seen in equation \ref{eq:ens_predict}

\section{Spectral-normalized Neural Gaussian Process}
\label{sec:sngp}
\begin{figure}
    \centering
    \includegraphics[width=0.5\linewidth]{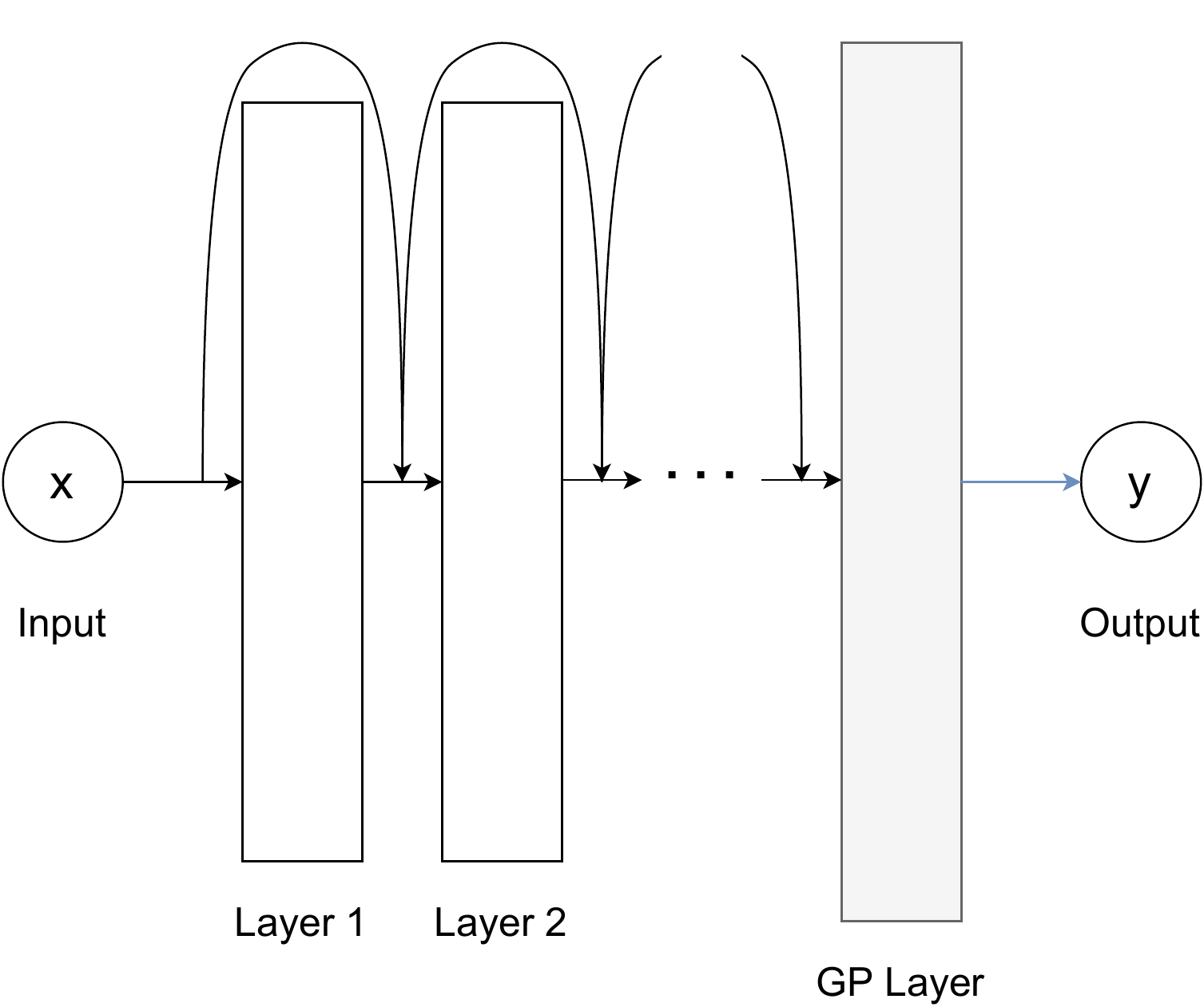}
    \caption{The grey layer indicates that the GP layer is the fixed RFF approximation. The blue connecting line indicates a bayesian prior on the last layer weights, allowing for Bayesian regression at prediction time.}
    \label{fig:my_label}
\end{figure}
The Spectral-normalized Neural Gaussian Process is a model was originally proposed for classification.
In this work we have adapted this model for regression by replacing the laplace approximation of the posterior with the analytical posterior.
The SNGP uses a distance-aware Recurrent Neural Network (RNN) to learn a feature map for the data. This feature map is then used as input for a GP.
Distance awareness is used for uncertainty estimation.
The original work \cite{liu2020simple} used a Laplace approximation to the random feature expansion of a GP, in this case Random Fourier Features.

We have used both the Laplace approximation with RFF last layer (RFF SNGP) and a GP with RBF kernel last layer (SNGP).
SNGPs enforce distance preservation in the RNN through spectral normalization.
That is, a normalization factor $c$ is chosen, and after each gradient update, the weights of each RNN layer are normalized according to 
\begin{equation*}
    \textbf{W} = \begin{cases} c\textbf{W}_l/\hat{\lambda} \\ \textbf{W}_l \end{cases}
\end{equation*}
where $\hat{\lambda}$ is the approximate largest eigenvalue of the weight matrix $W$, obtained with power iteration.

RFF SNGPs also estimate the posterior covariance through updating the precision matrix cheaply as $\hat{\Sigma}_{t}^{-1} = (1-m)\hat{\Sigma}_{t-1}^{-1} + m\sum_{i=1}^M\hat{p}_{i}(1-\hat{p}_{i})\bm{\Phi}_i\bm{\Phi}_i^T$.
Here $M$ is batch size $m$ is a scaling coefficient, $\Phi$ is the data feature, and $\hat{p}_i$ is the model prediction under MAP estimates.
Precision updating is done for each batch in the training cycle.

In the case of using a GP, the precision update Step 1s excluded as posterior covariance is learned in the GP training procedure.

At prediction time, the feature for a data point is calculated by:
\begin{equation*}
    \bm{\Phi}_{D_L \times 1} = \sqrt{2/D_L}\cos\left(\mathbf{W}_Lh(\mathbf{x})+\mathbf{b}_L\right)
\end{equation*}
This feature is then used to calculate the posterior mean and covariance as $\mu(\mathbf{x}) = \bm{\Phi}^T\beta$, $var(\bm{x}) = \bm{\Phi}^T\hat{\Sigma}\bm{\Phi}$.

For posterior re-learning, we simply fix the learned RNN and perform the precision update step for new data. This makes SNGP suitable for transfer learning and generalization experiments.

For the real data and BayesOpt benchmarks, SNGP was updated to compute the posterior analytically using the bayesian linear regression process as LUNA and NLM.

\chapter{Related Works}\label{ch:3}

\section{Gaussian Process Advances}
Due to the great predictive performance of GPs, many methods have been proposed to improve upon the cubic scaling of exact inference.
Primarily, these improvements come in the form of approximate methods.
However, more recent advances have shown exact inference to be tractable for large data sets.

\subsection{Approximate Methods}
An approximate method such as inducing points \cite{snelson2006sparse} work by introducing a pseudo data set $\Bar{\mathcal{D}} = \left\{\Bar{\mathbf{X}_i}, \Bar{\mathbf{y}_i} \right\}_{i=1}^M$ of size $M$, such that $M < N$ and $N$ is the number of training points.
A Gaussian prior is placed over the pseudo data outputs $p(\Bar{\mathbf{y}} | \Bar{\mathbf{X}}) = \mathcal{N}(\mathbf{0}, \mathbf{K}_M)$, where $\mathbf{K}_M$ is the covariance matrix for all points $K(\Bar{\mathbf{x}}_m, \Bar{\mathbf{x}}_{m'})$.
Now, for inference, the covariance is computed using $\boldsymbol\Lambda = \text{diag}(\boldsymbol\lambda)$, where $\lambda_n = K_{nn} - \mathbf{k}_n^T\mathbf{K}_m^{-1}\mathbf{k}_n$ and $\mathbf{k}_{m}$ is the covariance between a point $\mathbf{x}$ and the points $\Bar{\mathbf{x}}$.
This now replaces the standard covariance inversion we see in GP inference.
Because $\boldsymbol\Lambda$ is diagonal, inversion is $O(n)$, and the method reduces inference time to $O(M^2N)$ due to the other matrix multiplications.
This can offere a significant speedup when $M \lll N$

Additionally, using neural networks to learn the kernel function have been proposed \cite{wilson2016deep}.
This model uses a nonlinear transformation with a neural network $g$ parametrized by weights and biases $\mathbf{w}$
The data is transformed and the kernel function with parameters $\boldsymbol{\theta}$ is then applied to this latent representation as:
\begin{equation}
    k(\mathbf{x}_i, \mathbf{x}_j) \to k(g(\mathbf{x}_i, \mathbf{w}), g(\mathbf{x}_j, \mathbf{w}) | \boldsymbol\theta, \mathbf{w})
\end{equation}
The GP then performs inference as seen in section \ref{sec:gp_inference}.
However, before inference can be done, the transformation and kernel parameters, $\boldsymbol\gamma = \left\{\mathbf{w}, \boldsymbol\theta \right\}$ must be learned.
In practical use, the covariance matrix $\mathbf{K}_{\boldsymbol\gamma}$ is approximated using the KISS-GP covariance matrix (\cite{wilson2015thoughts}, \cite{wilsonKISSGP}).
This is chosen because it allows for $O(n)$ scaling with training data.

Stochastic variational optimization (\cite{cheng2017variational}) is a works by first decoupling the bases for the mean and covariance of a GP.
Once this is done, the GP is known as a Decoupled Gaussian Process (DGP).
Variational inference is then done with the decoupled subspace parametrization.
One of the downsides mentioned is that variance tends to be overestimated.

\subsection{Exact Inference}
A recent breakthrough in exact inference has shown it to be tractable for large data sets \cite{wang2019exact}.
The method involves reducing both the memory overhead as well as the computational overhead.
Traditionally, GP inference relies on the Cholesky decomposition of the covariance for inversion.
Here, the kernel matrix is never explicitly computed.
In its place, the kernel is partitioned to reduce memory requirements of matrix-vector multiplications (MVM).
Inference is then done in constant-sized pieces, which reduces the memory overhead to $O(n)$.
These partitioned calculations can then be done in parallel, and due to the new constant-size MVMs, the communication overhead is also reduced to $O(n^2/p)$, where $p$ is the number of partitions.

These approximate techniques improve scaling to large data sets, but it is currently unknown how well the exact GP posterior is approximated.
Parallel exact inference with mulitple GPUs is promising, but may be inappropriate when computational resources are limited.
Furthermore, is is unclear if GPs can exploit complex structure in data, such as features in images, or local/global dependencies in natural language, as well as NN based models.

\section{Bayesian Neural Network Advances}
Early work on Bayesian Neural Networks (BNN) inference focuses on Hamiltonian Monte Carlo (HMC) \citep{neal2012bayesian}, as seen in section \ref{sec:bnn_hmc}, and Laplace approximations of the posterior \citep{mackay1992practical, buntine1991bayesian}.

While HMC remains the ``gold standard'' for BNN inference, it does not scale well to large architectures or datasets; classical Laplace approximation, like Linearised Laplace (LL) \citep{mackay1992practical, ritter2018scalable}, has similar difficulties scaling to modern architectures with large parameter sets.
The Laplace approximation uses the second-order Taylor expansion of the posterior around a MAP estimate of the NN parameters $\boldsymbol\theta^*$.
\begin{equation}
    \log p(\boldsymbol\theta | \mathcal{D}) \approx \log p(\boldsymbol\theta^*|\mathcal{D}) - \frac{1}{2}(\boldsymbol\theta - \boldsymbol\theta^*)^T\Bar{\mathbf{H}}(\boldsymbol\theta - \boldsymbol\theta^*)
\end{equation}
Where $\Bar{\mathbf{H}}$ is the Hessian of the posterior averaged over the $N$ data points.
The posterior predictive can then be approximated for a given data point by sampling the approximate posterior.
For $T$ samples, the predictive mean is calculated according to:
\begin{equation}
    p(\mathbf{y}^* | \mathcal{D}) \approx \frac{1}{T}\sum_{t=1}^T p(\mathbf{y} | \boldsymbol\theta_t)
\end{equation}

Although variational inference methods can be easily applied to BNN models with larger architectures, a number of these methods like mean-field variational inference (VI), seen in Section \ref{sec:bnn_vi} \citep{anderson1987mean, hinton1993keeping, blundell2015weight} and Monte Carlo Dropout (MCD) \citep{mcd}, seen in Section \ref{sec:mcd} (which can be recast as a form of approximate variational inference) have recently been shown to underestimate predictive uncertainty in data-scarce regions \citep{uci_gap, uncertainty_quality, mcd_pathologies}.

Another drawback of BNNs is the difficulty of incorporating task-specific functional knowledge into the prior over weights.
There are works which specifies priors for BNNs directly over functions \citep{shi2019scalable, sun2018functional}.
That is, VI is done in function space, rather than weight space.
This is analogous to the weight space and function space interpretations of GPs.
However, inference in function space is challenging.

\section{Ensemble Advances}
Many models have been developed in order to quantify uncertainty.
Models such as Bootstrapped Ensembles and Anchored Ensembles discussed previously in sections \ref{sec:boot_ens} and \ref{sec:anc_ens}, respectively, work by training multiple neural networks on the same data and obtaining uncertainty from the difference in predictions across models.
Other works exist based on this idea.

While some of these models have quantified uncertainty, and empirically have been shown to perform well across benchmarks, the drawback is always that multiple models need to be trained.
This downside is mitigated with parallelized training, but can still be prohibitive to users that lack access to such computational resources.

Non-traditional ensembles have also been explored.
One such example treats an single NN as multiple subnetworks (\cite{havasi2020training}).
Training is done by calculating loss for multiple inputs and multiple corresponding outputs simultaneously.
This approach is based on the 'lottery ticket' hypothesis that states that neural netowrks (dense, fully connected ones) can be pruned down to a subnetwork that has comparable performance (\cite{frankle2019lottery}).
The authors demonstrate that multiple subnetworks are indeed learned, however peak performance is achieved at very few subnetworks, usually $3$.
This improves upon baseline implementations, but such few networks limits the expressivity.

\section{Bayesian Models with Deterministic Neural Network Features}

Models combining the deterministic training and simple Bayesian models have been proposed to bypass difficulties in BNN inference.
The simple Bayesian model allows for exact and/or scalable inference.
The Marginalized Neural Network for large scale regression tasks was introduced by \cite{lazaro2010marginalized}.
This model uses a neural network to parametrize a deterministic transformation of the data into features.
The features are then fed into Bayesian linear regression.
Later, this model was re-introduced as the Neural Linear Model (NLM) by \cite{Snoek}.
The NLM was used for Bayesian Optimization (BayesOpt) on large datasets that would normally be difficult for GP inference.
While NLMs have been successfully applied in a number of applications requiring predictive uncertainty (like BayesOpt), \cite{Rasmussen} were the first to formally evaluate the uncertainty of NLMs.
This work shows that traditional inference for NLMs will, in most cases, understimate uncertainty in data-scarce regions.
Manifold Gaussian Process (MGP) is a similar model that jointly trains a neural network feature map of the data and a GP model on top of these features \citep{calandra2016manifold}.
This idea is made scalable by \cite{liu2020simple}, who use a random feature expansion GP approximation and isometry-enforcing regularization on the neural network feature map (SNGP).
While MGP and SNGP are similar to NLM in form, inference for these models can be much slower when using GPs on the features, and when approximating GPs with a sufficiently large number of random features~\citep{random_bases}. 

\section{Uncertainty Benchmarks}
While benchmarks certainly exist to evaluate and compare models, such benchmarks are often real data such as the UCI data sets \cite{uci_data}.
These data sets are useful in determining how well models perform on high dimensional, noisy data.
While this is helpful in determining performance, the standard metrics of root mean squared error (RMSE) and log-likehliood (LL) lack the additional information to understand performance.
In the context of uncertainty, these data sets lack a known gap in the data.
When one is introduced, we can visualize it easily as in figure \ref{fig:gap_high_d}, but when we do not know where the gap is a priori, we cannot, as in figure \ref{fig:not_gap_high_d}.
Outside of these data sets, toy data sets are often used, but rarely consistently across papers.
With this in mind, a standard benchmark is introduced in section \ref{sec:uncertainty_benchmark} that can be both visualized easily, as well as scaled to higher dimensions, and makes comparing different aspects of uncertainty across models easy.

Task dependent benchmarks are also given, and these benchmarks are used in section \ref{sec:bayesopt_experiments} for Bayesian Optimization.
Benchmark functions used are functions that are difficult or impossible to optimize with traditional methods.
Hyperparameter searches for support vector machines and logistic regression classification are common benchmarks (\cite{Snoek}).
Function with many local optima, such as the Branin (\cite{}) and Hartmann (\cite{}) functions are also used.
There is little work on understanding what aspect of a model's uncertainty would lead to good or bad performance.




\chapter{Uncertainty-Aware (UNA) Bases for Bayesian Regression}\label{ch:4}
\begin{figure}
    \centering
    \includegraphics[width=0.5\linewidth]{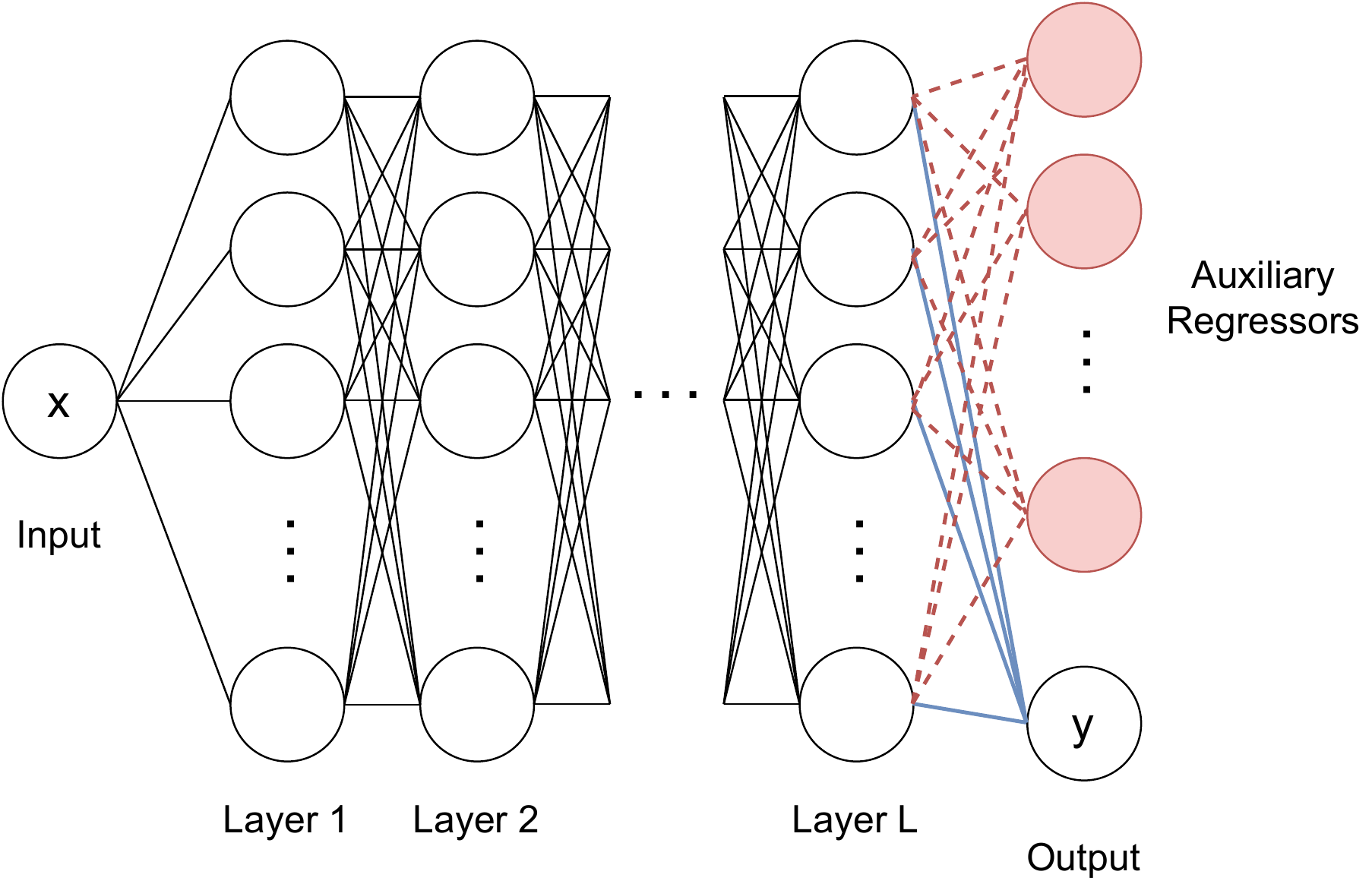}
    \caption{The UNA framework encompasses both LUNA and TUNA. Here we have the red auxiliary regressors, indicating they are discarded at prediction time. Just as in the NLM, the blue lines indicate a bayesian prior placed on the last layer for bayesian regression.}
    \label{fig:lb_nlm}
\end{figure}
First, the expressivity of the Neural Linear Model is explored.
Theoretical reasons for a lack of expressivity are given.
Next, we introduce a novel training framework that avoids the  failure  modes  of  the  three  traditional  NLM  training methods, allowing us to learn models with expressive posterior predictive distributions that can distinguish between data-poor and data-rich regions.
Furthermore, our frame-work allows domain experts to easily and intuitively encode task-specific knowledge into the feature basis.
To learn NLM feature bases that span a diverse set of functions in the prior predictive, we propose a novel training framework, UNcertainty Aware Training (UNA), in which we directly train the feature map $\phi_\theta$ to encode for functional diversity. UNA consists of two steps:

\textbf{Step 1: Feature Training with Diverse and Task-Appropriate Auxiliary Regressors}. We train $M > 1$ auxiliary linear regressors, $f_m(\mathbf{x}) = \bm{\Phi}_\theta\mathbf{\widetilde{w}}_m$, on a shared feature basis $\bm{\Phi}_\theta$ (see illustration in Figure \ref{fig:lb_nlm}).
If the auxiliary regressors $f_m(x)$ are trained to be diverse, then the shared feature basis will support a diverse prior predictive distributions over functions. Furthermore, we can impose constraints on the auxiliary regressors expressing domain knowledge. 

\textbf{Step 2: Bayesian Linear Regression on Features.} After optimizing the feature map, \emph{we discard the $M$ auxiliary regressors} $\mathbf{\widetilde{w}}_m$ and perform Bayesian linear regression on the expressive feature basis $\bm{\Phi}_\theta$ learned in Step 1. That is, we infer the posterior $p(\mathbf{w} | \mathcal{D}, \theta)$.

In the following, we describe a concrete instantiation of our framework: LUNA, with a training objective for Step 1 of the framework that is suited for capturing in-between uncertainties.
In Section \ref{sec:tuna}, we describe a second instantiation: TUNA, with a Step 1 training objective that can easily and intuitively encode functional or domain knowledge. 

\section{Neural Linear Model Uncertainty}
\label{sec:nlm_pathologies}
In Proposition \ref{thm:marg_blowup}, we also show that without the regularization term $\gamma\norm{\theta}_2^2$, the features $\bm{\Phi}_\theta$ experience pathological blow-up for ReLU networks, since large $\bm{\Phi}_\theta$ reduce the magnitude of the posterior mean $\mathbf{w}_N$, and hence increase $\mathcal{L}_{\mathrm{Marginal}}$.

Like MAP training, Marginal Likelihood training (Eqn. \ref{eq:marginal_ll}) discourages learning models with expressive prior predictive when $\gamma > 0$.
In Appendix \ref{sec:appendix_nlm} Figure \ref{fig:marg_highgamma}, we show that with $\gamma > 0$, the feature bases learned by optimizing $\mathcal{L}_{\mathrm{Marginal}}$ do not span diverse functions across random restarts; hence the corresponding posterior predictive distributions are inexpressive and are unable to capture in-between uncertainty. In Appendix \ref{sec:appendix_nlm} Figure \ref{fig:marg_lowgamma} we show that even with $\gamma$ set close to zero, the learned feature bases are rarely expressive across random restarts.

Without regularization, we show that Marginal Likelihood training suffers from a new failure: the feature map blows up for ReLU networks. 
Intuitively, increasing the magnitude of $\bm{\Phi}_\theta$ by a scalar multiple allows us to decrease the weights of the last layer by the same multiple with no loss to the likelihood, and thus we can trivially increase $\mathcal{L}_{\mathrm{Marginal}}$ by scaling the feature basis $\bm{\Phi}_\theta$. We formalize this intuition in the proposition below.

We choose ReLU activations functions. For $\theta$, $\mathbf{w}$ and any $c >0$, we define $\theta^c$ as equal to $\theta$ up to the last layer of weights, and at the last layer equal to the last layer of $\theta$ scaled by $c$. We also define $\mathbf{w}^c$ as $\frac{1}{c} \mathbf{w}$.
\begin{proposition}
\label{thm:marg_blowup}
Fix $\theta$, $\mathbf{w}$. For sufficiently large $C>0$ and any $c> C$, we have that $$\mathcal{L}_{\mathrm{Marginal}}(\theta_\mathrm{Full}) < \mathcal{L}_{\mathrm{Marginal}}(\theta^c_\mathrm{Full}),$$ 
where $\theta_\mathrm{Full} = (\theta, \mathbf{w})$ and $\theta^c_\mathrm{Full} = \left(\theta^c, \mathbf{w}^c\right)$.
\end{proposition}

\begin{proof}
To demonstrate this, we first assume that $\bm{\Phi}_\theta^T\bm{\Phi}_\theta$ is invertible. Let us establish the relationship between $\mathbf{w}_N$ and $\bm{\Phi}_\theta$ in the asymptotic case $\norm{\bm{\Phi}_\theta}\to\infty$.
From Eq \ref{eq:nlm_posterior}, 
\begin{align*}
    \frac{1}{\sigma^2}\bm{\Phi}_\theta^T\bm{\Phi}_\theta \gg \frac{1}{\alpha}\mathbf{I}_{M\times M} 
    &\implies \mathbf{V}_N^{-1}\to\frac{1}{\sigma^2}\bm{\Phi}_\theta^T\bm{\Phi}_\theta\\
    &\implies \mathbf{V}_N\to\sigma^2\left(\bm{\Phi}_\theta^T\bm{\Phi}_\theta\right)^{-1}
\end{align*}
Hence,
$$\mathbf{w}_N\to\left(\bm{\Phi}_\theta^T\bm{\Phi}_\theta\right)^{-1}\bm{\Phi}_\theta^T\mathbf{y}\implies\norm{\mathbf{w}_N}\thicksim\frac{1}{\norm{\bm{\Phi}_\theta}}.$$
Note that the loss $\| \mathbf{y} -\bm{\Phi}_\theta\mathbf{w}\|$ is equal to $\| \mathbf{y} -\bm{\Phi_{\theta^c}}\mathbf{w}^c\|$, since $\bm{\Phi_{\theta^c}}$ is $\bm{\Phi_{\theta}}$ scaled by $c$ and this scaling is canceled by $\mathbf{w}^c = \frac{1}{c} \mathbf{w}$. 
Thus, since $\|\mathbf{w}^c_N\| < \|\mathbf{w}_N\|$, we have that $\mathcal{L}_{\mathrm{Marginal}}(\theta_\mathrm{Full}) < \mathcal{L}_{\mathrm{Marginal}}(\theta^c_\mathrm{Full})$.
\end{proof}
The above proposition tells us that that we can continue to increase $\mathcal{L}_{\mathrm{Marginal}}$ by reducing $\norm{\mathbf{w}_N}$. Hence, if we do not regularize $\theta$, the training will continually increase $\norm{\bm{\Phi}_\theta}$ to affect a decrease in $\norm{\mathbf{w}_N}$.

However, the addition of the regularization term $\gamma\norm{\theta}_2^2$ to $\mathcal{L}_{\mathrm{Marginal}}$ biases training towards inexpressive feature bases for the same reason we identified previously. In Figure \ref{fig:marg_highgamma}, we show that with regularization, the feature bases learned by optimizing $\mathcal{L}_{\mathrm{Marginal}}$ are inexpressive. In Figure \ref{fig:marg_lowgamma} we see that even with $\gamma$ set close to zero, the learned feature bases are not consistently expressive across random restarts.

Lastly, Marginal-Likelihood training consists of maximizing the likelihood of the observed data, after having marginalized out the weights $\mathbf{w}$:
\small
\begin{equation}\label{eq:marginal_ll}
\mathcal{L}_{\text{Marginal}}(\theta) = \log{\mathbb{E}_{p(\mathbf{w})} \left\lbrack \mathcal{N}\left(\mathbf{y}; \bm{\Phi}_\theta\mathbf{w}, \sigma^2\mathbf{I}\right) \right\rbrack} - \gamma\norm{\theta}_2^2
\end{equation}
\normalsize

In this paper, we show that \emph{all three} inference methods for learning the feature basis (determined by $\theta$) in Step 1 produce models that are unable to distinguish between data-poor and data-rich regions (i.e. these models fail to capture in-between uncertainty).
In Chapter \ref{ch:4} we then propose a novel framework for training NLMs that learns models capable of expressing in-between uncertainty, and that allows domain experts to tailor posterior predictive uncertainties for specific tasks.

\section{Learned Uncertainty-Aware (LUNA) Bases}
\label{sec:luna_main_body}
In the first instantiation of UNA, we choose a training objective $\mathcal{L}_\text{LUNA}$ for Step 1 that maximizes the mean log-likelihood of the auxiliary regressors on the training data, measured by $\mathcal{L}_\text{FIT}$, while encouraging for functional diversity amongst them, measured by $\mathcal{L}_\text{DIVERSE}$ (defined later):
\begin{equation}
\mathcal{L}_\text{LUNA}(\Psi) = \mathcal{L}_\text{FIT}(\Psi) - \lambda \cdot \mathcal{L}_\text{DIVERSE}(\Psi) 
\end{equation}
where $\Psi = (\theta, \mathbf{\tilde{w}}_1, \hdots \mathbf{\tilde{w}}_M)$, $\theta$ parameterizes the shared designed matrix and $\mathbf{\tilde{w}}_m$ are the weights of the auxiliary regressor $f_m =  \bm{\Phi}_\theta\mathbf{\tilde{w}}_m$. The constant $\lambda$ controls for the degree to which we prioritize diversity. We encourage for diversity in the regressors trained on our basis, since our analysis in Section \ref{sec:nlm_pathologies} shows that if the feature basis spans diverse functions under the prior $p(\mathbf{w})$, the posterior predictive can capture in-between uncertainty.

After optimizing our feature map via: 
$$
\theta_\text{LUNA}, \{\mathbf{\tilde{w}}_{m}^*\}  = \mathrm{argmax}_\Psi\; \mathcal{L}_\text{LUNA}(\Psi),
$$
we discard the auxiliary regressors $\{\mathbf{\tilde{w}}_{m}^*\}$ and perform Bayesian linear regression on the diversified feature basis, the LUNA basis. That is, we analytically infer the posterior $p(\mathbf{w} | \mathcal{D}, \theta_{\text{LUNA}})$ over the last Bayesian layer of weights $\mathbf{w}$ in the NLM.  
\emph{\textbf{In summary, LUNA training results in a basis that supports a diverse set of predictions by varying $\mathbf{w}$.}}

\textbf{$\mathcal{L}_\text{FIT}$: Fitting the Auxiliary Regressors.}
We learn the regressors jointly with $\bm{\Phi}_\theta$, by maximizing the average train log-likelihood of the regressors on the training data, with $\ell_2$ penalty on $\theta$ as well as on the weights of each regressor:
\begin{equation*}
\mathcal{L}_\text{FIT}(\Psi) = \frac{1}{M}\sum_{m=1}^M \log \mathcal{N}(\mathbf{y}; f_m(\mathbf{x}), \sigma^2 \mathbf{I})- \gamma\norm{\Psi}_2^2.
\end{equation*}

\textbf{$\mathcal{L}_\text{DIVERSE}$: Enforcing  diversity.}
We enforce diversity in the auxiliary regressors as a proxy for the diversity of the functions spanned by the feature basis. 
We adapt the Local Independence Training (LIT) objective in \cite{diversity_enforcement} to encourage our regressors to extrapolate differently away from the training data, where $\text{CosSim}$ is cosine similarity:
\begin{equation*}
\mathcal{L}_\text{DIVERSE}(\Psi) = 
\sum_{i=1}^M\sum_{j=i+1}^M
\text{CosSim}^2\left(\nabla_{\mathbf{x}}f_i(\mathbf{x}), \nabla_{\mathbf{x}}f_j(\mathbf{x}) \right).
\label{eq:luna_obj}
\end{equation*}

Here we encourage extrapolation difference in every pair of regressors $f_i$ and $f_j$ by penalizing non-orthogonal gradients. Furthermore, 
we avoid expensive gradient computations using a finite difference approximation.
\subsection{Diversity Penalty}
We adopted the diversity penalty in LUNA's objective from \cite{diversity_enforcement}.
We use the cosine similarity function on the gradients of the auxiliary regressors:
\begin{equation*}
    \begin{split}
        \text{CosSim}^2&\left(\nabla_{\mathbf{x}} f_i(\mathbf{x}), \nabla_{\mathbf{x}} f_j(\mathbf{x}) \right) = \\
        &\frac{\left(\nabla_{\mathbf{x}} f_i(\mathbf{x})^T \nabla_{\mathbf{x}} f_j(\mathbf{x})\right)^2}
        {\left(\nabla_{\mathbf{x}} f_i(\mathbf{x})^T  \nabla_{\mathbf{x}} f_i(\mathbf{x})  \right) \left( \nabla_{\mathbf{x}} f_j(\mathbf{x})^T \nabla_{\mathbf{x}} f_j(\mathbf{x})\right)}
    \end{split}
\end{equation*}
This acts as a measure of orthogonality, equal to one when the two inputs are parallel, and 0 when they are orthogonal.
A higher penalty, $\lambda$, in the training objective penalizes parallel components, hence enforcing diversity.

In practice, these gradients can be computed using a finite differences approximation.
That is, we approximate gradients as:
\begin{equation*}
    \frac{\partial f_i(\mathbf{x})}{\partial x_d} \approx \frac{f_i\left(\mathbf{x} + \delta\mathbf{x}_d\right) - f_i\left(\mathbf{x}\right)}{\delta x_d}
\end{equation*}
where $\delta\mathbf{x}_d = [0,\hdots,0,\delta x_d,0, \hdots,0]^T$ represents a D-dimensional vector of zeros with a small perturbation in the $d^\text{th}$ dimension. 
We sample these perturbations according to $\delta x_d \sim \mathcal{N}(0, \epsilon)$, where $\epsilon$ is often either 0.1 or 0.01.
The perturbation variance

In the experiments, LUNA's diversity penalty was often annealed according to a number of different schedules. The diversity penalty was also scaled by a factor $C = 2n/(m(m-1))$, where $n$ is batch size, and $m$ is the number of auxiliary regressors.
Where $N$ is the number of epochs, the three schedules tested were $f_{sqrt}(x) = C\sqrt{x/N}$, $f_{sigmoid} = C/(1 + exp(-6x/N + 3))$, and $f_{tanh}(x) = C(\tanh(6x/N - 3) + 1)/2$.

\begin{figure}[H]
    \centering
    \includegraphics[width=0.8\linewidth]{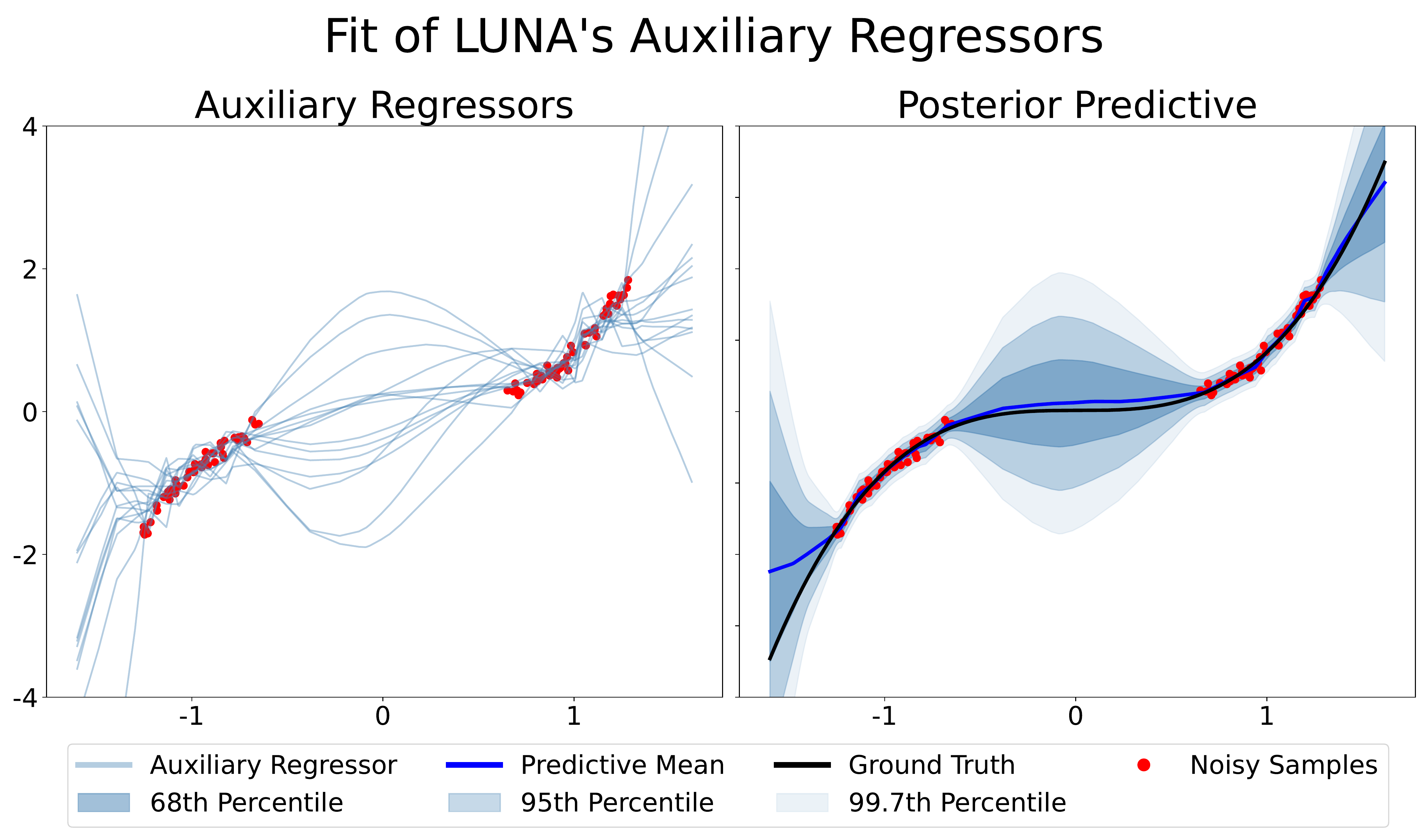}
    \caption{LUNA's auxiliary regressors are seen here to fit the data and extrapolate differently away from the data. This leads to good posterior predictive uncertainty.}
    \label{fig:luna_aux}
\end{figure}

If we use a large number of auxiliary regressors, we can run into issues where not all of them fit the data cleanly, as seen in Figure \ref{fig:luna_aux_bad}.
The Bayesian regression layer is able to use the diversity of the auxiliary regressors to still fit the data wiht good predictive uncertainty.
\begin{figure}[H]
    \centering
    \includegraphics[width=0.8\linewidth]{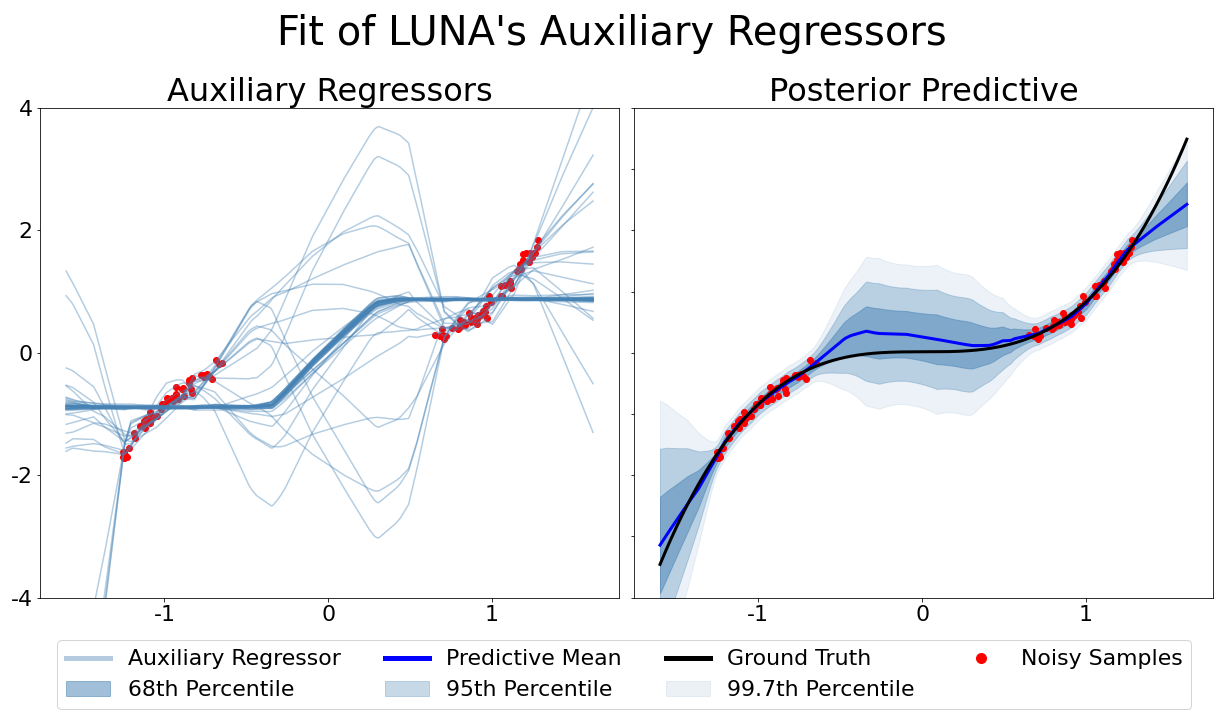}
    \caption{LUNA has good predictive uncertainty even when not all of the auxiliary regressors fit the data.}
    \label{fig:luna_aux_bad}
\end{figure}
The diversity penalty is useful for model selection as well.
By scaling the diversity penalty by the number of auxiliary regressors, the diversity can be compared across architectures and number of auxiliary regressors.
That is, after training, the diversity penalty is calculated on the validation set and caled by $1/{M \choose 2}$, the number of combinations of auxiliary regressors.
This is used later in the UCI gap experiment in section \ref{sec:uci_gap}.

\section{Tuned Uncertainty-Aware (TUNA) Bases}
\label{sec:tuna}

LUNA bases span functions that both fit the observed data and generalize differently in data-poor regions; these bases are useful when we simply want to capture in-between uncertainty. Now, we propose an alternative instantiation of our training framework, TUNA, that optimizes NLM prior predictive distributions to match any functional prior or, indeed, any set of reference functions. In the context of a specific application, by selecting a target functional prior or a set of reference functions with desired task-relevant properties (e.g. by drawing functions from a GP prior with task-appropriate kernel), TUNA's posterior predictive is ``automatically tuned'' for the downstream task. 

For TUNA, Step I of the framework is as follows:

1. \textbf{Select a set of reference points:} We select a set of inputs $\{\widetilde{\mathbf{x}}_i\}_{i=1}^I$ on which we want the NLM's prior predictive to match our target prior predictive. This set of points may include training inputs as well as additional inputs. In Chapter \ref{ch:5}, $\{\widetilde{\mathbf{x}}_i\}_{i=1}^I$ includes training inputs as well as the set of training inputs perturbed with Gaussian noise of a fixed variance $\sigma^2_x$. 

2. \textbf{Select reference functions encoding task-specific desiderata:} We evaluate a set of samples of functions $\{g_m\}_{m=1}^M$ from our target prior predictive or reference function set at each $\widetilde{\mathbf{x}}_{i}$. This produces $M$ sets $\{\widetilde{\mathbf{x}}_{i}, g_m(\widetilde{\mathbf{x}}_i)\}_{i=1}^I$. 

3. \textbf{Each auxiliary regressor learns a reference function:} We train the $M$ auxiliary regressors to capture the $M$ reference functions by minimizing the following objective with respect to the basis parameters, $\Psi = (\theta, \mathbf{\widetilde{w}}_1, \dots, \mathbf{\widetilde{w}}_M)$:
\begin{align}
    \mathcal{L}_\text{TUNA}(\Psi) &= \frac{1}{M} \sum\limits_{m=1}^M || g_m(\widetilde{\mathbf{X}}) - \Phi_\theta \mathbf{\widetilde{w}}_m ||^2_2
\end{align}
where the $i$-th row in $\widetilde{\mathbf{X}}\in \mathbb{R}^{I\times D}$ is $\widetilde{\mathbf{x}}_i$. 
\begin{figure}[H]
    \centering
    \includegraphics[width=\linewidth]{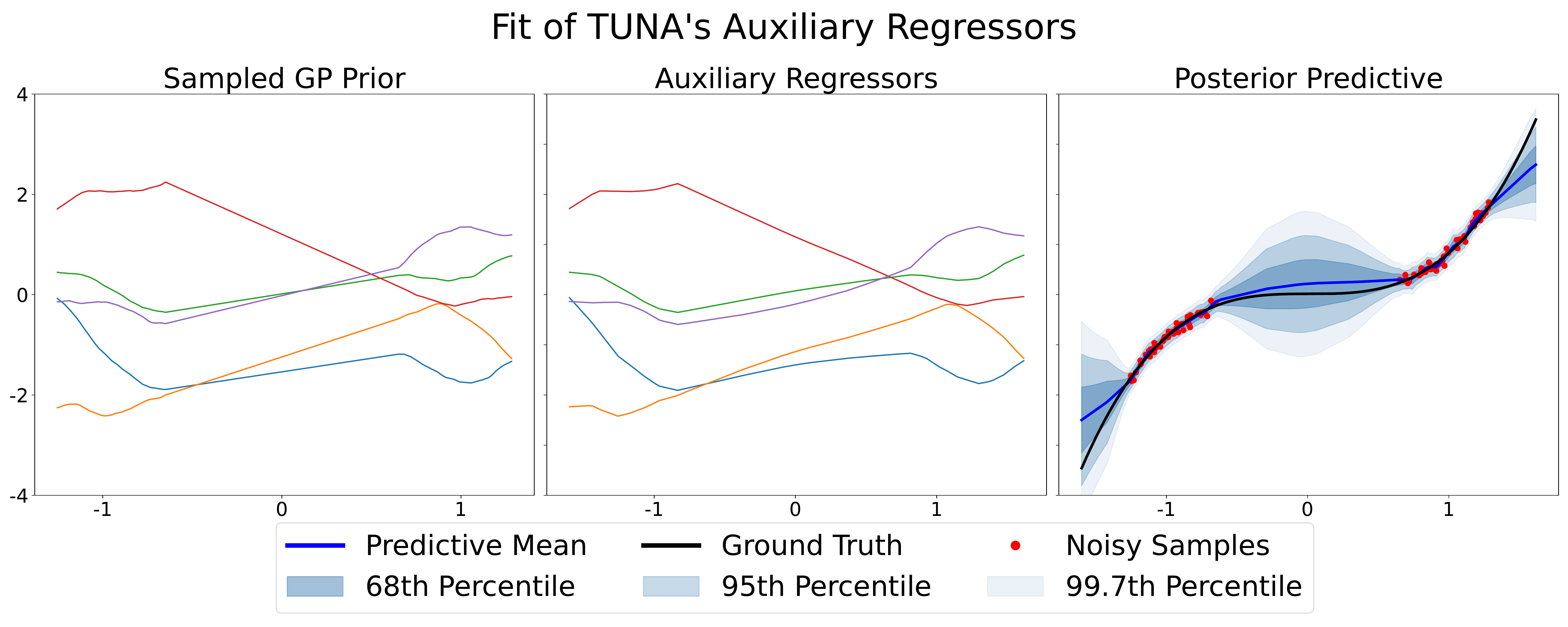}
    \caption{TUNA's auxiliary regressors and the sampled reference functions are almost indistinguishable. This diversity in auxiliary regressors leads to good predictive uncertainty. Plotted are the first five (of 40) regressors. The sampled reference functions correspond to the auxiliary regressor of the same color.}
    \label{fig:tuna_aux}
\end{figure}

In Chapter \ref{ch:5}, we show that through training the auxiliary regressors against the reference functions, we are able to encode task-specific functional knowledge in our feature basis. In particular, when $g_m$ is sampled from a target functional prior $p(g)$, we show that $\mathcal{L}_\text{TUNA}$ is an Empirical Bayes MAP Type II approximation of the KL divergence between $p(g)$ and the function prior implied by the NLM prior $p(\mathbf{w})$ over the last layer of weights. 
That is, \textbf{\emph{TUNA yields prior predictive distributions that well-approximates target functional priors, e.g. GP priors}}.
For example, in Section \ref{sec:encoding_uncertainty}, we show that the prior predictive distribution for a TUNA basis trained against a target GP prior captures salient features of the target prior; and in Figure \ref{fig:gp_tuna_strip}, we show that the posterior predictive distribution for such a TUNA basis captures salient features of the target GP posterior.

Again, after optimizing the feature map, \emph{we discard the auxiliary regressors} $\mathbf{\widetilde{w}}_m$ and perform Bayesian linear regression on the expressive feature basis $\boldsymbol\Phi_{\boldsymbol\theta}$.

\subsection{TUNA Training as a KL Approximation}
\label{sec:tuna-kl}

TUNA's training objective can be understood as minimizing the KL divergence between its target prior predictive, and the NLM's prior predictive.
We see in Figure \ref{fig:tuna_prior_pred_match}, TUNA's prior predictive is very similar to a GP's.
\begin{figure}[H]
    \centering
    \includegraphics[width=0.8\linewidth]{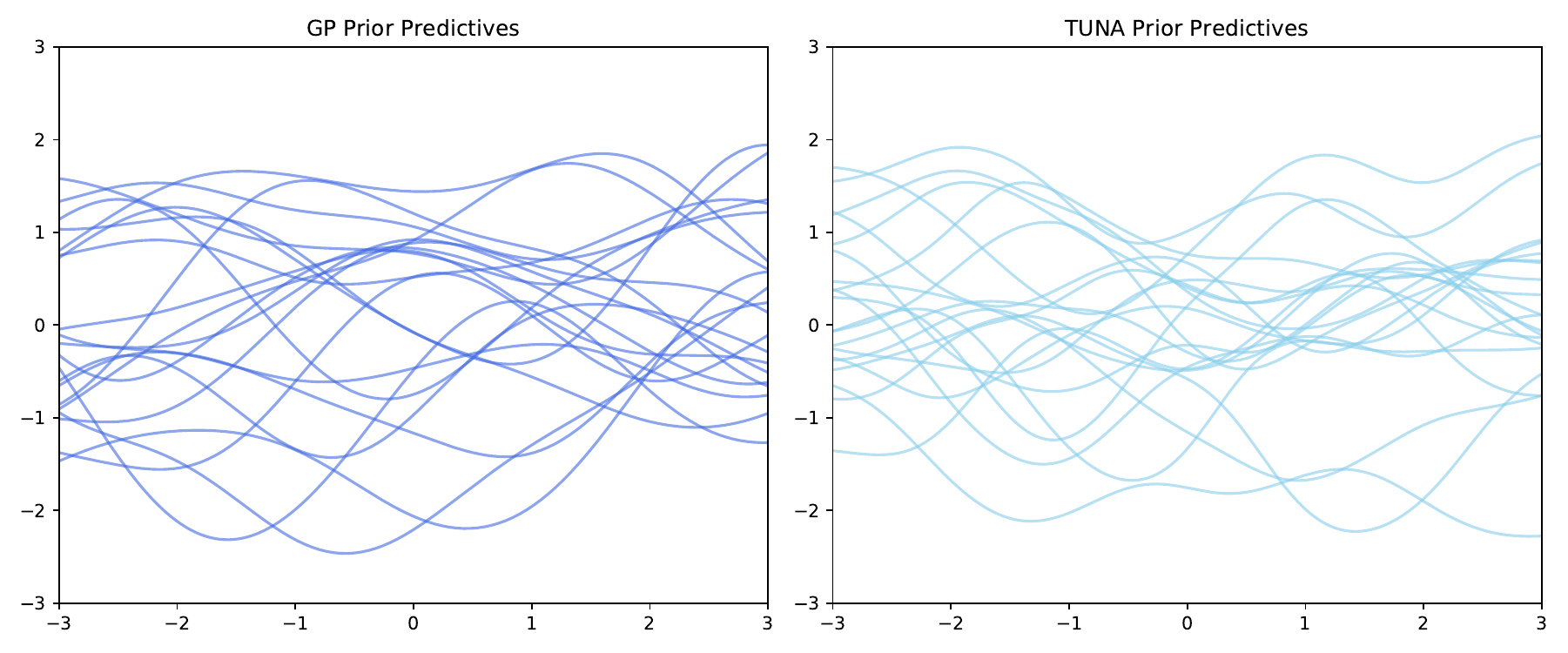}
    \caption{TUNA's prior predictive is very similar to the GP's.}
    \label{fig:tuna_prior_pred_match}
\end{figure}
We are given a set of inputs $\mathcal{I} = \{\mathbf{x}_i\}_{i=1}^I$ on which we want the NLM's prior predictive to match our target prior predictive.
We define the NLM's prior predictive as,
\begin{align*}
    p(\mathbf{y} | \mathbf{x}, \theta) &= \mathbb{E}_{p(\mathbf{w})} \lbrack \mathcal{N}(\mathbf{y}; \boldsymbol\Phi_{\boldsymbol\theta} \mathbf{w}, \sigma^2\mathbf{I}) \rbrack
\end{align*}
and let $q(\mathbf{y} | \mathbf{x})$ be the target prior predictive,
\begin{align*}
    q(\mathbf{y} | \mathbf{x}) &= \mathbb{E}_{p(g)} \lbrack \mathcal{N}(\mathbf{y}; g(\mathbf{x}), \sigma^2\mathbf{I}) \rbrack
\end{align*}
where $p(g)$ is a prior over function (e.g. a GP prior at $\mathcal{I}$). 
We note that the $\sigma^2$ used here to learn the TUNA basis need not be the same one used later in posterior inference. 
Our goal is to find $\theta_\text{TUNA}$ that matches the NLM's prior predictive with the target prior predictive:
\begin{align*}
    \theta_\text{TUNA} &= \mathrm{argmin}_{\theta}
    D_{\text{KL}} \left[ 
    q(\mathbf{y} | \mathbf{x})
    \middle\| 
    p(\mathbf{y} | \mathbf{x}, \theta)
    \right] \\
    &= 
    \text{argmin}_{\theta}
    \mathbb{E}_{q(\mathbf{y} | \mathbf{x})} \left[
    \log \frac{q(\mathbf{y} | \mathbf{x})}{p(\mathbf{y} | \mathbf{x}, \theta)}
    \right] \\
    &= 
    \text{argmin}_{\theta}
    \mathbb{E}_{q(\mathbf{y} | \mathbf{x})} \left[
    -\log p(\mathbf{y} | \mathbf{x}, \theta)
    \right] \\
    &\approx 
    \text{argmin}_{\theta}
    \frac{1}{M} \sum\limits_{m=1}^M
    -\log p(\mathbf{y}_m | \mathbf{x}, \theta)
\end{align*}
where $\mathbf{y}_m \sim q(\mathbf{y} | \mathbf{x}) $.
We then approximate $p(\mathbf{y}_m | \mathbf{x}, \theta)$ using an Empirical Bayes MAP Type II estimate as follows:
\begin{align}
    p(\mathbf{y}_m | \mathbf{x}, \theta) &= \mathbb{E}_{p(\mathbf{w})} \lbrack \mathcal{N}(\mathbf{y}_m; \boldsymbol\Phi_\theta \mathbf{w}, \sigma^2\mathbf{I}) 
    \rbrack \label{eq:before-eb} \\
    &\approx \mathcal{N}(\mathbf{y}_m; \boldsymbol\Phi_\theta \mathbf{\tilde{w}}, \sigma^2\mathbf{I})
\end{align}
where $\mathbf{\tilde{w}} = \mathrm{argmax}_{\mathbf{w}} \mathcal{N}(\mathbf{y}_m; \boldsymbol\Phi_\theta \mathbf{w}, \sigma^2\mathbf{I})$.
This approximation is accurate when the expectation in Equation \ref{eq:before-eb} is dominated by a single $\mathbf{w}$ (i.e. $\mathbf{\tilde{w}}$ explains the samples from the target prior predictive well). 
Plugging this approximation back into the objective, we have:
\begin{align*}
    \theta_\text{TUNA}, \mathbf{\tilde{w}} &\approx \mathrm{argmin}_{\theta, \mathbf{w}}
    \frac{1}{M} \sum\limits_{m=1}^M
    -\log \mathcal{N}(\mathbf{y}_m; \boldsymbol\Phi_\theta \mathbf{w}, \sigma^2\mathbf{I}) \\
    &= 
    \mathrm{argmin}_{\theta, \mathbf{w}} \frac{1}{M}\sum_{m=1}^M \norm{\mathbf{y}_m - \boldsymbol\Phi_\theta \mathbf{w}}_2^2
\end{align*}
Putting it all together, we obtain the TUNA training objective:
\begin{align*}
    \min_{\theta, \mathbf{w}} \frac{1}{M}\sum_{m=1}^M \norm{\mathbf{y}_m - \boldsymbol\Phi_\theta \mathbf{w}}_2^2, \quad \mathbf{y}_m \sim q(\mathbf{y} | \mathbf{x})
\end{align*}
We note that we need not be able to evaluate $q(\mathbf{y} | \mathbf{x})$, which allows us to to use reference functions as opposed to a prior predictive. 

\chapter{Experiments}\label{ch:5}

Because 'good' uncertainty is task dependent, UNA must be benchmarked across many different tasks to determine its utility.
The first such task is a novel benchmark created as a way to easily compare multiple aspects of uncertainty across models.
Next, the models are benchmarked on different toy data sets.
The toy data sets can be easily visualized and offer an intuitive way to understand more subtle aspects of uncertainty.
After that, LUNA is compared against benchmark models on UCI \cite{uci_data} data sets.
Both standard and gap \cite{uci_gap} sets are used.
The downstream task of Bayesian Optimization, that relies on using uncertainty, is used as a functional benchmark.
Lastly, concept learning is incorporated in order to encode knowledge into UNA's uncertainty.

\section{Radial Uncertainty Benchmark}
\label{sec:uncertainty_benchmark}

As discussed earlier in \ref{sec:uncertainty}, what constitutes good uncertainty is dependent on the task at hand, as well as prior knowledge.
Here the Radial Uncertainty Benchmark (RUB) is proposed and UNA, as well as other benchmark models, are compared against our gold-standard method, GPs.
The aim of this benchmark is to provide a fast, easy to understand experiment that captures multiple facets of 'good' uncertainty.

The RUB works by first samples data in a way that creates a gap in the middle.
The data $\mathbf{x}$ is sampled such that $1 \leq \left\lVert\mathbf{x} \right\rVert_2 \leq 2$, with ground truth $f(\mathbf{x}) = \left\lVert\mathbf{x} \right\rVert_2 + \epsilon$, and $\epsilon \overset{\text{iid}}{\sim} \mathcal{N}(0, 10^{-5})$.
This function has signal that is simple enough we expect all of our models to be able to learn it.
In 1D, we sample 50 data points, in 2D we sample 200 data points, and in 3D we sample 500 data points.
Each model is then trained on this sampled data.
In order to capture both predictive uncertainty mean and smoothness, we then sample radially distributed rays outward from the origin.
Predictive uncertainty is measured along those rays, and then the rays are averaged over distance from the origin.

In this case, we say 'good' uncertainty increases smoothly away from the data, matches the symmetry of the data, and has uncertainty that is between signal and noise in uncertainty.
Since data is radially symmetric, ideal uncertainty should be as well.
A 'reasonable' increase in uncertainty here can be defined as proportional to the difference in volume between the gap and data regions.
Since we are comparing volumes of hypersphere of differing radii, the volume ratios between data and gap regions in $D$ dimensions is $1/2^D$.
This leads to 'ideal' uncertainties of $0.5$, $0.25$, and $0.125$ for one, two, and three dimensions, respectively.
This volume ratio is compared to the $99.7$th percentile uncertainty because that level captures all functions that we can reasonably expect from a given data set.
Because there is signal surrounding the gap region and we expect the uncertainty to be smooth, peak uncertainty should be slightly less than this as the uncertainty meets in the center.
One downside of this method is it requires many forward passes from each model.
Models such as BBVI that use sampled weights and biases are prohibitively memory-intensive for this task.

\begin{figure}[H]
    \centering
    \includegraphics[width=\linewidth]{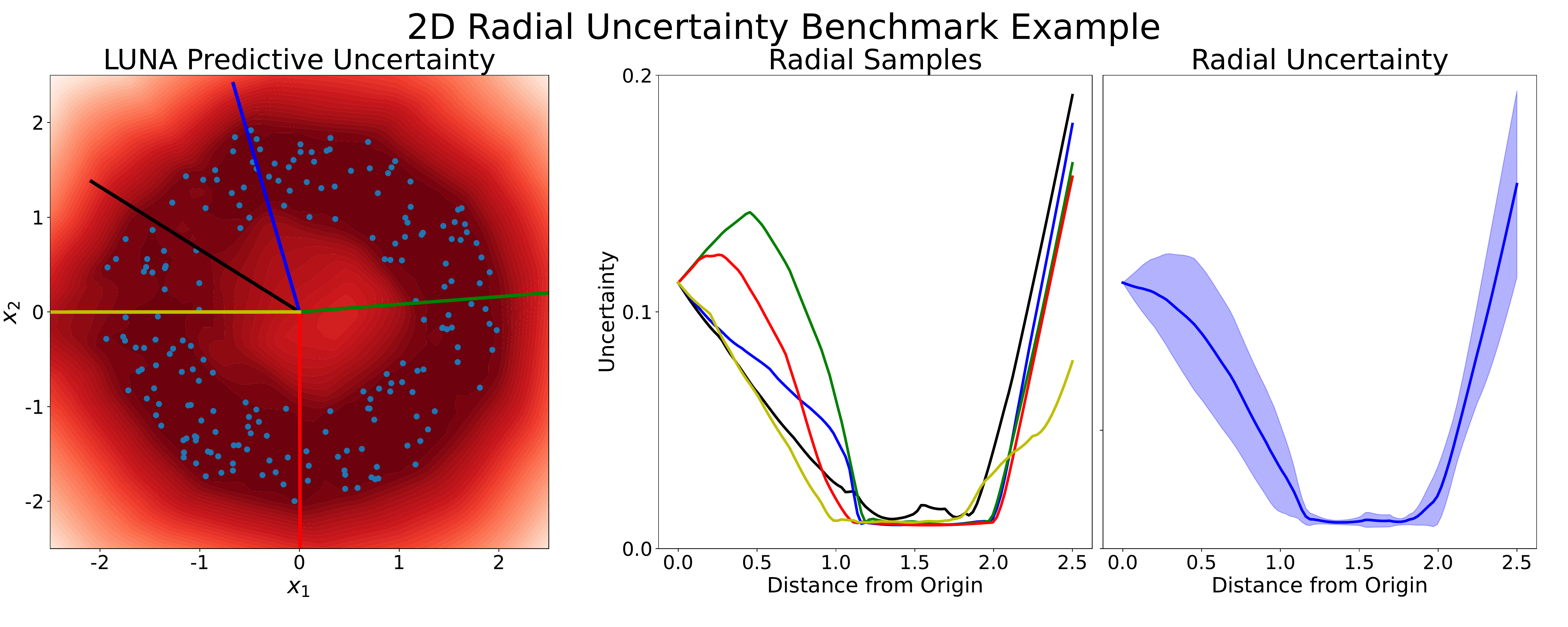}
    \caption{The radial uncertainty benchmark samples radially distributed vectors over which predictive uncertainty is calculated. These sampled vectors are averaged over distance from origin. The resulting mean and standard deviation provide a useful look at predictive uncertainty.}
    \label{fig:radial_uncertainty_example}
\end{figure}
This can be seen in Figure \ref{fig:radial_uncertainty_example}.
The green ray passes through an area of high uncertainty in the gap, which is seen in the middle plot as the peak uncertainty at about 0.75.
The yellow ray passes through a region of low uncertainty outside of the data region, which is seen as the lowest uncertainty at the highest distance from the origin.
These five samples are averaged, leading to the rightmost plot.
Since this is only five samples, the standard deviation is quite high, higher than we would expect, since the predictive uncertainty is fairly radially uniform.
In practice, more samples are taken.
In 1D, we simply take two samples, since there is only left and right.
In 2D and 3D, we take 1000 samples to get good coverage of the domain.
The exact experimental details are in Appendix \ref{sec:exp-setup}

\begin{figure}[h]
    \centering
    \includegraphics[width=\linewidth]{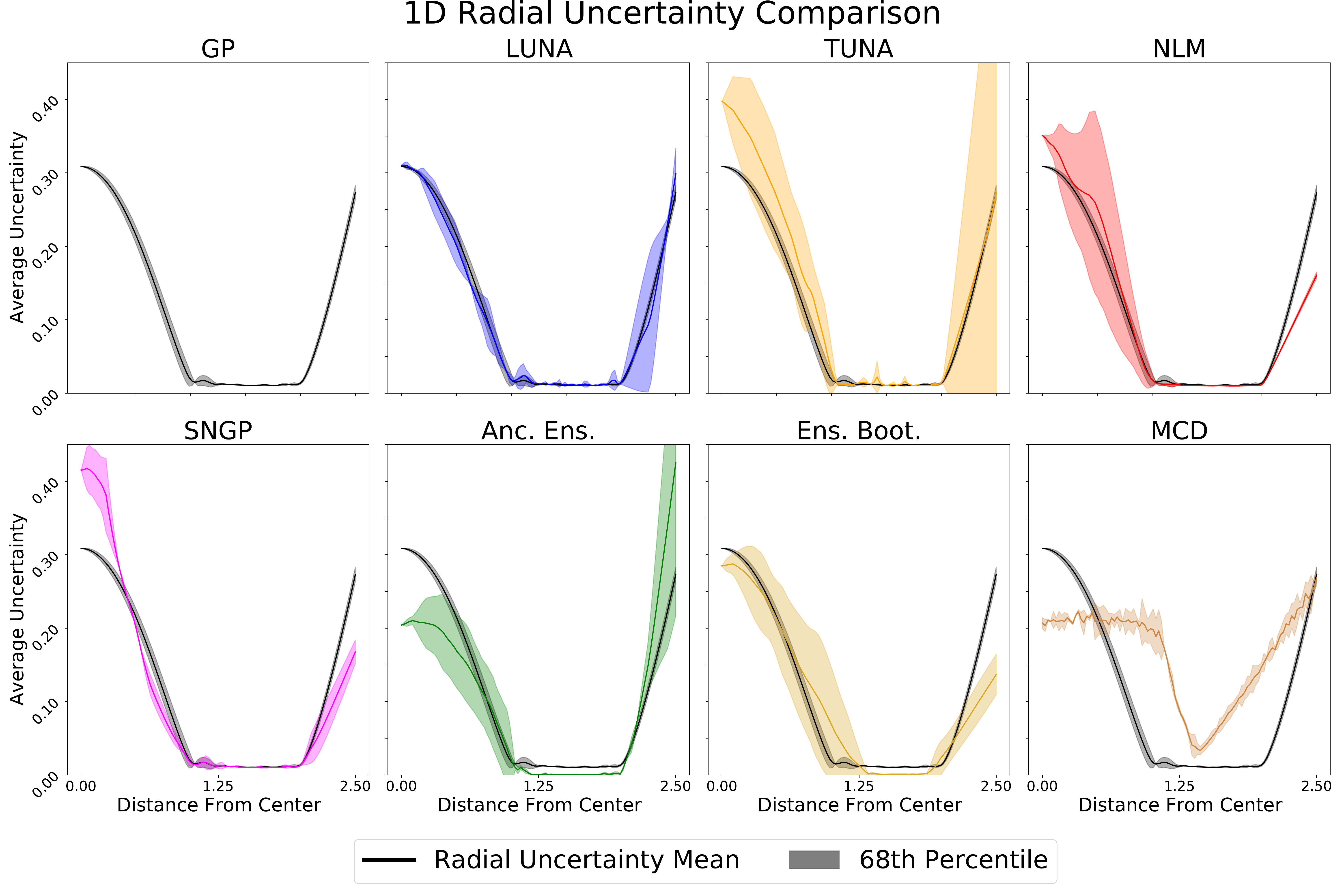}
    \caption{In one dimension, we see every method except for MCD is able to mimic GPs performance quite accurately. We also see TUNA has noisier predictive uncertainty.}
    \label{fig:ub_1d}
\end{figure}
For the 1D case, we see every method except MCD is able to reproduce the GP result with a high degree of accuracy.
On such a simple task, this is not unexpected, and we cannot reasonably rule out any methods except MCD as having good, GP-like uncertainty.

\begin{figure}[h]
    \centering
    \includegraphics[width=\linewidth]{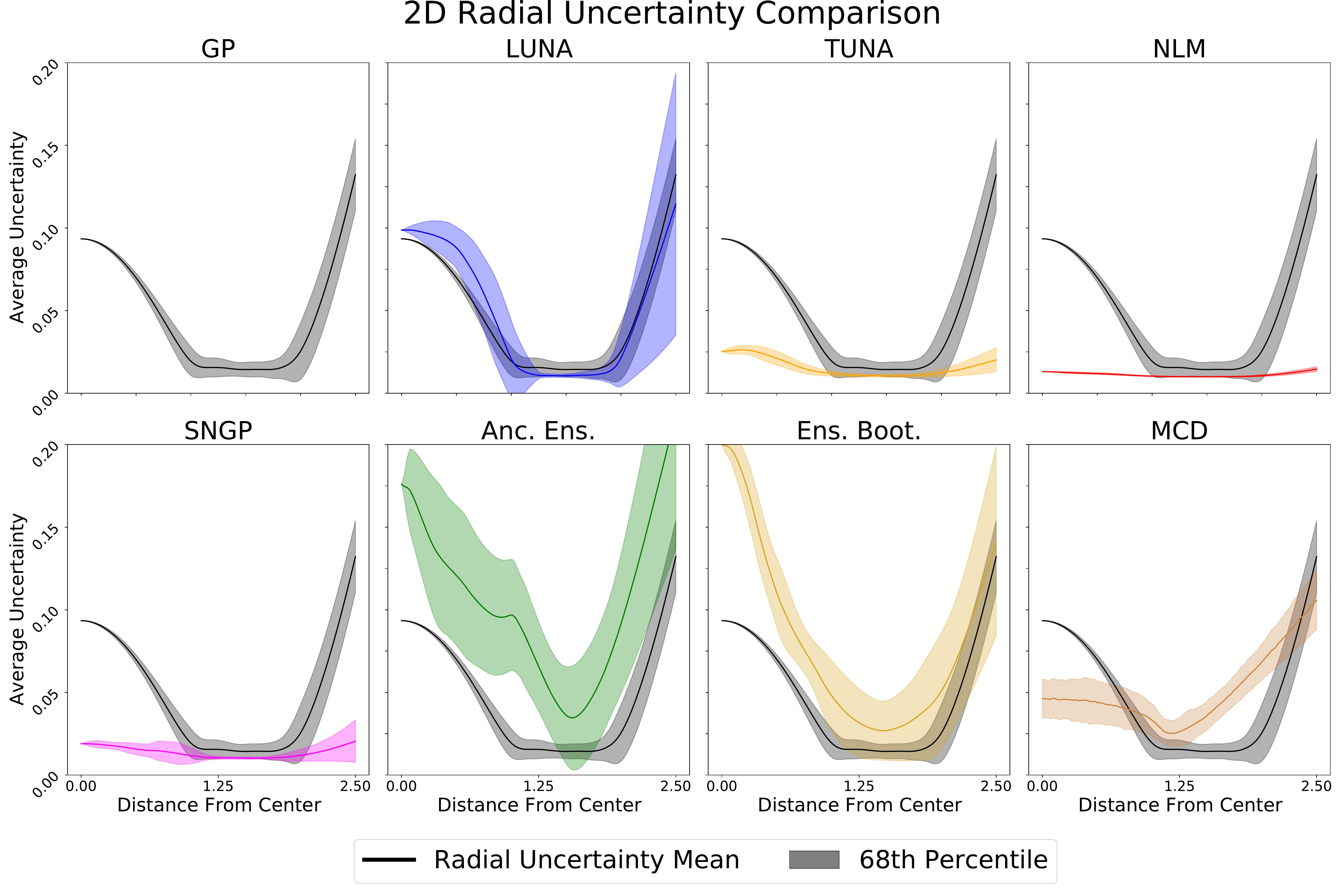}
    \caption{In two dimensions, we see that GP has slightly more variance in uncertainty. LUNA is the only method to closely mimic the results of GP.}
    \label{fig:ub_2d}
\end{figure}
In higher dimensions, we see fewer models are able to closely reproduce the GP result.
The Ensemble methods start to overestimate uncertainty, and have significantly more noise.
TUNA, NLM, and SNGP all have predictive uncertainty that collapses, even where there's no data.
LUNA is the only method that is able to produce uncertainty close to that of the GP.

\begin{figure}[h]
    \centering
    \includegraphics[width=\linewidth]{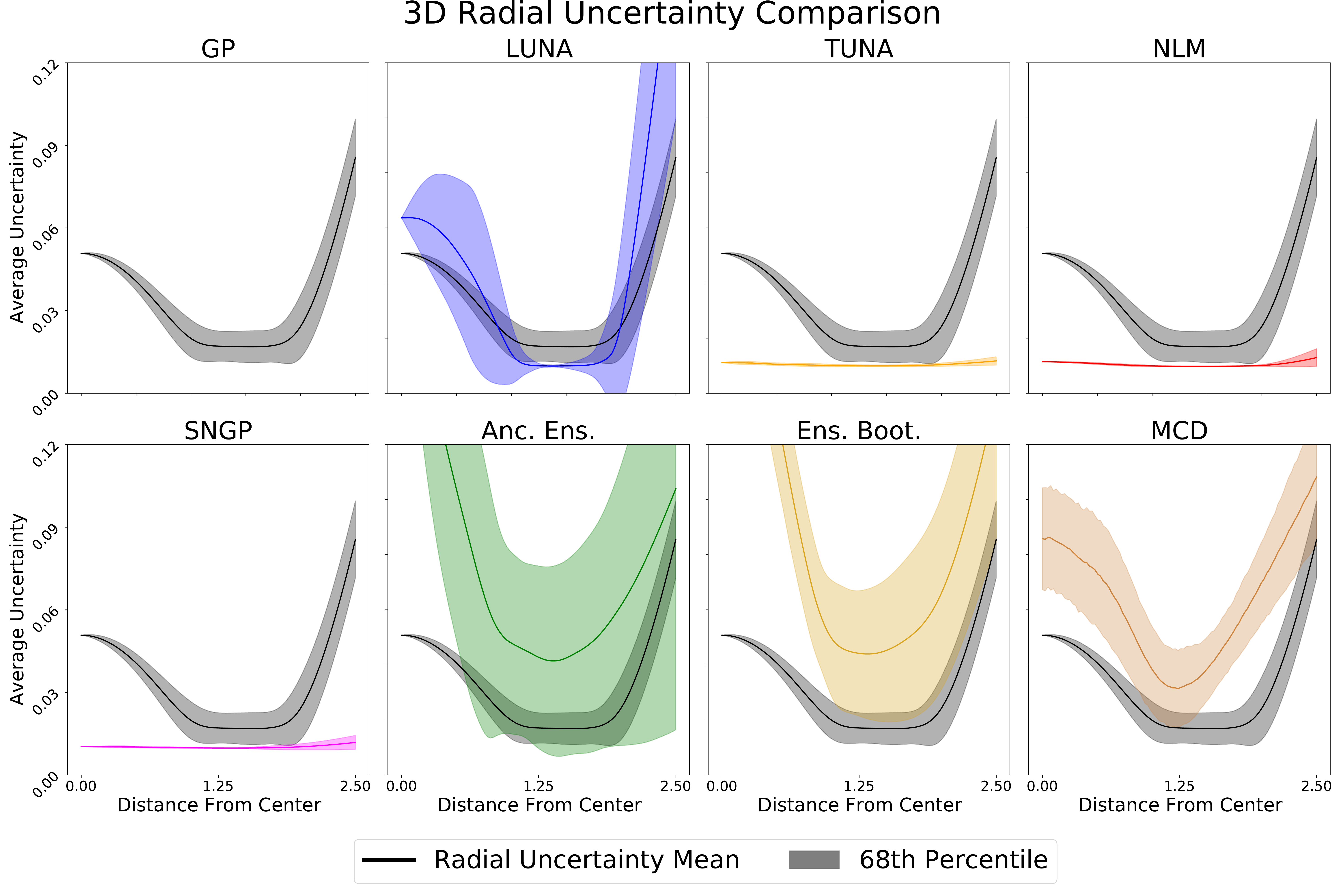}
    \caption{In three dimensions, we see that GP has slightly more variance in uncertainty. LUNA is the only method to closely mimic the results of GP.}
    \label{fig:ub_3d}
\end{figure}
In higher dimensions, the difference is even more pronounced.
GP prediction is slightly noisier. 
LUNA has noticeably more noise in its predictive uncertainty than the GP, but still is able to produce similar uncertainties.
TUNA, NLM, and SNGP have all but completely collapsed their uncertainty.
The ensemble methods are mostly dominated by noise, and tend to overestimate uncertainty to a great degree.
On this task, LUNA is the only method able to consistently match GP predictive uncertainty as dimensionality is increased.

\section{Toy Data}
Toy data sets and experiments can be designed to test and evaluate specific aspects of a model's performance.
Here, many such experiments are conducted in order to understand UNA's performance compared to other methods.
More examples are given in Appendix \ref{sec:appendix_toy_data_set}.

\subsection{Cubic Gap}
A very useful toy experiment to help demonstrate good predictive uncertainty is the cubic gap set (\cite{Rasmussen}).
The ground truth function is simply $f(x) = x^3$, with noise added.
Our sampled function is then $f(x) = x^3 + \epsilon$, $\epsilon \sim \mathcal{N}(0, 9)$, where the data set consists of 50 points sampled from $[-4, -2]$, and 50 points sampled from $[2,4]$.
This data set is well suited for this task because the ground truth is challenging, but not too challenging, and the gap in the data is large enough to show if models have unbounded uncertainty, while being small enough that models with poor uncertainty qualities fail to increase uncertainty in the gap.
\begin{figure}[H]
    \centering
    \includegraphics[width=\linewidth]{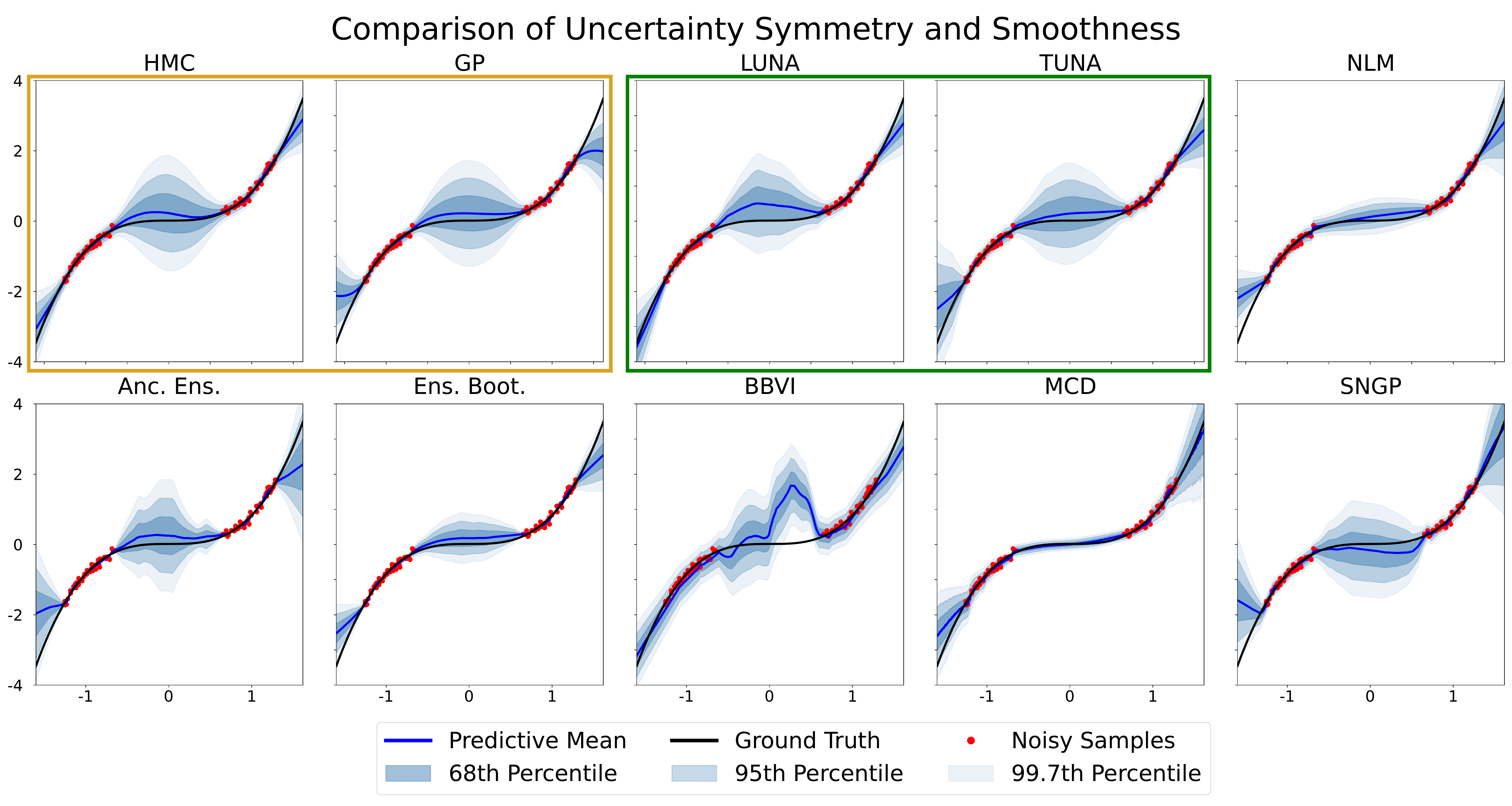}
    \caption{In a comparison of models on the normalized cubic gap data set, we see that the UNA models, LUNA and TUNA, have uncertainty very comparable to the gold-standard methods.}
    \label{fig:cubic_toy}
\end{figure}

Performance on a task such as this can serve to show whether or not a model can produce smoothly increasing uncertainty.
However, this limited view of uncertainty does not show the whole picture.
For instance, in all cases, the model's uncertainty is relatively constant where there is data.
This behavior is desirable in here because the data regions were uniformly sampled.
In cases where the data density is not uniform, we may want uncertainty to increase where there is less data, and increase yet faster where there is no data.
This work is focused primarily on data sets with uniform noise across all samples, and therefore the uncertainty we see in the gold-standard methods of figure \ref{fig:cubic_toy} are ideal.

\subsection{Transfer Learning}
\label{sec:transfer_learning}
We construct a synthetic 1-D dataset where train $x$ is uniformly sampled from the range $[-4,-2]\cup[2,4]$
and $y =x^3 + 20\exp(-x^2) \cdot \sin(10x) + \epsilon, \epsilon\sim\mathcal{N}(0, 3^2)$.
As seen in Figures \ref{fig:squiggle} and \ref{fig:low_good_high}, this function is like the Cubic Gap Example, but with unexpected variations in the gap.
For the generalization experiment, test $x$ is sampled from the same range. 
For the transfer learning experiment, test $x$ is sampled from the range $[-2,2]$, inside the gap.
\begin{figure}[H]
    \centering
    \includegraphics[width=0.5\linewidth]{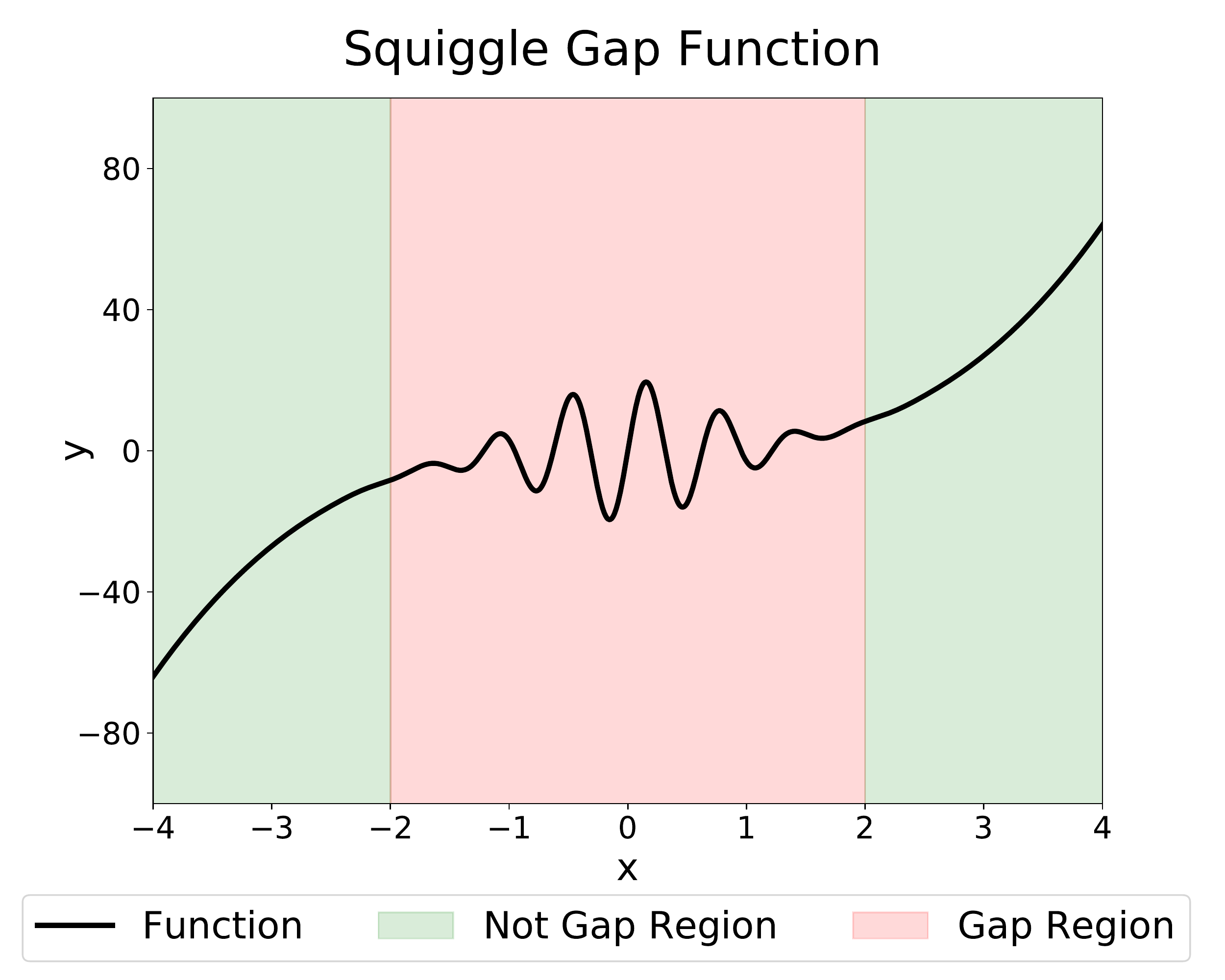}
    \caption{The squiggle gap function $y =x^3 + 20\exp(-x^2) \cdot \sin(10x)$.}
    \label{fig:squiggle}
\end{figure}
The experiment is conducted by training on data sampled from the not gap region, and relearning the posterior for predicting the gap region.
Data in this experiment was noramlized.
Effectively, we are testing how well features learned from data outside the gap perform at predicting data inside the gap.
As we would expect, methods with diversity in the gap region would perform best.
The average log-likelihood of LUNA increases the fastest as the number of features increases, beating both NLM and SNGP.
\begin{figure}[H]
    \centering
    \includegraphics[width=0.6\linewidth]{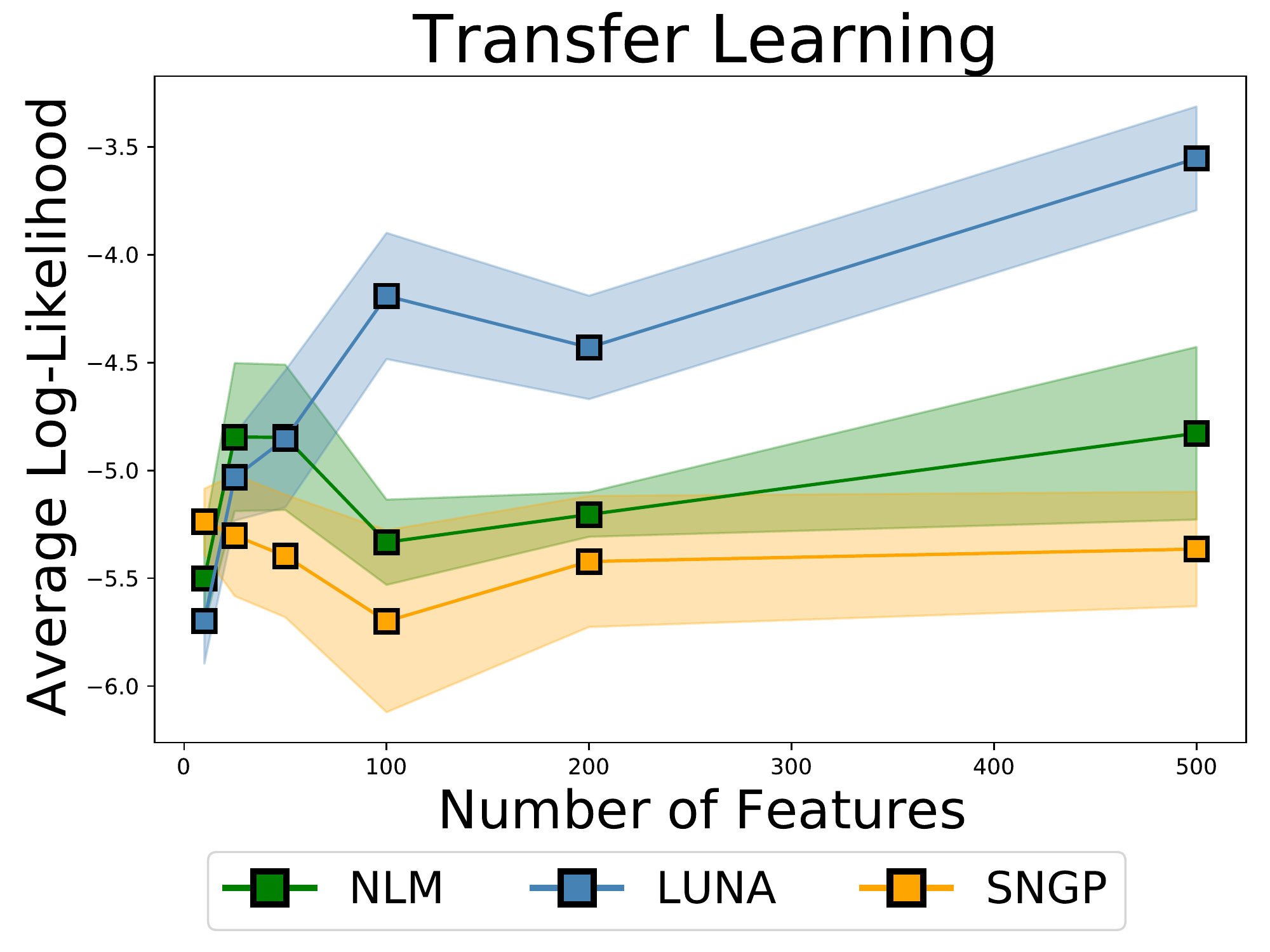}
    \caption{Log-likelihood of LUNA, SNGP, NLM on ``Squiggle Gap'' (Appendix \ref{sec:synthetic-data}) given number of features. With any number of features, LUNA outperform NLM and SNGP when the features are transferred to new data.}
    \label{fig:ll_vs_features_tf_relearn}
\end{figure}

\subsection{Encoding Domain Knowledge for TUNA}
\label{sec:encoding_uncertainty}
TUNA's reliance on reference functions for training allows us to incorporate domain knowledge into both predictions and predictive uncertainty.
For example, if we know the ground truth is sinusoidal, we can use noisy sine curves as the reference functions.
This domain knowledge allows TUNA's predictions to extrapolate much further away than if we were to use poorly chosen reference functions, such as samples from an RBF kernel, as seen in Figure \ref{fig:tuna_sine_extrapolate}.
\begin{figure}
    \centering
    \includegraphics[width=\linewidth]{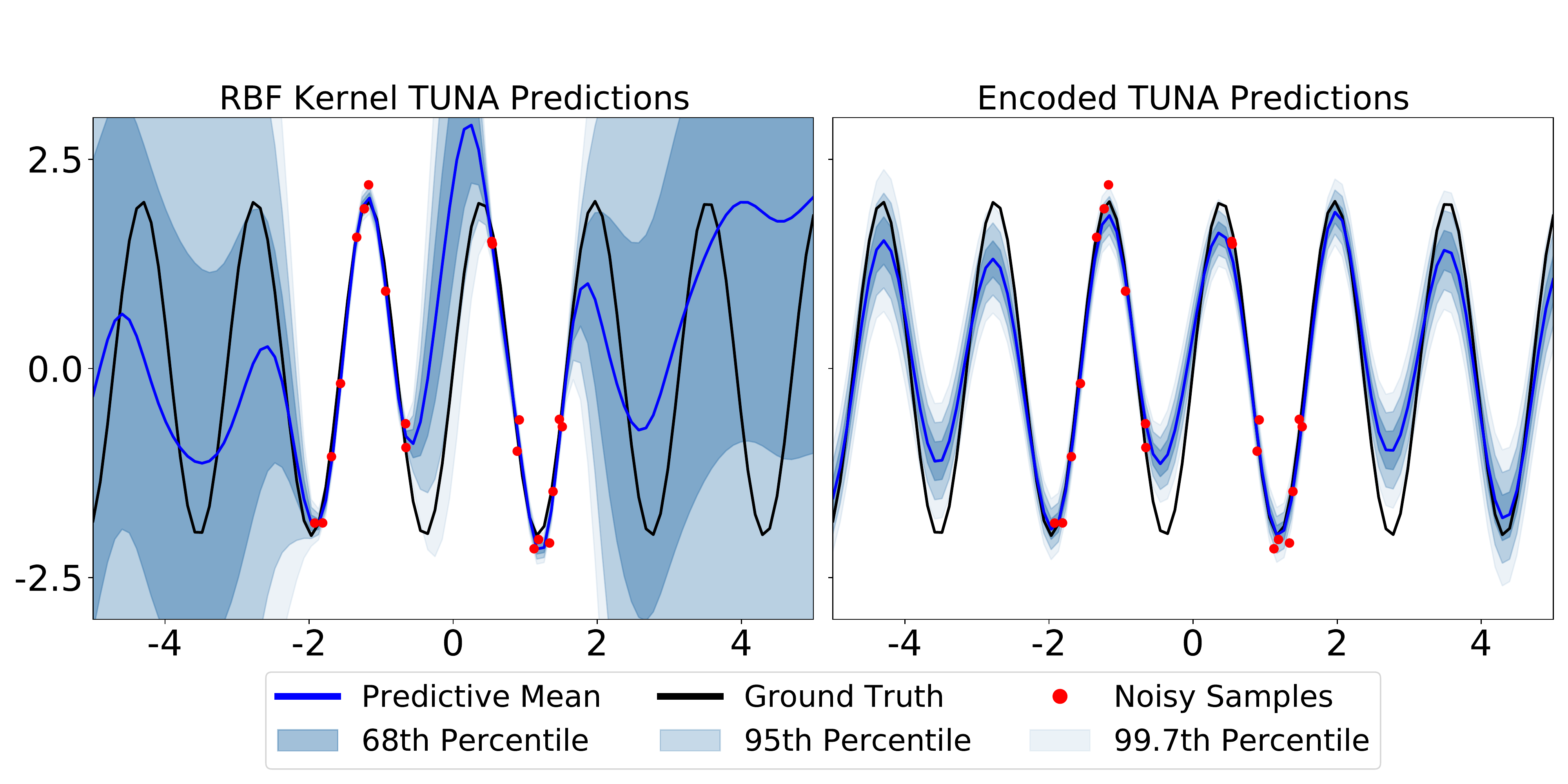}
    \caption{TUNA is able to extrapolate domain knowledge outside of the data region. This is done by choosing reference functions that capture the ground truth function.}
    \label{fig:tuna_sine_extrapolate}
\end{figure}

Domain knowledge can also be encoded into the predictive uncertainty of TUNA.
In this example, the ground truth function is $f(\mathbf{x}) = x_1^3$.
By using a kernel that does not incorporate this knowledge, we see that uncertainty is high where there is no data, and low where there is data in Figure \ref{fig:gp_tuna_circle}.
By choosing a kernel that incorporates the knowledge that $f(\mathbf{x})$ is independent of $x_2$, we are able to get uncertainties that are also constant in the direction of $x_2$, seen in Figure \ref{fig:gp_tuna_strip}.

\subsection{Controlling Uncertainty}

Through the use of pseudo data, we are able to control predictive uncertainty directly, tuning it to belarger or smaller.
First we see by using two auxiliary regressors and psuedo data in the middle, we can control the predictive uncertainty to be smaller than it otherwise would be in Figure \ref{fig:tuna_pseudo_low}.
\begin{figure}[H]
    \centering
    \includegraphics[width=0.8\linewidth]{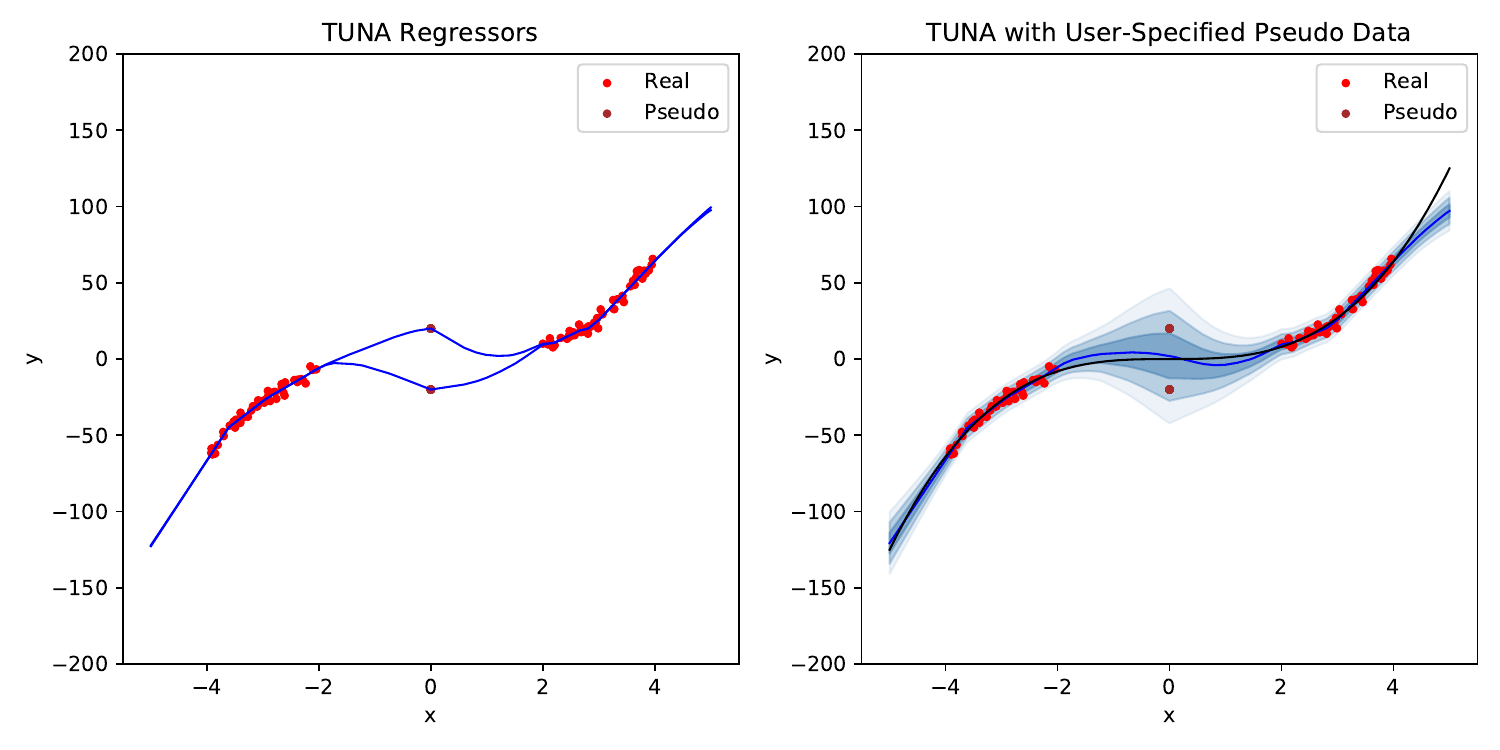}
    \caption{Pseudo data reduces predictive uncertainty in the middle of the gap.}
    \label{fig:tuna_pseudo_low}
\end{figure}
We can also use pseudo data to increase the predictive uncertainty by spreading the points out, seen in Figure \ref{fig:tuna_pseudo_hi}.
\begin{figure}[H]
    \centering
    \includegraphics[width=0.8\linewidth]{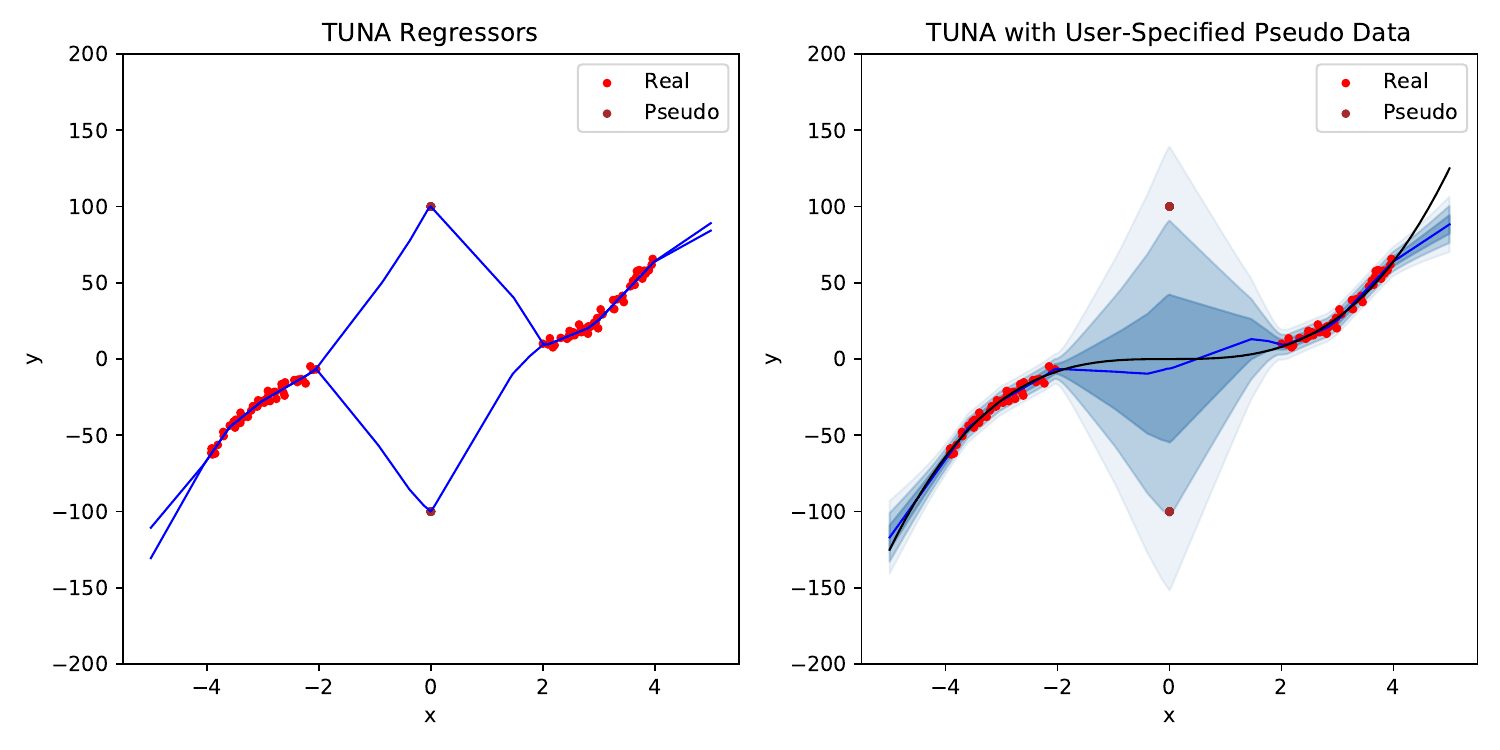}
    \caption{Pseudo data increases predictive uncertainty in the middle of the gap.}
    \label{fig:my_label}
\end{figure}
\begin{figure}
    \centering
    \begin{subfigure}[b]{0.45\textwidth}
    \centering
    \includegraphics[width=1.1\textwidth]{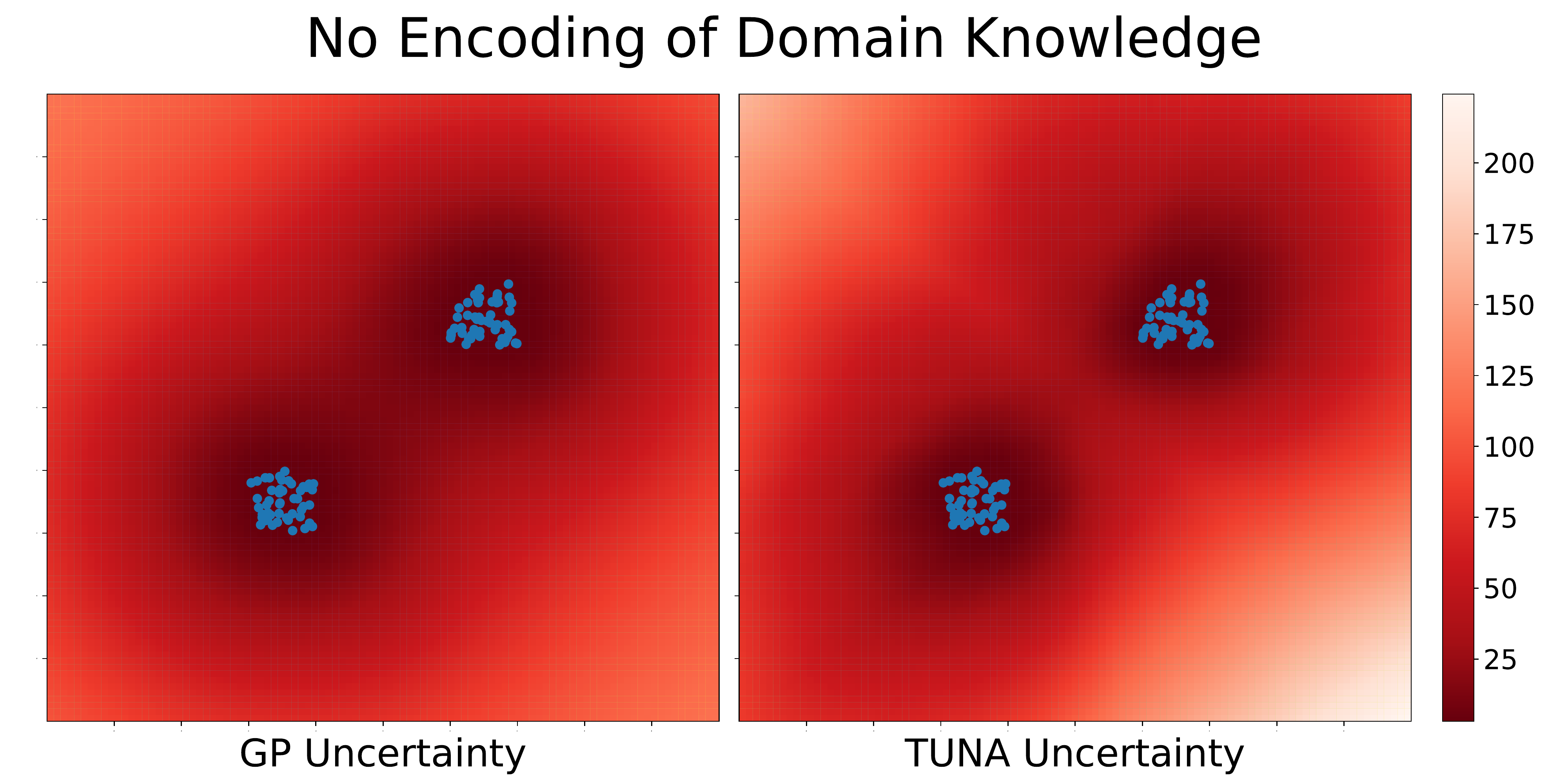}
    \caption{Without domain knowledge, GP and TUNA uncertainty is low where there is data, and high where there is no data.}
    \label{fig:gp_tuna_circle}
    \end{subfigure}
    \hfill
    \begin{subfigure}[b]{0.45\textwidth}
    \centering
    \includegraphics[width=1.1\textwidth]{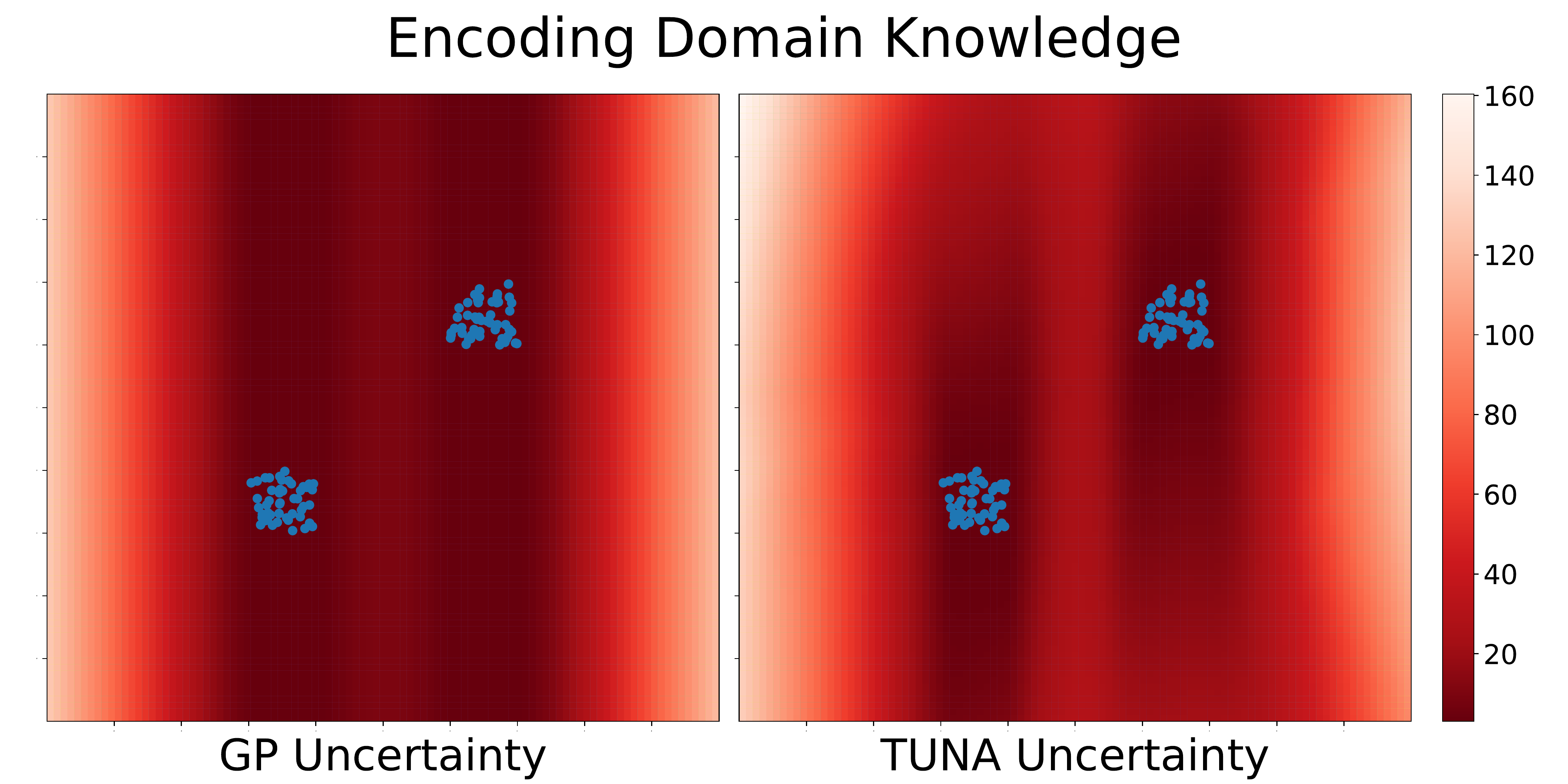}
    \caption{By selecting the GP kernel based on knowledge that $x_2$ is constant, uncertainty is also constant in that direction.}
    \label{fig:gp_tuna_strip}
    \end{subfigure}
\end{figure}

\section{UCI Data}
\label{sec:uci_experiments}
University of California at Irvine provides many real-world data sets that have become a standard benchmark for machine learning models (\cite{uci_data}).
Data sets cover a wide array settings for both regression and classification tasks.

\subsection{Standard UCI Regression}
\label{sec:standard_uci}
The standard regression tasks involve using an 80-10-10 train, test, validation split.
Hyperparameter selection was done on a grid using the validation set.
The grid search specifications are given in section \ref{sec:uci-gap-examples}.
Error bars are obtained by averaging across different data splits.
Presented in Tables \ref{tab:uci_regression_rmse} and \ref{tab:uci_regression_ll} are the RMSE and log-likelihood results of our benchmark models, respectively.
We see LUNA performs comparably to or better than all of our benchmarks across all of the data sets for both RMSE and average log-likelihood.

\begin{table*}
    \centering
    \text{Root Mean Square Error}
    \resizebox{\textwidth}{!}{
    \begin{tabular}{ccccccc}
        \toprule
        & Boston & Concrete & Yacht & Kin8nm & Energy & Wine \\ 
		\cmidrule(lr){2-2} \cmidrule(lr){3-3} \cmidrule(lr){4-4} \cmidrule(lr){5-5} \cmidrule(lr){6-6} \cmidrule(lr){7-7}
		Ens. Boot. & 2.86 $\pm$ 0.97 & 4.68 $\pm$ 0.50 & 0.81 $\pm$ 0.38 & 0.08 $\pm$ 0.00 & 0.50 $\pm$ 0.07 & 0.57 $\pm$ 0.06 \\
		Ensemble & 2.78 $\pm$ 0.91 & 4.48 $\pm$ 0.57 & 0.53 $\pm$ 0.24 & 0.08 $\pm$ 0.00 & 0.42 $\pm$ 0.06 & 0.58 $\pm$ 0.07 \\
		Anc. Ens. & 2.80 $\pm$ 0.80 & 4.61 $\pm$ 0.52 & 0.67 $\pm$ 0.23 & 0.08 $\pm$ 0.00 & 0.53 $\pm$ 0.06 & 0.58 $\pm$ 0.06 \\
		\cmidrule{1-7}
		NLM & 3.11 $\pm$ 0.93 & 4.68 $\pm$ 0.65 & 0.55 $\pm$ 0.30 & 0.08 $\pm$ 0.00 & 0.37 $\pm$ 0.06 & 0.59 $\pm$ 0.04 \\
		GP & 4.21 $\pm$ 1.22 & 11.94 $\pm$ 0.40 & 3.45 $\pm$ 0.90 & 0.09 $\pm$ 0.01 & 2.58 $\pm$ 0.23 & 0.60 $\pm$ 0.05 \\
		MCD & 2.94 $\pm$ 0.86 & 4.36 $\pm$ 0.68 & 0.58 $\pm$ 0.22 & 0.09 $\pm$ 0.01 & 0.40 $\pm$ 0.08 & 0.58 $\pm$ 0.05 \\
		SNGP & 3.06 $\pm$ 0.93 & 5.00 $\pm$ 0.50 & 1.10 $\pm$ 0.43  & 0.10 $\pm$ 0.01 & 0.87 $\pm$ 0.50 & 0.62 $\pm$ 0.07 \\
		BBVI & 5.31 $\pm$ 1.29 & 16.45 $\pm$ 0.62 & 1.91 $\pm$ 0.51 & 0.09 $\pm$ 0.00 & 2.39 $\pm$ 0.16 & 0.64 $\pm$ 0.05 \\
		\cmidrule{1-7}
		LUNA & 3.18 $\pm$ 1.00 & 4.70 $\pm$ 0.56 & 0.51 $\pm$ 0.20 & 0.08 $\pm$ 0.00 & 0.40 $\pm$ 0.06 & 0.62 $\pm$ 0.08 \\
	\end{tabular}}
	\caption{RMSE for our benchmark models on different UCI regression data sets.}
	\label{tab:uci_regression_rmse}
\end{table*}

\begin{table*}
    \centering
    \text{Avg. Log-Likelihood}
    \resizebox{\textwidth}{!}{
    \begin{tabular}{ccccccc}
        \toprule
        & Boston & Concrete & Yacht & Kin8nm & Energy & Wine \\ 
		\cmidrule(lr){2-2} \cmidrule(lr){3-3} \cmidrule(lr){4-4} \cmidrule(lr){5-5} \cmidrule(lr){6-6} \cmidrule(lr){7-7}
		Ens. Boot. & N/A & N/A & N/A & N/A & N/A & N/A \\
		Ensemble & N/A & N/A & N/A & N/A & N/A & N/A \\
		Anc. Ens. & N/A & N/A & N/A & N/A & N/A & N/A \\
		\cmidrule{1-7}
		NLM & -3.67 $\pm$ 0.01 & -5.33 $\pm$ 0.00 & -2.32 $\pm$ 0.01 & 1.03 $\pm$ 0.03 & -2.78 $\pm$ 0.00 & -1.02 $\pm$ 0.03 \\
		GP & -3.69 $\pm$ 0.02 & -5.34 $\pm$ 0.00 & -2.69 $\pm$ 0.13 & 0.91 $\pm$ 0.03 & -2.86 $\pm$ 0.01 & -1.03 $\pm$ 0.04 \\
		MCD & -3.67 $\pm$ 0.01 & -5.32 $\pm$ 0.00 & -2.31 $\pm$ 0.01 & 0.93 $\pm$ 0.04 & -2.77 $\pm$ 0.00 & -1.02 $\pm$ 0.04 \\
		SNGP & -3.66 $\pm$ 0.01 & -5.33 $\pm$ 0.00 & -2.34 $\pm$ 0.03  & 0.86 $\pm$ 0.15 & -2.79 $\pm$ 0.02 & -1.05 $\pm$ 0.05 \\
		BBVI & -3.77 $\pm$ 0.02 & -5.37 $\pm$ 0.00 & -2.61 $\pm$ 0.07 & 0.97 $\pm$ 0.02 & -2.90 $\pm$ 0.01 & -1.08 $\pm$ 0.04 \\
		\cmidrule{1-7}
		LUNA & -3.67 $\pm$ 0.01 & -5.33 $\pm$ 0.00 & -2.31 $\pm$ 0.01 & 1.02 $\pm$ 0.03 & -2.79 $\pm$ 0.00 & -1.05 $\pm$ 0.06 \\
	\end{tabular}}
	\caption{Average Log-Likelihood for our benchmark models on different UCI regression data sets.}
	\label{tab:uci_regression_ll}
\end{table*}

\subsection{UCI Gap Data Sets}
\label{sec:uci_gap}

The UCI gap sets take the standard UCI sets and introduce a gap into one of the features (\cite{uci_gap}).
The procedure converts these UCI data sets into UCI gap data sets.
For a selected input dimension,
we (1) sort the data in increasing order in that dimension,
and (2) remove middle $1/3$ to create a gap.
Figure \ref{fig:gap_high_d} shows the gap clearly, in the feature in which it was introduced.
We took three standard UCI sets, and modify them according to the procedure above to create five gap sets.
The features that correlated most strongly with the output were chosen and are:
\begin{itemize}
    \item Boston Housing: ``Rooms per Dwelling'' (RM), ``Percentage Lower Status of the Population'' (LSTAT), and ``Parent-Teacher Ratio by Town'' (PTRATIO)
    \item Concrete Compressive Strength: ``Cement'' and ``Superplasticizer''
    \item Yacht Hydrodynamics: ``Froude Number''
\end{itemize}
\begin{table*}
    \centering
    \text{Root Mean Square Error}
    \resizebox{\textwidth}{!}{
    \begin{tabular}{ccccccccccccc}
        \toprule
        & \multicolumn{2}{c}{Yacht - FROUDE} & \multicolumn{2}{c}{Concrete - CEMENT} 
        & \multicolumn{2}{c}{Concrete - SUPER} & \multicolumn{2}{c}{Boston - RM}
        & \multicolumn{2}{c}{Boston - LSTAT} & \multicolumn{2}{c}{Boston - PTRATIO} \\
		\cmidrule(lr){2-3} \cmidrule(lr){4-5} \cmidrule(lr){6-7} \cmidrule(lr){8-9} \cmidrule(lr){10-11} \cmidrule(lr){12-13}
		& Not Gap & Gap & Not Gap & Gap & Not Gap & Gap & Not Gap & Gap & Not Gap & Gap & Not Gap & Gap \\
		\cmidrule{2-13}
		Ens. Boot. & 1.21 $\pm$ 0.47 & 0.51 $\pm$ 0.08 & 5.27 $\pm$ 0.97 & 6.04 $\pm$ 0.10 & 4.70 $\pm$ 0.93 & 7.55 $\pm$ 0.26 & 2.81 $\pm$ 0.86 & 3.07 $\pm$ 0.10 & 3.28 $\pm$ 1.04 & 3.47 $\pm$ 0.14 & 3.46 $\pm$ 0.87 & 3.20 $\pm$ 0.08 \\
		Ensemble & 0.84 $\pm$ 0.39 & 0.40 $\pm$ 0.04 & 5.06 $\pm$ 0.90 & 6.10 $\pm$ 0.18 & 4.44 $\pm$ 0.78 & 7.49 $\pm$ 0.18 & 2.78 $\pm$ 0.90 & 3.04 $\pm$ 0.08 & 3.12 $\pm$ 1.12 & 3.20 $\pm$ 0.13 & 3.41 $\pm$ 0.77 & 3.27 $\pm$ 0.08 \\
		Anc. Ens. & 0.91 $\pm$ 0.33 & 0.86 $\pm$ 0.07 & 5.54 $\pm$ 0.87 & 6.36 $\pm$ 0.22 & 5.33 $\pm$ 0.69 & 7.77 $\pm$ 0.46 & 2.90 $\pm$ 0.74 & 3.09 $\pm$ 0.12 & 3.30 $\pm$ 1.17 & 3.33 $\pm$ 0.15 & 3.40 $\pm$ 0.87 & 3.16 $\pm$ 0.07 \\
		\cmidrule{1-13}
		NLM & 0.65 $\pm$ 0.26 & 0.72 $\pm$ 0.14 & 5.31 $\pm$ 0.97 & 7.01 $\pm$ 0.49 & 4.60 $\pm$ 0.94 & 8.47 $\pm$ 0.36 & 3.02 $\pm$ 0.87 & 3.21 $\pm$ 0.11 & 3.69 $\pm$ 1.51 & 3.93 $\pm$ 0.35 & 3.70 $\pm$ 0.67 & 3.68 $\pm$ 0.12 \\
		GP & 1.89 $\pm$ 0.54 & 1.37 $\pm$ 0.21 & 6.01 $\pm$ 0.89 & 6.21 $\pm$ 0.07 & 5.91 $\pm$ 0.78 & 8.01 $\pm$ 0.16 & 3.10 $\pm$ 0.91 & 3.27 $\pm$ 0.19 & 3.52 $\pm$ 1.16 & 3.32 $\pm$ 0.15 & 3.45 $\pm$ 0.78 & 3.28 $\pm$ 0.04 \\
		MCD & 0.89 $\pm$ 0.31 & 6.78 $\pm$ 0.37 & 5.09 $\pm$ 1.07 & 7.27 $\pm$ 0.40 & 4.80 $\pm$ 0.91 & 7.93 $\pm$ 0.34 & 3.45 $\pm$ 1.15 & 3.17 $\pm$ 0.10 & 3.40 $\pm$ 1.09 & 4.08 $\pm$ 0.36 & 3.69 $\pm$ 1.02 & 3.27 $\pm$ 0.18 \\
		SNGP & 1.04 $\pm$ 0.68 & 1.31 $\pm$ 1.11 & 5.15 $\pm$ 0.74 & 5.93 $\pm$ 0.17 & 5.00 $\pm$ 0.69 & 7.33 $\pm$ 0.32 & 3.07 $\pm$ 0.56 & 3.41 $\pm$ 0.18 & 3.79 $\pm$ 1.02 & 4.18 $\pm$ 0.26 & 3.75 $\pm$ 1.00 & 3.77 $\pm$ 0.21 \\
		BBVI & 17.27 $\pm$ 5.87 & 30.05 $\pm$ 2.99 & 5.68 $\pm$ 0.80 & 6.36 $\pm$ 0.07 & 24.17 $\pm$ 7.56 & 54.56 $\pm$ 4.58 & 3.47 $\pm$ 0.87 & 3.53 $\pm$ 0.05 & 3.76 $\pm$ 1.22 & 3.82 $\pm$ 0.12 & 9.16 $\pm$ 3.29 & 30.52 $\pm$ 2.84 \\
		\cmidrule{1-13}
        LUNA & 1.16 $\pm$ 0.42 & 0.57 $\pm$ 0.10 & 5.50 $\pm$ 1.32 & 7.12 $\pm$ 0.36 & 4.92 $\pm$ 0.66 & 10.13 $\pm$ 1.05 & 3.34 $\pm$ 1.09 & 3.17 $\pm$ 0.38 & 3.57 $\pm$ 1.44 & 3.92 $\pm$ 0.26 & 3.58 $\pm$ 1.09 & 3.34 $\pm$ 0.14 \\
	\end{tabular}}
	\caption{RMSE for our benchmark models on different UCI gap data sets.}
	\label{tab:uci_gap_rmse}
\end{table*}

\begin{table*}
    \centering
    \text{Avg. Log-Likelihood}
    \resizebox{\textwidth}{!}{
    \begin{tabular}{ccccccccccccc}
        \toprule
        & \multicolumn{2}{c}{Yacht - FROUDE} & \multicolumn{2}{c}{Concrete - CEMENT} 
        & \multicolumn{2}{c}{Concrete - SUPER} & \multicolumn{2}{c}{Boston - RM}
        & \multicolumn{2}{c}{Boston - LSTAT} & \multicolumn{2}{c}{Boston - PTRATIO} \\
		\cmidrule(lr){2-3} \cmidrule(lr){4-5} \cmidrule(lr){6-7} \cmidrule(lr){8-9} \cmidrule(lr){10-11} \cmidrule(lr){12-13}
		& Not Gap & Gap & Not Gap & Gap & Not Gap & Gap & Not Gap & Gap & Not Gap & Gap & Not Gap & Gap \\
		\cmidrule{2-13}
		Ens. Boot. & N/A & N/A & N/A & N/A & N/A & N/A & N/A & N/A & N/A & N/A & N/A & N/A \\
		Ensemble & N/A & N/A & N/A & N/A & N/A & N/A & N/A & N/A & N/A & N/A & N/A & N/A \\
		Anc. Ens. & N/A & N/A & N/A & N/A & N/A & N/A & N/A & N/A & N/A & N/A & N/A & N/A \\
		\cmidrule{1-13}
		NLM & -1.29 $\pm$ 0.90 & -1.45 $\pm$ 0.52 & -3.15 $\pm$ 0.26 & -3.71 $\pm$ 0.19 & -2.97 $\pm$ 0.25 & -4.29 $\pm$ 0.16 & -2.56 $\pm$ 0.30 & -2.58 $\pm$ 0.03 & -2.82 $\pm$ 0.70 & -2.79 $\pm$ 0.12 & -2.73 $\pm$ 0.25 & -2.72 $\pm$ 0.04 \\
		GP & -1.76 $\pm$ 0.30 & -1.56 $\pm$ 0.04 & -3.18 $\pm$ 0.14 & -3.17 $\pm$ 0.01 & -3.19 $\pm$ 0.19 & -3.40 $\pm$ 0.01 & -2.53 $\pm$ 0.12 & -2.60 $\pm$ 0.02 & -2.62 $\pm$ 0.18 & -2.61 $\pm$ 0.02 & -2.63 $\pm$ 0.19 & -2.63 $\pm$ 0.01 \\
		MCD & -1.13 $\pm$ 0.27 & -32.49 $\pm$ 12.07 & -2.94 $\pm$ 0.18 & -3.54 $\pm$ 0.06 & -2.96 $\pm$ 0.18 & -3.71 $\pm$ 0.06 & -2.59 $\pm$ 0.22 & -2.52 $\pm$ 0.02 & -2.59 $\pm$ 0.20 & -2.69 $\pm$ 0.11 & -2.59 $\pm$ 0.16 & -2.58 $\pm$ 0.05 \\
		SNGP & -4.35 $\pm$ 5.66 & -8.14 $\pm$ 13.55 & -3.12 $\pm$ 0.21 & -3.35 $\pm$ 0.05 & -3.07 $\pm$ 0.20 & -3.87 $\pm$ 0.14 & -2.56 $\pm$ 0.17 & -2.65 $\pm$ 0.06 & -2.81 $\pm$ 0.39 & -2.90 $\pm$ 0.10 & -2.81 $\pm$ 0.39 & -2.77 $\pm$ 0.07 \\
		BBVI & -68.40 $\pm$ 31.42 & -207.53 $\pm$ 34.99 & -3.15 $\pm$ 0.17 & -3.30 $\pm$ 0.02 & -6.74 $\pm$ 2.24 & -27.14 $\pm$ 3.45 & -2.63 $\pm$ 0.09 & -2.64 $\pm$ 0.01 & -2.70 $\pm$ 0.21 & -2.63 $\pm$ 0.02 & -3.49 $\pm$ 0.35 & -8.14 $\pm$ 1.22 \\
		\cmidrule{1-13}
        LUNA & -2.82 $\pm$ 2.09 & -0.96 $\pm$ 0.16 & -3.14 $\pm$ 0.30 & -3.55 $\pm$ 0.13 & -3.00 $\pm$ 0.13 & -4.28 $\pm$ 0.36 & -2.56 $\pm$ 0.15 & -2.54 $\pm$ 0.03 & -2.72 $\pm$ 0.44 & -2.75 $\pm$ 0.07 & -2.69 $\pm$ 0.31 & -2.65 $\pm$ 0.06 \\
	\end{tabular}}
	\caption{Average Log-Likelihood for our benchmark models on different UCI gap data sets.}
	\label{tab:uci_gap_ll}
\end{table*}

\begin{table*}
    \centering
    \text{Avg. Epistemic Uncertainty Gap-Not Gap Ratio}
    \resizebox{\textwidth}{!}{
    \begin{tabular}{ccccccc}
        \toprule
        & Yacht - FROUDE & Concrete - CEMENT
        & Concrete - SUPER & Boston - RM
        & Boston - LSTAT  & Boston - PTRATIO\\
		\cmidrule(lr){2-2} \cmidrule(lr){3-3} \cmidrule(lr){4-4} \cmidrule(lr){5-5} \cmidrule(lr){6-6} \cmidrule(lr){7-7}
        Ens. Boot. & 0.00\% $\pm$ 0.00\% & \textbf{40.85\% $\pm$ 7.60\%} & \textbf{87.64\% $\pm$ 15.79\%} & -9.22\% $\pm$ 9.57\% & 6.40\% $\pm$ 10.97\% & -1.33\% $\pm$ 8.13\% \\
        Ensemble & \textbf{26.68\% $\pm$ 11.75\%} & \textbf{120.70\% $\pm$ 26.68\%} & \textbf{163.28\% $\pm$ 26.81\%} & -5.61\% $\pm$ 11.41\% & \textbf{18.84\% $\pm$ 11.47\%} & \textbf{39.83\% $\pm$ 16.37\%} \\
        Anc. Ens. & 8.66\% $\pm$ 23.98\% & \textbf{93.03\% $\pm$ 25.66\%} & \textbf{149.31\% $\pm$ 31.30\%} & -8.94\% $\pm$ 11.15\% & 10.03\% $\pm$ 10.72\% & \textbf{31.12\% $\pm$ 22.25\%} \\
        \cmidrule{1-7}
        NLM & \textbf{24.08\% $\pm$ 17.23\%} & -5.92\% $\pm$ 3.63\% & 4.50\% $\pm$ 5.89\% & -7.47\% $\pm$ 3.57\% & -5.79\% $\pm$ 7.53\% & -3.81\% $\pm$ 6.32\% \\
        GP & \textbf{71.73\% $\pm$ 19.75\%} & \textbf{75.59\% $\pm$ 15.12\%} & \textbf{119.48\% $\pm$ 19.55\%} & -16.34\% $\pm$ 6.76\% & -2.40\% $\pm$ 9.20\% & \textbf{17.74\% $\pm$ 13.17\%} \\
        MCD & -52.43\% $\pm$ 10.17\% & \textbf{6.86\% $\pm$ 5.92\%} & \textbf{11.95\% $\pm$ 7.30\%} & -16.09\% $\pm$ 9.55\% & -11.76\% $\pm$ 9.48\% & -5.22\% $\pm$ 6.88\% \\
        SNGP & 38.33\% $\pm$ 63.97\% & -6.05\% $\pm$ 4.39\% & -2.93\% $\pm$ 4.23\% & -10.61\% $\pm$ 3.88\% & -6.43\% $\pm$ 5.46\% & -6.92\% $\pm$ 5.94\% \\
        BBVI & -10.08\% $\pm$ 6.15\% & -18.15\% $\pm$ 5.58\% & 10.85\% $\pm$ 12.68\% & -13.74\% $\pm$ 5.15\% & -31.91\% $\pm$ 3.27\% & \textbf{20.30\% $\pm$ 8.54\%} \\
        \cmidrule{1-7}
        LUNA & \textbf{59.17\% $\pm$ 36.87\%} & \textbf{55.93\% $\pm$ 28.83\%} & \textbf{416.02\% $\pm$ 197.69\%} & -11.09\% $\pm$ 5.20\% & \textbf{29.09\% $\pm$ 18.15\%} & \textbf{60.60\% $\pm$ 45.32\%} \\
	\end{tabular}}
	\caption{Ratio of predictive epistemic uncertainty for data in the gap to data out of the gap for our benchmark models on different UCI gap data sets.}
	\label{tab:uci_gap_eps_ratio}
\end{table*}

It is important to note we are primarily concerned with each model's performance on the not gap region.
This data comes from the same region of data as the train data, and is representative of each model's performance on standard regression tasks.
We see that LUNA is within error bars of NLM for both RMSE and average log-likelihood for all data sets except Yahct - Froude in tables \ref{tab:uci_gap_rmse} and  \ref{tab:uci_gap_ll}, respectively.
LUNA's RMSE is higher and average log-likelihood lower than NLM, with much larger error bars on the Yacht data set.
We see LUNA performs comparably to our other benchmark models, with Ensemble tending to have the best RMSE and GP and MCD tending to have the best average log-likelihood.
LUNA is within error bars of these models' performance.

The goal of this experiment for each model is to detect the gap region.
We have seen in section \ref{sec:nlm_pathologies} that average log-likelihood cannot reliably determine this, so we must look at the predictive uncertainties.
Epistemic uncertainty is calculated by taking predictive uncertainty and subtracting the aleatoric uncertainty, or data noise.
This data noise is set for each Bayesian model and was hand tuned according to the procedure in Appendix \ref{sec:uci-gap-examples}.
Because epistemic uncertainty lacks a solid interpretation in this context, we look at the ratio of epistemic uncertainty in the gap to the epistemic uncertainty not in the gap, seen in Table \ref{tab:uci_gap_eps_ratio}.
Raw values are provided in Appendix \ref{sec:appendix_uci_gap}.
A ratio that is at least one standard deviation above 0\% is considered detection.
We see, first that none of the models are able to detect this region for the Boston - RM data set.
We offer an explanation for this in Appendix \ref{sec:appendix_uci_gap}.
Only LUNA and Ensemble are able to detect this gap region for the other data sets, with LUNA being the only Bayesian model.
This offers substantial improvement over the traditional NLM, which only detects the gap for the Yacht - Froude data sets.

\section{Bayesian Optimization}
\label{sec:bayesopt_experiments}
We used common BayesOpt benchmarks to evaluate the usefulness of our uncertainties.
These benchmarks were adapted from HPOLib 1.5 \citep{eggensperger_toward} and represent a variety of tasks that are difficult or impossible for traditional optimization techniques.
BayesOpt works by fitting a surrogate function to our samples, then maximizing an acquisition function to select our next sample.
The acquisition function aims to balance exploitation and exploration of our function space.
In this case, we used the expected improvement acquisition function \cite{Mockus10.1007/3-540-07165-2_55}.
This acquisition function uses the predictive mean and uncertainty of our surrogate function to select the next point at which to evaluate our target function.
Good uncertainty here is critical for the acquisition function to take good samples.
Our surrogate function can be thought of as defining a distribution over diferent functions that our acquitision function then chooses from.
The distribution of functions needs to be reasonable in order to select good samples.
This can be seen in Figure \ref{fig:bayesopt_example}.
A more detailed exposition of Bayesopt is given in \cite{Lizotte10.5555/1626686}.

First, we used the Branin function, a 2-dimensional benchmark with multiple global minima and shallow valleys between the minima, features that traditional optimization techniques struggle with.
The function is defined as:
\begin{equation*}
    f(\textbf{x}) = \left(x_2 - \frac{5.1}{4\pi^2}x_1^2 + \frac{5}{\pi}x_1 - 6\right)^2 + 10\left(1-\frac{1}{8\pi}\right)\cos \left(x_1\right) + 10
\end{equation*}
The input domain used is the square $x_1 \in [-5, 10]$ and $x_2 \in [0, 15]$.
In this domain, the global minima occur at $\textbf{x}^* = (-\pi, 12.275)$, $(\pi, 2.275)$, and $(9.42478, 2,475)$, with minimum $f(\textbf{x}^*) = 0.397887$.

The Hartmann6 function was also used and is a higher dimensional function on a small domain.
It is defined as:
\begin{equation*}
    f(\textbf{x}) = -\sum_{i=1}^{4}\alpha_i \exp\left( -\sum_{j=1}^{6} A_{ij}\left(x_j - P_{ij} \right)^2 \right)
\end{equation*}
where $\alpha = [1.0, 1.2, 3.0, 3.2]^T$,
    
$$\mathbf{A} = \begin{bmatrix}
10 & 3 & 17 & 3.5 & 1.7 & 8 \\
0.05 & 10 & 17 & 0.1 & 8 & 14 \\
3 & 3.5 & 1.7 & 10 & 17 & 8 \\
17 & 8 & 0.05 & 10 & 0.1 & 14
\end{bmatrix}$$

$$\mathbf{P} = 
10^-4 \begin{bmatrix}
1312 & 1696 & 5569 & 124 & 8283 & 5886 \\
2329 & 4135 & 8307 & 3736 & 1004 & 9991 \\
2348 & 1451 & 3522 & 2883 & 3047 & 6650 \\
4047 & 8828 & 8732 & 5743 & 1091 & 381 
\end{bmatrix}
$$
The input domain used is the hypercube $x_i \in (0, 1)$ for all $x_i$, with global minimum 
$f(\textbf{x}^*) = -3.32237$ for
\begin{gather*}
\textbf{x}^* = (0.20169, 0.150011, 0.476874, 0.275332, 0.311652, 0.6573)
\end{gather*}
The Hartmann6 function has many local optima where traditional methods of optimization would get stuck.

\begin{figure}
    \centering
    \includegraphics[width=1.0\linewidth]{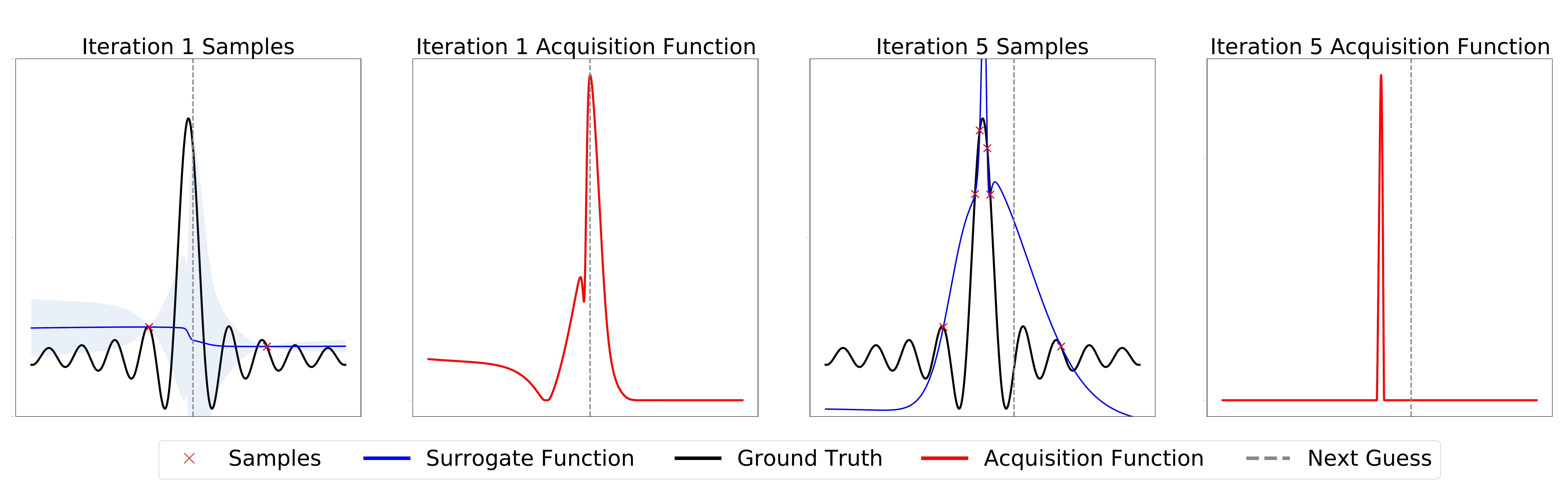}
    \caption{On this simple example, Bayesian Optimization using NLM is able to quickly reach the optimal value.}
    \label{fig:bayesopt_example}
\end{figure}

A popular application for BayesOpt is hyperparameter tuning.
In most cases it is impossible or intractable to backpropagate gradients through model training and prediction.
As such, gradient-based optimization methods cannot be used in this context.
We used each model to optimize classification models for MNIST data.
The SVM model was optimized over the regularization parameter and kernel coefficient, where the domain for each parameter is $[-10, 10]$ on the log-scale.
Due to computational complexity, the MNIST data was subsampled to contain equal parts 1's, 5's, 8's, and 9's, with 4,000 points in the train set, 1,000 points in the validation set, and 1,000 points in the test set.
Similarly, the logistic benchmark is a logistic regression classifier where the learning rate, $l_2$-regularization, batch size, and dropout ratio on inputs was tuned.
The learning rate domain is $[-6, 0]$ on the log-scale, the $l_2$-regularization domain is $[0,1]$, the batch size domain is $[20,2000]$, and the dropout ratio domain is $[0. 0.75]$.

The results are reported as error, that is $\left| f(\mathbf{x}) - f(\mathbf{x}^*) \right|$
where classification global minimum is at 0.
We used our feasible baseline models from before.

We see all of the models do well on the Branin, SVM , and logistic benchmarks, with GP performing slightly better than the others on logistic. The Hartmann6 benchmark shows a difference between models, where we see a GP performs the best, with LUNA having slightly higher error.

\begin{table*}
    \centering
    \text{Bayesian Optimization}
	\resizebox{\textwidth}{!}{
    \begin{tabular}{cc||c|cccccccc}
        \toprule
        Function & Steps & LUNA & GP & NLM & MCD & ANC. ENS. & ENS. BOOT. & Ensemble & SNGP & BBVI \\
        \midrule
		branin & 50 & 
		0.01 $\pm$ 0.00 & 	 
		0.00 $\pm$ 0.00 & 	 
		0.01 $\pm$ 0.01 & 	 
		0.01 $\pm$ 0.01 & 	 
		0.01 $\pm$ 0.01 & 	 
		0.06 $\pm$ 0.12 & 	 
		0.01 $\pm$ 0.02 & 	 
		0.00 $\pm$ 0.00 & 	 
		0.01 $\pm$ 0.00  	 
		 \\ 
		hartmann6 & 200 & 
		0.32 $\pm$ 0.02 & 	 
		0.01 $\pm$ 0.00 & 	 
		0.57 $\pm$ 0.44 & 	 
		0.76 $\pm$ 0.25 & 	 
		0.23 $\pm$ 0.21 & 	 
		0.65 $\pm$ 0.28 & 	 
		0.68 $\pm$ 0.28 & 	 
		0.22 $\pm$ 0.00 & 	 
		0.71 $\pm$ 0.23 	 
		\\ 
		svm & 30 & 
		1.19 $\pm$ 0.12 & 	 
		1.20 $\pm$ 0.00 & 	 
		1.20 $\pm$ 0.06 & 	 
		1.10 $\pm$ 0.00 & 	 
		1.18 $\pm$ 0.12 & 	 
		1.11 $\pm$ 0.18 & 	 
		1.13 $\pm$ 0.05 & 	 
		1.19 $\pm$ 0.14 & 	 
		1.30 $\pm$ 0.35  	 
		 \\ 
		logistic & 30 & 
		7.64 $\pm$ 0.06 & 	 
		7.40 $\pm$ 0.00 & 	 
		7.64 $\pm$ 0.10 & 	 
		7.91 $\pm$ 0.29 & 	 
		7.66 $\pm$ 0.09 & 	 
		7.64 $\pm$ 0.08 & 	 
		7.59 $\pm$ 0.07 & 	 
		7.64 $\pm$ 0.07 & 	 
		7.92 $\pm$ 0.33  	 
		 \\ 
    \end{tabular}
    }
    \caption{Results for several Bayesian optimization benchmarks.}
    \label{tab:bayesopt}
\end{table*}

\section{Detecting Sampling Bias in Data}
We study a case where the predictive uncertainties of LUNA models can be used to detect sampling bias in a dataset (full details in Appendix \ref{sec:training}). 
For this task, we use LUNA trained NLMs with a Resnet18 architecture to perform age regression on the Wikipedia faces dataset, containing 62,328 facial images of actors. To study LUNA’s performance on out-of-distribution data, we train on 26,375 faces of only male actors and test on 10,918 male (in-distribution) and 10918 female (out-of-distribution) faces. On the training data, we obtain an MAE (mean absolute error) of 9.52, while on in-distribution test data we obtain an MAE of 10.22 and on out-of-distribution test data an MAE of 11.78 (comparable with the performance of a vanilla Resnet18 trained for age regression). At the same time, the epistemic uncertainty is on average 14\% higher on the in-distribution test data than on training data, whereas the epistemic uncertainty is on average 168\% higher on the out-of-distribution test data than on training data. In a separate experiment, we train on 28271 faces of individuals who are younger than 30 or older than 40, then we test on 9424 in-distribution faces and 10376 faces of individuals between the ages of 30-40 (out-of-distribution). On this task, we again see higher average epistemic uncertainty on out-of-distribution test data (27\% increase) than on in-distribution test data (2.02\% increase).  

In these cases, we show that the predictive uncertainty provided by LUNA trained models can be used to identify test data from underrepresented sub-populations in the training data; predictions for such out-of-distribution test data can then be deferred to human experts. This task also shows that LUNA can leverage structured data more easily than GP models by using task-appropriate network architecture.

\chapter{Discussion}\label{ch:6}
The contributions of this work are multifaceted.
First, shortcomings of the NLM training for uncertainty quantification are explained.
Second, a novel training framework for NLM models is proposed.
Third, an uncertainty benchmark is proposed that allows easy comparison of multiple aspects of uncertainty across models.
Many experiments were run that demonstrate the efficacy of the UNA framework.

\section{Uncertainty}
\label{sec:uncertainty_discussion}
While prediction uncertainty is a fairly straightforward concept, what it means for that uncertainty to be 'good' is hard to quantify due to task and prior knowledge dependence.
A benchmark of uncertainty should, first, have easily definable ideal behavior.
Second, the benchmark should capture multiple aspects of 'good' uncertainty.
Third, the benchmark should scale to multiple dimensions in order to create more challenging tasks.
Therein lies the RUB.
The RUB has all of these characteristics, and shows LUNA is the only method explored here to be able to match GP uncertainty performance.

More work can be done developing additional benchmarks that target a specific aspect of a model's uncertainty.
The 

\section{UNA}
\label{sec:una_discussion}
The shortcomings of the NLM are seen in section \ref{sec:nlm}.
Largely, these shortcomings come in the form of uncertainty estimates.
Seen in Chapter \ref{ch:5}, the NLM tends to underestimate uncertainty when compared to the gold-standard methods.
Additionally, the NLM does not allow for encoding a priori knowledge.
Due to these shortcomings, a novel training framework is necessary for desireable uncertainty properties.
Both shortcomings are addressed by UNA.

The two instantiations of UNA (LUNA and TUNA) are designed to capture different GP-like model properties.
Training for both instantiations is tractable and scalable.

There are several lines of future work:
Firstly, both LUNA and TUNA require additional hyper-parameters - in addition to choosing a number of regressors $M$, LUNA requires a choice of diversity-penalty strength, and TUNA's target reference functions may require some hyper-parameters.
Although the number of additional hyperparameters introduced is modest, it would be worthwhile to be able to set these hyper-parameters based on properties of the data alone (i.e. without selection on a validation set).
Additionally, when TUNA's reference functions come from a GP prior, it would be ideal to tune the GP prior's hyperparameters during TUNA training. 

Empirically, we find that the capacity needed for TUNA training depends less on the number of reference functions (modest size neural networks can easily capture 100+ simple reference functions) and more on the complexity of the reference functions. Thus, to match the performance of LUNA on regression tasks, TUNA may need extra capacity (which it uses to encode specific functional knowledge). In the future, we want to quantify the trade-off between capacity and complexity of inductive bias between LUNA and TUNA.

The diversity penalty in LUNA training currently assumes a single-output model. In future work, we plan to extend LUNA to multi-output data in one of several ways: (1) measure diversity per output dimension, (2) penalize the Jacobian of the regressors or (3) use alternative penalties. 
Lastly, here we focus on commonly used ReLU activations; in practice, we find UNA works well with all activations.



\begin{singlespacing}
  \renewcommand{\bibname}{References}

  \bibliographystyle{ecca}
  \bibliography{main}
\end{singlespacing}

\begin{appendices}
\label{sec:appendix}

\chapter{Appendix to Chapter \ref{ch:4}}
\label{ch:una_appendix}
In proposing a novel framework, it is important to understand the effect of hyperparameters.
Here, the effect of various hyperparamters, and architectures are explored, as well as the effect of random restarts.
It is important to note this is not an exhaustive exploration.

\section{Neural Linear Model}
This experiment was run with a 2-layer ReLU network with 50 and 20 neurons in the first and second layers respectively (20 features). We used MAP training. With no regularization and very noisy priors, the NLM is able to model in-between uncertainty, albeit inconsistently, seen in Figure \ref{fig:rr_reg}. With regularization, we see the priors are not expressive enough and the NLM fails to ever capture in-between uncertainty.
\label{sec:appendix_nlm}
\begin{figure*}[p]
    \centering
    \includegraphics[width=\textwidth]{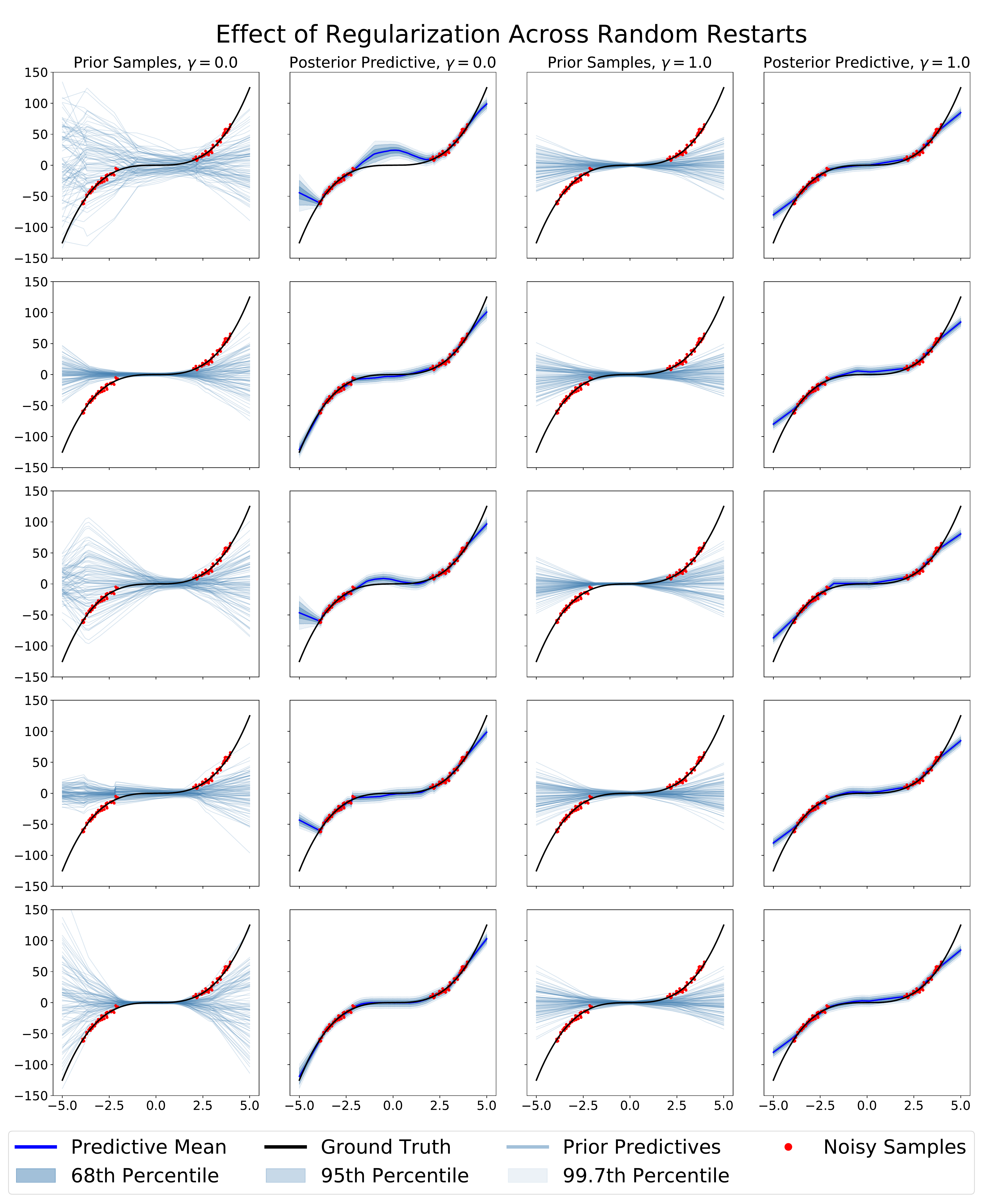}
    \caption{NLM is inconcistent in capturing the in-between uncertainty.}
    \label{fig:rr_reg}
\end{figure*}

This experiment was run with a 2-layer ReLU network with 50 neurons in the first layer and without regularization. We used MAP training. The number of neurons in the second layer correspond to the number of features. We see clearly as model capacity increases NLM better fits the data, seen in Figure \ref{fig:rr_features}. However, this increased capacity still fails to consistently model in-between uncertainty.
\begin{figure*}[p]
    \centering
    \includegraphics[width=\textwidth]{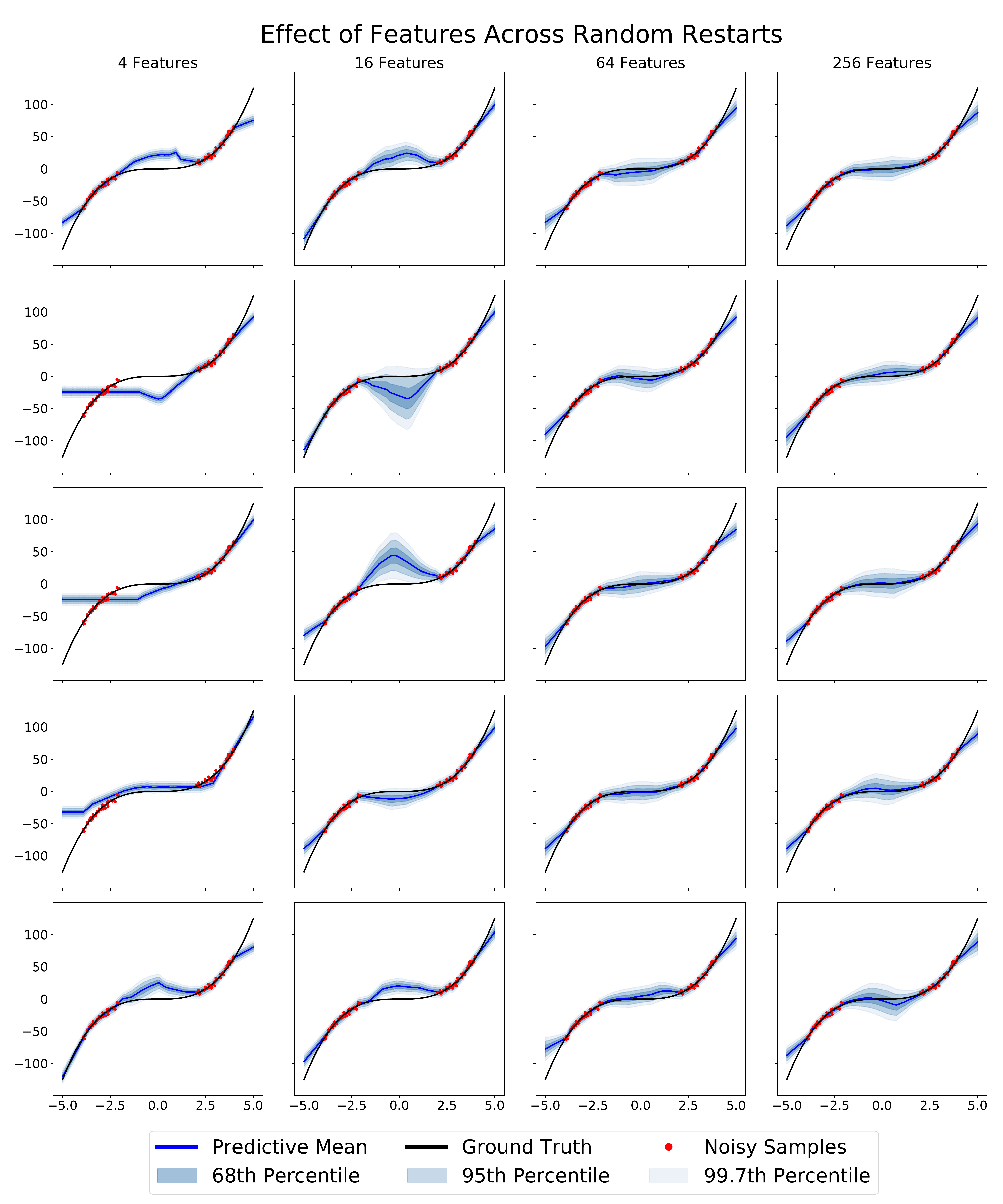}
    \caption{NLM is inconsistent in capturing hte in-between uncertainty even with increase capacity.}
    \label{fig:rr_features}
\end{figure*}

This experiment was run with a 2-layer ReLU network with 50 neurons in the first hidden layers and 20 neurons in the last hidden layer (20 features). We used MAP training. We see that NLM is able to capture more complex relationships as capacity increases, but this increased capacity does not lead to consistent in-between uncertainty in Figure \ref{fig:rr_depth}. Additionally, these added layers make NLM susceptible to overfitting.
\begin{figure*}[p]
    \centering
    \includegraphics[width=\textwidth]{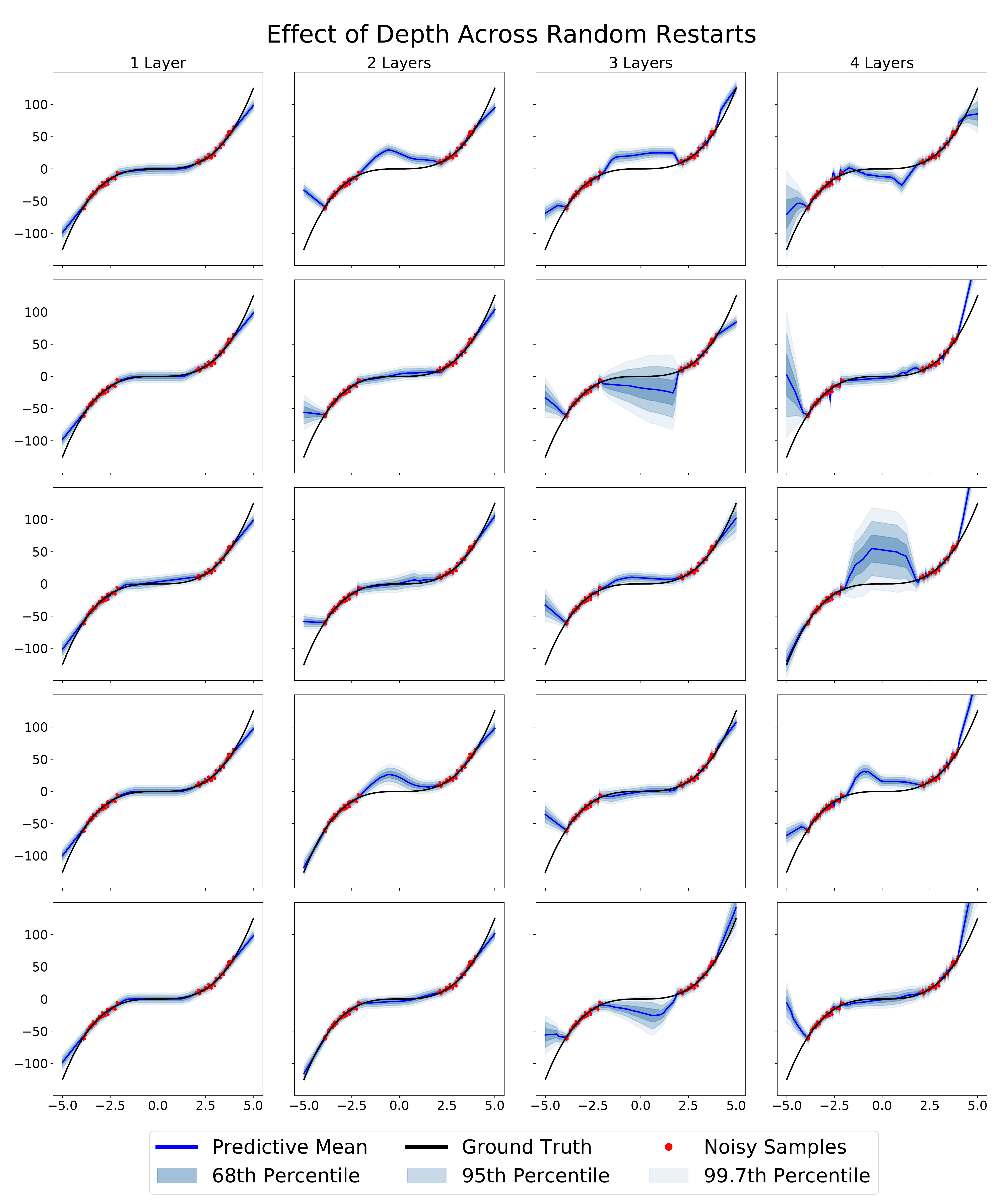}
    \caption{Deeper NLMs are still inconsistent in capturing in-between uncertainty.}
    \label{fig:rr_depth}
\end{figure*}

This experiment was run with a 2-layer ReLU network with 50 and 20 neurons in the first and second layers respectively (20 features). We used marginal likelihood training and a smaller $\gamma=0.1$. We see that this NLM is able to capture higher in-between uncertainty when $\alpha$ is high enough in Figure \ref{fig:marg_highgamma}, but is inconsistent in doing so.
\begin{figure*}[p]
    \centering
    \includegraphics[width=\textwidth]{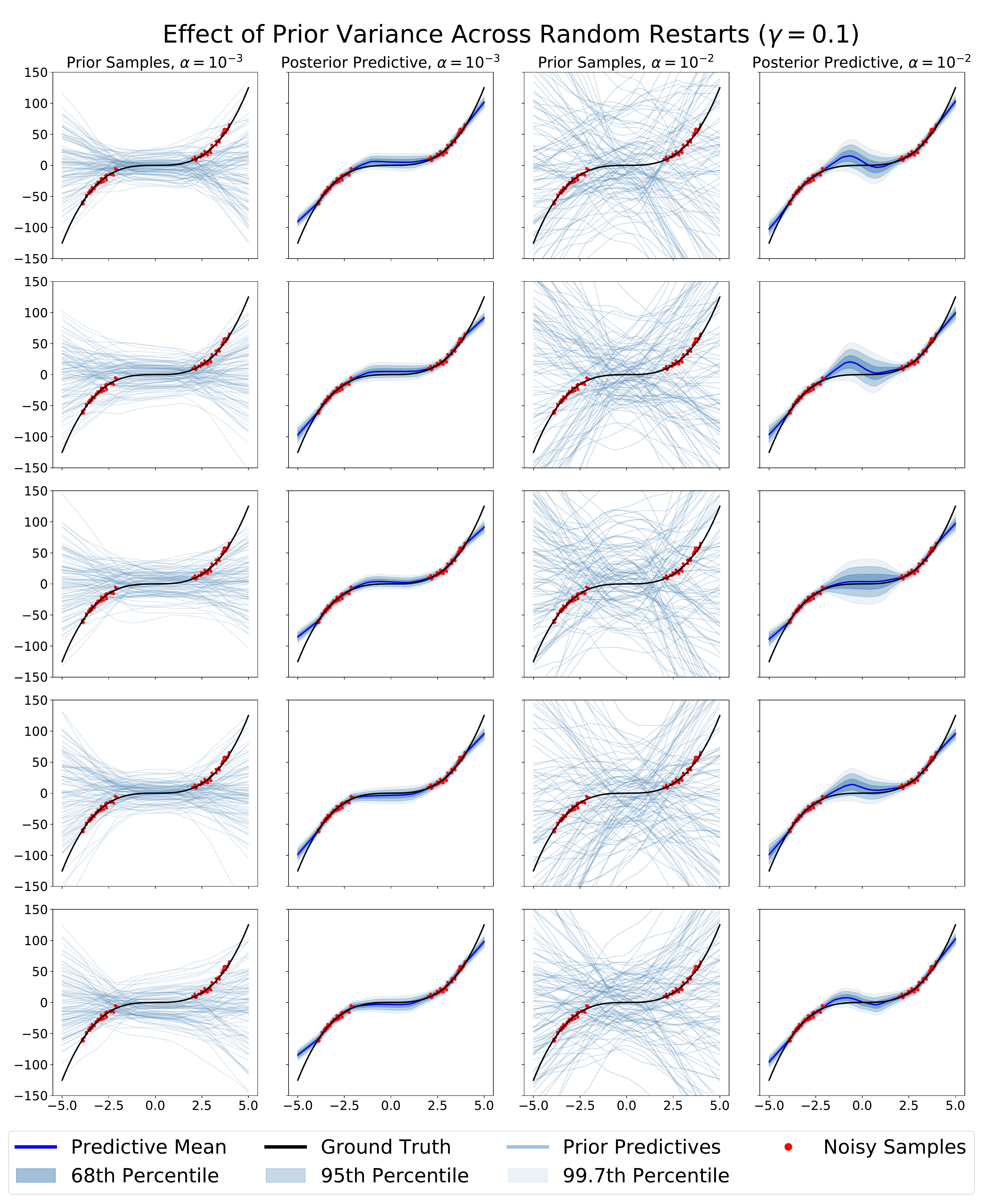}
    \caption{NLM is inconsistent in capturing the in-between uncertainty, even with higher $\alpha$.}
    \label{fig:marg_lowgamma}
\end{figure*}

This experiment was run with a 2-layer ReLU network with 50 and 20 neurons in the first and second layers respectively (20 features). We used marginal likelihood training and a larger $\gamma=1.0$. We see that this NLM is unable to capture higher in-between uncertainty even when $\alpha$ is high in Figure \ref{fig:marg_highgamma}.
\begin{figure*}[p]
    \centering
    \includegraphics[width=\textwidth]{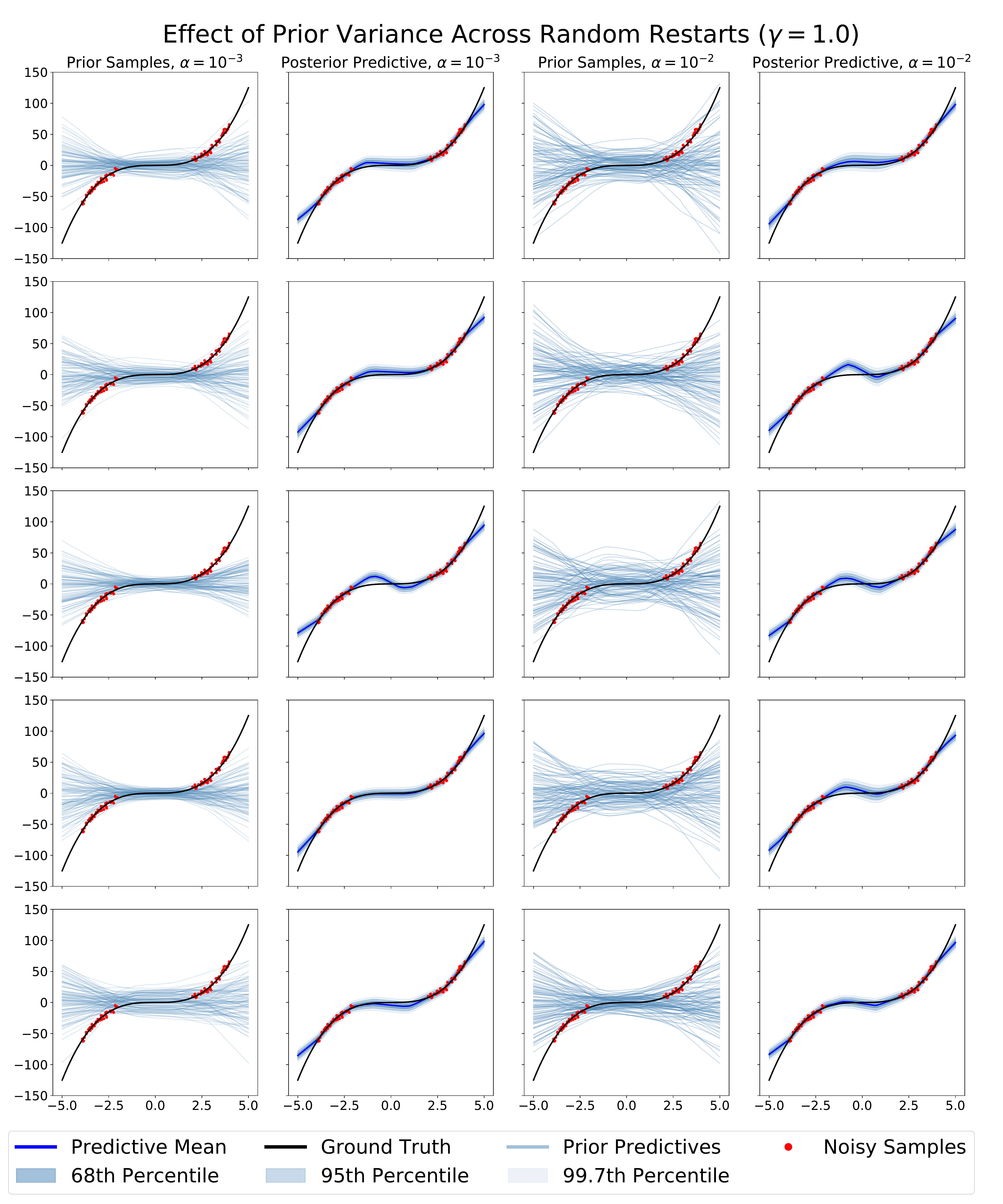}
    \caption{NLM is inconsistent in capturing the in-btween uncertainty, even with higher $\alpha$.}
    \label{fig:marg_highgamma}
\end{figure*}

\section{LUNA}
\label{sec:appendix_luna}

This experiment was run with a 2-layer ReLU network with 50 and 20 neurons in the first and second layers respectively (20 features). We used MAP training. We see that LUNA is able to capture uncertainty with diverse regressors. The regressors extrapolate neatly away from the data when there are 2 in Figure \ref{fig:aux_rr}. The regressors form a very diverse basis, with some fitting the data and extrapolating away from each other, and others approximating the means on each side of the gap.
\begin{figure*}[p]
    \centering
    \includegraphics[width=\textwidth]{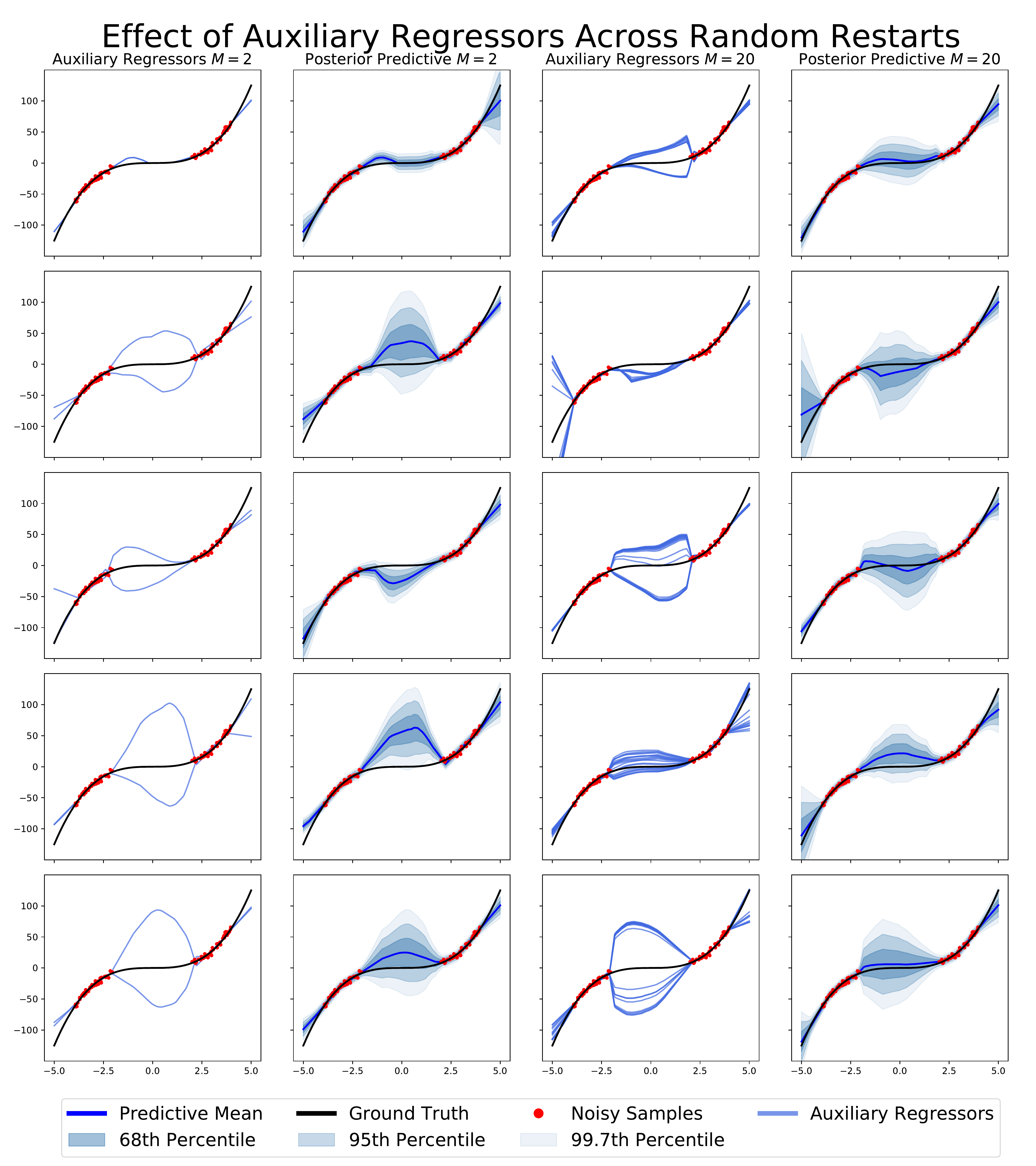}
    \caption{LUNA's auxiliary regressors consistently fit the data well and extrapolate differently away from the data.}
    \label{fig:aux_rr}
\end{figure*}

\section{TUNA}
\label{sec:appendix_tuna}

This experiment was run with a 2-layer ReLU network with 50 and 20 neurons in the first and second layers respectively (20 features). We used MAP training and an RBF kernel with length scale $l$. We see that TUNA is able to consistently capture uncertainty in the data scarce region when the auxiliary regressors are diverse. When the auxiliary regressors are not diverse, we see that TUNA fails to capture in-between uncertainty in Figure \ref{fig:kernel_rr}. \ref{fig:kernel_rr}.
\begin{figure*}[p]
    \centering
    \includegraphics[width=\textwidth]{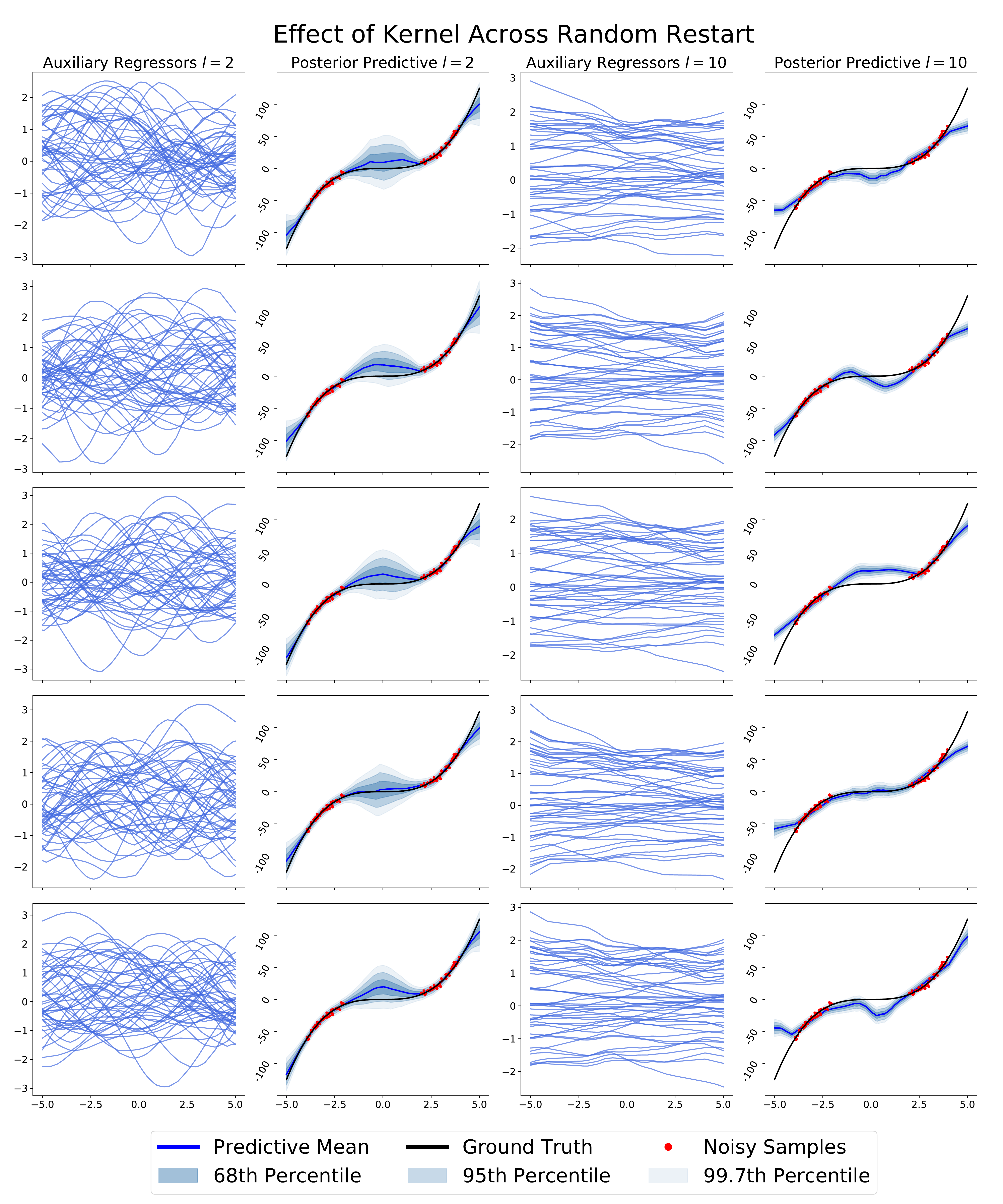}
    \caption{TUNA is able to consistently capture in-between uncertainty with appropriate kernel length scale.}
    \label{fig:kernel_rr}
\end{figure*}

\chapter{Appendix to Chapter \ref{ch:5}}\label{cha:append-chapt-refch:5}
\section{Experimental Results} \label{sec:appendix_experiment_results}
\subsection{Toy Experiments} \label{sec:appendix_toy}
Additional toy data experiments were run to showcase TUNA's predictive performance when compared to GPs.

\subsection{UCI Gap} \label{sec:appendix_uci_gap}

The raw values for epistemic uncertainty are given in Table \ref{tab:uci_gap_epistemic_values}.
We see that LUNA tends to have larger values of epistemic uncertainty in both the gap and not gap regions.
This is expected because we are explicitly penalizing the auxiliary regressors to be diverse.
All models have epistemic uncertainty within the same order of magnitude, with BBVI tending to have the largest.

\begin{table*}
    \centering
    \text{Avg. Epistemic Uncertainty}
    \resizebox{\textwidth}{!}{
    \begin{tabular}{ccccccccccccc}
        \toprule
        & \multicolumn{2}{c}{Yacht - FROUDE} & \multicolumn{2}{c}{Concrete - CEMENT} 
        & \multicolumn{2}{c}{Concrete - SUPER} & \multicolumn{2}{c}{Boston - RM}
        & \multicolumn{2}{c}{Boston - LSTAT} & \multicolumn{2}{c}{Boston - PTRATIO} \\
		\cmidrule(lr){2-3} \cmidrule(lr){4-5} \cmidrule(lr){6-7} \cmidrule(lr){8-9} \cmidrule(lr){10-11} \cmidrule(lr){12-13}
		& Not Gap & Gap & Not Gap & Gap & Not Gap & Gap & Not Gap & Gap & Not Gap & Gap & Not Gap & Gap \\
		\cmidrule{2-13}
		Ens. Boot. & 0.73 $\pm$ 0.11 & 0.58 $\pm$ 0.03 & 3.63 $\pm$ 0.24 & 5.10 $\pm$ 0.15 & 4.46 $\pm$ 0.39 & 8.32 $\pm$ 0.23 & 1.54 $\pm$ 0.18 & 1.38 $\pm$ 0.04 & 1.59 $\pm$ 0.20 & 1.67 $\pm$ 0.08 & 1.61 $\pm$ 0.17 & 1.58 $\pm$ 0.08 \\
		Ensemble & 0.39 $\pm$ 0.03 & 0.49 $\pm$ 0.02 & 2.10 $\pm$ 0.23 & 4.58 $\pm$ 0.21 & 1.92 $\pm$ 0.20 & 5.01 $\pm$ 0.15 & 0.92 $\pm$ 0.12 & 0.85 $\pm$ 0.02 & 0.92 $\pm$ 0.10 & 1.09 $\pm$ 0.03 & 0.90 $\pm$ 0.14 & 1.23 $\pm$ 0.04 \\
		Anc. Ens. & 0.60 $\pm$ 0.23 & 0.63 $\pm$ 0.19 & 2.76 $\pm$ 0.38 & 5.25 $\pm$ 0.29 & 2.53 $\pm$ 0.36 & 6.23 $\pm$ 0.49 & 1.18 $\pm$ 0.16 & 1.06 $\pm$ 0.05 & 1.10 $\pm$ 0.10 & 1.20 $\pm$ 0.08 & 1.13 $\pm$ 0.19 & 1.44 $\pm$ 0.07 \\
		\cmidrule{1-13}
		NLM & 0.12 $\pm$ 0.01 & 0.15 $\pm$ 0.02 & 0.77 $\pm$ 0.05 & 0.73 $\pm$ 0.04 & 0.74 $\pm$ 0.05 & 0.78 $\pm$ 0.04 & 0.37 $\pm$ 0.03 & 0.34 $\pm$ 0.03 & 0.90 $\pm$ 0.07 & 0.85 $\pm$ 0.06 & 0.76 $\pm$ 0.07 & 0.73 $\pm$ 0.04 \\
		GP & 0.94 $\pm$ 0.13 & 1.59 $\pm$ 0.10 & 3.04 $\pm$ 0.25 & 5.31 $\pm$ 0.07 & 2.76 $\pm$ 0.27 & 6.01 $\pm$ 0.13 & 1.88 $\pm$ 0.14 & 1.57 $\pm$ 0.02 & 1.76 $\pm$ 0.15 & 1.71 $\pm$ 0.06 & 1.93 $\pm$ 0.25 & 2.23 $\pm$ 0.08 \\
		MCD & 1.61 $\pm$ 0.23 & 0.75 $\pm$ 0.08 & 1.36 $\pm$ 0.07 & 1.45 $\pm$ 0.03 & 1.29 $\pm$ 0.08 & 1.44 $\pm$ 0.03 & 0.79 $\pm$ 0.08 & 0.65 $\pm$ 0.02 & 0.80 $\pm$ 0.07 & 0.70 $\pm$ 0.04 & 0.81 $\pm$ 0.07 & 0.77 $\pm$ 0.03 \\
		SNGP & 0.07 $\pm$ 0.01 & 0.10 $\pm$ 0.05 & 0.33 $\pm$ 0.03 & 0.31 $\pm$ 0.03 & 0.28 $\pm$ 0.04 & 0.27 $\pm$ 0.04 & 0.29 $\pm$ 0.04 & 0.26 $\pm$ 0.04 & 0.56 $\pm$ 0.05 & 0.52 $\pm$ 0.06 & 0.29 $\pm$ 0.03 & 0.27 $\pm$ 0.02 \\
		BBVI & 1.60 $\pm$ 0.11 & 1.43 $\pm$ 0.02 & 2.79 $\pm$ 0.16 & 2.28 $\pm$ 0.09 & 6.34 $\pm$ 0.35 & 7.00 $\pm$ 0.70 & 2.48 $\pm$ 0.20 & 2.13 $\pm$ 0.05 & 2.50 $\pm$ 0.18 & 1.70 $\pm$ 0.08 & 5.70 $\pm$ 0.38 & 6.84 $\pm$ 0.45 \\
		\cmidrule{1-13}
        LUNA & 0.44 $\pm$ 0.10 & 0.68 $\pm$ 0.14 & 1.32 $\pm$ 0.15 & 2.05 $\pm$ 0.37 & 1.45 $\pm$ 0.31 & 7.29 $\pm$ 2.61 & 1.12 $\pm$ 0.09 & 1.00 $\pm$ 0.07 & 1.81 $\pm$ 0.25 & 2.30 $\pm$ 0.22 & 1.14 $\pm$ 0.09 & 1.83 $\pm$ 0.52 \\
	\end{tabular}}
	\caption{Average Epistemic Uncertainty in the Not Gap and Gap regions.}
	\label{tab:uci_gap_epistemic_values}
\end{table*}

We do see that none of the models are able to detect the gap region for the Boston - RM data set.
We see in Figure \ref{fig:boston_rm} that the gap region is quite thin relative to the rest of the data.
The width of the gap region is approximately 0.5, where as the data has a range of 6.
Additionally, the gap occurs in a region of fairly steep signal.
The combination of these two factors makes detecting the gap region much more difficult than in other data sets, even from the same parent data set, such as Boston - LSTAT, where the gap is substantially wider, or Concrete - CEMENT, where the gap occurs in a relatively flat region.
\begin{figure}[H]
    \centering
    \includegraphics[width=0.9\linewidth]{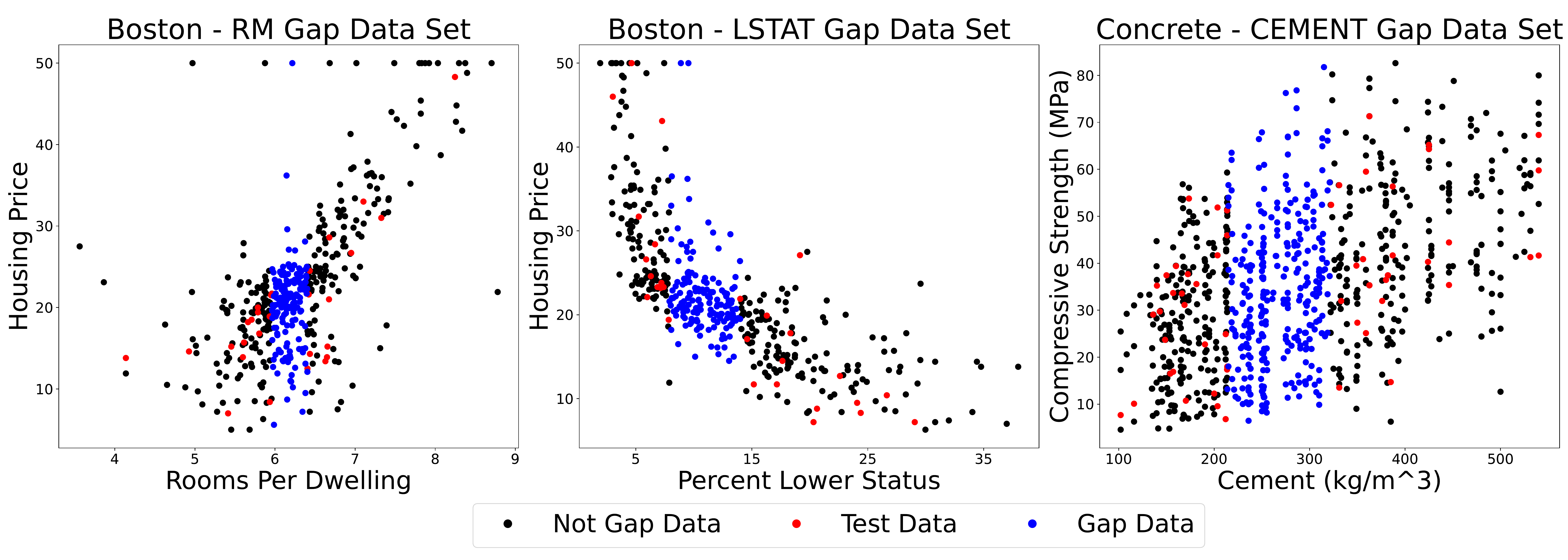}
    \caption{The Boston - RM UCI Gap Data Set is uniquely challenging for detecting the gap region.}
    \label{fig:boston_rm}
\end{figure}

\section{Experimental Setup} \label{sec:exp-setup}

\subsection{Synthetic Data} \label{sec:synthetic-data}

\paragraph{Cubic Gap Example}
Following the set-up in ~\citep{Rasmussen}, we construct a synthetic 1-D dataset comprising 100 train and 100 test pairs $(x, y)$, where $x$ is sampled uniformly in the range $[-4,-2]\cup[2,4]$ and $y$ is generated as $y =x^3 + \epsilon, \epsilon\sim\mathcal{N}(0, 3^2)$. 
All models were hand tuned using an architecture of two hidden layers with 50 nodes and ReLU activation.

\paragraph{Transfer Learning}
The transfer learning experiment was done on the squiggle gap function seen in Section \ref{sec:transfer_learning}.
100 data points were sampled from the not gap region for train, test, and validation, split using an 80-10-10 split.
10 random samples were taken for both the gap and not gap sections, and models were run on each sample with three random restarts.
Data was normalized for training and prediction, reported results were then unnormalized.
One hidden layer of 50 nodes, and one hidden layer of $N$ nodes, were used for both LUNA and NLM.
One layer of 50 nodes and $N$ Random Fourier Features were used for SNGP.
$N$ is the number of features, plotted along the x-axis in Section \ref{sec:transfer_learning}.
The Adam optimizer was used with a learning rate of $10^{-3}$ and 20000 epochs.
Hyperparameters were chosen over a grid with maximum log-likelihood on validation data used as the model selection criterion for both random restarts and on validation data.
No minibatching was used.

For LUNA, prior variance was selected over $\log \alpha \in [-1, 0, 0]$, regularization was selected over $\log \gamma \{-6, -5, -4\}$, diversity penalty was selected over $\log \lambda \in \{-4, -3\}$, or the sigmoid annealing schedule,  perturbation standard deviation was selected over $\log \sigma \in \{-2, -1\}$, and $N/2$ auxiliary regressors were used.
For NLM, prior variance was selected over $\log \alpha \in \{-1, 0, 1\}$, and regularization was selected over $\log \gamma \in \{-6, -5, -4\}$.
For SNGP, prior variance was selected over $\log \alpha \in \{-1, 0, 1\}$, regularization was selected over $\log \gamma \in \{-6, -5, -4\}$, and the normalizing factor was selected over $c \in \{1, 5\}$.

\paragraph{Radial Uncertainty Benchmark}
Hyperparameters were selected for all models by first using a BayesOpt procedure, then hand tuning each parameter, for each model, in each iteration of the experiment.
In one dimension, 10 random initializations with 20 iterations of BayesOpt were used.
In two and three dimensions, five random initializations and 10 iterations of BayesOpt were used.
All models were trained with 20,000 epochs using the Adam optimizer \cite{AdamOpt2014} with a learning rate of $10^-3$.
All NN-based models had an architecture of two layers of 50 hidden nodes with ReLU activation.
The gold-standard GP used the sum of a Matern Kernel starting at length scale 1 and a White Kernel with the noise level set to the sampled data noise in one and two dimensions.
In three dimensions, the Matern kernel started with a length scale of 0.1.
10 random restarts were used with the optimizer, as part of Scikit-Learn's built-in functionality \cite{scikit-learn}.
No minibatching was used.

For LUNA, a sigmoid diversity annealing schedule was used.
The scaling factor of the annealing, that is, $C$ in the equation $f_{sigmoid} = C/(1 + exp(-6x/N + 3))$, was searched for in log space in the range $\log C \in [-5 ,1]$.
Regularization was searched for in the range $\log \gamma \in [-10, 2]$.
Prior variance was searched for in the range $\alpha \in [10^{-7}, 10^{2}]$.
Perturbation standard deviation was searched for in the range $\log \sigma \in [-2, 0]$, and 25 auxiliary regressors were used.
For NLM, the prior variance was searched for in the range $\alpha \in [10^{-8}, 10^2]$ and the regularization parameter was searched for in the range $\log \gamma \in [-10, 2]$.
For TUNA, the prior variance was searched for in the range $\alpha \in [10^{-7}, 10^2]$, the regularization was searched for in the range $\log \gamma \in [-10, 2]$, the kernel scaling factor was searched for in the range $\log c \in [-7, 2]$, the perturbation standard deviation was searched for in the range $\log \sigma \in [-3, 0]$, and the length scale was searched for in the range $\log l \in [-3, 1]$.
For SNGP, the prior variance was searched for in the range $[10^{-8}, 10^2]$, the regularization was searched for in the range $\log \gamma \in [-10, 2]$, and the normalizing factor was searched for in the range $c \in [0, 2]$.
For Anchored Ensembles, 5 networks were used, with the initial variance, the data variance, and anchor variance all searched for in log space over the range $[-9, 1]$.
For Bootstrapped Ensembles, 20 networks were used and the regularizaion was searched for in the range $\log \gamma \in [-9, 1]$.
For BBVI, the weight variance was searched for in the range $\log \sigma^2 \in [-9, 1]$, and 1000 posterior samples were used for prediction.
For MCDropout, the regularization was searched for in the range $\log \gamma \in [-9, 1]$, the dropout was searched for in the range $[0, 1]$, and 100 output samples were used for prediction.

\paragraph{BayesOpt Example}
For all models a three layer, 50 hidden node architecture with Tanh activation was used unless noted otherwise.
For the Branin and Hartmann6 functions, we simply were searching for the global minimum.
For the logistic regression and SVM experiments, we first searched for hyperparameters on the validation set, then reported accuracy on the test set.
An 80-10-10 train-test-validation split was used.
Each experiment was run with 10 random restarts of three different random initial guesses.
A grid search over hyperparameters were chosen once at the beginning of each experiment.
The Adam optimizer with a learning rate of $10^{-3}$ and 1000 epochs was used in each experiment.
No minibatching was used.

For LUNA, the prior variance was searched for in log space over $\log \alpha \in [-1, 0, 1]$, the regularization was searched for in log space over $\log \gamma \in [-5, -4, -3]$, the diversity penalty was searched over $\lambda = 0.1$, tanh, sigmoid, square root annealing schedules specified in section \ref{sec:luna_main_body}, the perturbation standard deviation was searched for in $\sigma \in [0.1, 0.01]$, and 50 auxiliary regressors were used.
Additional hand-tuning was done for LUNA on the Hartmann6 experiment.
For NLM, the prior variance was searched for in log space over $\log \alpha \in [-1, 0, 1]$, the regularization was searched for in $\log \lambda \in [-5, -4, -3]$.
The GP length scale was searched for in $l \in [0.001, 0.01, 0.1, 1, 5, 10]$ with an RBF kernel.
SNGP prior variance was searched for over the range $\log \alpha \in [-1, 0, 1]$, the normalizing factor was searched for in the range $c \in [1, 5]$, no dropout, and 200 Random Fourier Features were used.
For Anchored Ensembles, 5 models were used, and the anchor variance, data variance, and initial variance were all searched for over the log space $[-1, 0, 1]$.
For both bootstrap and regular ensembles, 25 networks were used and regularization was searched for over the range $\log \gamma \in [-5, -4, -3, -2, -1]$.
For MCDropout, 5 forward passes were used due to the multiple evaluations needed during acquisition function optimization, dropout was searched for over the range $[0.01, 0.05, 0.1]$, and regularization was searched for over the range $\log \gamma \in [-4, -3, -2]$.
For BBVI, a reduced architecture of one layer of 50 nodes was used due to computational complexity and memory overhead, a learning rate of $10^{-2}$ was also used,, dropout rate was selected in the range $\{0.01, 0.05, 0.1, 0.2\}$ and regularization was automatically set to $1$ 1000 posterior samples were used, weight variance was searched for over the range $\sigma^2_w \in [1, 10, 50]$, and prior variance was searched for in the range $\log \alpha \in [-1, 0, 1]$.




\subsection{Real Data} \label{sec:uci-gap-examples}

\paragraph{UCI Regression}
We used six UCI regression data sets to benchmark our models.
Those are the Boston Housing, Concrete Compressive Strength, Yacht Hydrodynamics, Energy Efficiency, Abalone (Kin8nm), and Red Wine Quality datasets.
For each experiment, the data was split into 90\% train data and 10\% test data.
The train data was then split using the first 80\% train, last 20\% validation sets.
The Adam optimizer was used with 20,000 epochs and a learning rate of $10^{-3}$.
A two hidden layer, 50 node architecture with ReLU activation was used for each NN-based model.
Data noise was hand selected based on visual examination of the data.
Hyperparameters were selected using grid search, with minimum RMSE being used to select the best random restart and model for Anchored Ensembles, Bootstrap Ensembles, and Ensembles.
Maximum log-liklihood for Bayesian models, and minimum RMSE for non-Bayesian models, was used for random restart and model selection for each of the other models.
A minibatch size of 32 was used unless noted otherwise.

For LUNA, prior variance was selected over $\log \alpha \in \{-1, 0, 1\}$, regularization was selected over $\log \gamma \in \{-5, -4, -3, -2, -1\}$, diversity penalty was selected over $\log \lambda \in \{-8, -7, -6\}$ and the sigmoid annealing schedule, perturbation standard deviation was selected over $\log \sigma \in \{-2, -1\}$, minibatch size of 128 was used.
Additional hand tuning was done for the Concrete and Boston data sets.
For NLM, prior variance was selected over $\log \alpha \in \{-2, -1, 0, 1, 2\}$, regularization was selected over $\log \gamma \in \{-5, -4, -3, -2, -1\}$.
For GP, and RBF Kernel was used and length scale was searched for in $\log l \in \{-2, -1, 0, 1\}$.
For SNGP, prior variance was selected over $\log \alpha \in \{-2, -1, 0, 1, 2\}$, regularization was selected over $\log \gamma \in \{-5, -4, -3\}$, normalizing factor was searched for over $c \in \{1, 5, 10\}$, 200 Random Fourier Features were used.
For Anchored Ensembles, 5 models were used, and the anchor variance, data variance, and initial variance were all searched for over the log space $\{-2, -1, 0\}$.
For both Bootstrap Ensembles and Ensembles, 20 networks were used and regularization was searched for over the range $\log \gamma \in \{-5, -4, -3, -2, -1\}$.
For BBVI, 1000 posterior samples were used, weight variance was searched for over the range $\sigma^2_w \in [1, 10, 50]$, and prior variance was searched for in the range $\log \alpha \in [-1, 0, 1]$.
For MCDropout, dropout rate was selected in the range $p \in \{0.01, 0.05, 0.1, 0.2\}$ and regularization was automatically set to $1/\sigma^2_{data}$.

\paragraph{UCI Gap}
We used 3 standard UCI~\citep{uci_data} regression data sets and modify them to create 6 ``gap data-sets'', wherein we purposefully created a gap in the data where we can test our model's in-between uncertainty (i.e. we train our model on the non-gap data and test the model's epistemic uncertainty on the gap data).
We adapt the procedure from ~\cite{uci_gap} to convert these UCI data sets into UCI gap data sets. For a selected input dimension,
we (1) sort the data in increasing order in that dimension,
and (2) remove middle $1/3$ to create a gap.
We specifically selected input dimensions that have high correlation with the output in order to ensure that the learned model should have epistemic uncertainty in the gap;
that is, if we select a dimension that is not useful for prediction, any model need not have increased uncertainty in the gap.
The features we selected are:
\begin{itemize}
    \item Boston Housing: ``Rooms per Dwelling'' (RM), ``Percentage Lower Status of the Population'' (LSTAT), and ``Parent Teacher Ratio'' (PTRATIO)
    \item Concrete Compressive Strength: ``Cement'' and ``Superplasticizer''
    \item Yacht Hydrodynamics: ``Froude Number''
\end{itemize}
The not gap region of the data was then split into 10 different 80-10-10 train,test, validation splits.
Final results are computed as the mean and standard deviation of the predicions over all of the splits.
The same architecture was used here as in the UCI regression experiments above.
Data noise was hand tuned by maximizing validation log-likelihood of an NLM with prior variance $0.1$ and regularization $10^{-5}$.
Maximum log-liklihood for Bayesian models, and minimum RMSE for non-Bayesian models, was used for random restart and model selection for each of the other models.
All models used a batch size of 32 unless noted otherwise.

For LUNA, prior variance was selected over $\log \alpha \in \{-1, 0, 1\}$, regularization was selected over $\log \gamma \in \{-5, -4, -3, -2, -1\}$, diversity penalty was selected over $\log \lambda \in \{-2, -1\}$, the sigmoid, and tanh annealing schedules, perturbation standard deviation was set to $\log \sigma = -2$, minibatch size of 128 was used.
Model selection was done by first selecting the top 90th percentile runs based on validation log-likelihood, then selecting the lowest validation diversity penalty from these runs.
For NLM, prior variance was selected over $\log \alpha \in \{-1, 0, 1\}$, regularization was selected over $\gamma \in \{-5, -4, -3, -2, -1\}$.
For GP, and RBF Kernel was used and length scale was searched for in $l \in \{10^{-2}, 5\cdot10^{-2}, 10^{-1}, 5\cdot^{-1}, 0, 1, 5, 10, 50, 500\}$.
For SNGP, prior variance was selected over $\log \alpha \in \{-1, 0, 1\}$, regularization was selected over $\log \gamma \in \{-5, -4, -3, -2, -1\}$, normalizing factor was searched for over $c \in \{1, 5\}$, 200 Random Fourier Features were used.
For Anchored Ensembles, 5 models were used, and the anchor variance, data variance, and initial variance were all searched for over the log space $\{-2, -1, 0\}$.
For both Bootstrap Ensembles and Ensembles, 20 networks were used and regularization was searched for over the range $\log \gamma \in \{-5, -4, -3, -2, -1\}$.
For BBVI, 1000 posterior samples were used, weight variance was searched for over the range $\sigma^2_w \in \{1, 10, 50\}$, and prior variance was searched for in the range $\log \alpha \in [-1, 0, 1]$.
50,000 training epochs and no random restarts were used due to increased computational time.
For MCDropout, dropout rate was selected in the range $p \in \{0.005, 0.01, 0.05, 0.1\}$ and regularization was automatically set to $1/\sigma^2_{data}$.

\end{appendices}

\end{document}